\documentclass[mnsc,nonblindrev]{informs4}

\RequirePackage{newtxmath}
\RequirePackage{endnotes}

\OneAndAHalfSpacedXI

\RequirePackage{bm}
\usepackage{amsmath,amsfonts,amssymb}
\usepackage{mathtools}
\allowdisplaybreaks
\usepackage[mathscr]{euscript}
\usepackage{enumitem}
\usepackage{xcolor}
\usepackage{graphicx}
\usepackage{multirow}
\usepackage{booktabs}
\usepackage{microtype}
\usepackage{fix-cm}
\usepackage{graphicx}
\usepackage{subcaption}
\usepackage{physics}
\usepackage{xfrac}
\usepackage{tabularx}

\usepackage{algorithm}  
\usepackage{algorithmicx}  
\usepackage{algpseudocode}
\usepackage{psfrag}
\usepackage{makecell}
\usepackage{tikz}
\usetikzlibrary{arrows.meta}

\usepackage{xcolor}

\DeclareMathOperator{\E}{\mathbb E}
\DeclareMathOperator{\R}{\mathbb R}
\DeclareMathOperator{\Var}{\mathsf{Var}}

\usepackage{comment}

\usepackage{natbib}
 \bibpunct[, ]{(}{)}{,}{a}{}{,}
 \def\bibfont{\small}

\usepackage{multibib}
\newcites{EC}{References}

\usepackage[
    colorlinks=true,
    linkcolor=blue,
    citecolor=black,
    urlcolor=black
]{hyperref}

\TheoremsNumberedThrough

\EquationsNumberedThrough

\makeatletter

\renewcommand{\theARTICLETOP}{}

\RRHSecondLine{}
\LRHSecondLine{}

\makeatother

\begin{document}

\RUNAUTHOR{Hong, Huang, Li}

\RUNTITLE{Contextual Decision Making: Predictive or Generative}

\RRHFirstLine{\bfseries\theRUNAUTHOR:}
\RRHSecondLine{\itshape\theRUNTITLE}

\LRHFirstLine{\bfseries\theRUNAUTHOR:}
\LRHSecondLine{\itshape\theRUNTITLE}

\TITLE{Contextual Decision Making: Predictive or Generative}

\ARTICLEAUTHORS{
\AUTHOR{L. Jeff Hong}
\AFF{Department of Industrial and Systems Engineering, University of Minnesota, Minneapolis, MN 55455, U.S., 
\href{mailto:lhong@umn.edu}{\textcolor{blue}{lhong@umn.edu}}}

\AUTHOR{Weihuan Huang}
\AFF{School of Management \& Engineering, Nanjing University, Nanjing 210093, China, 
\href{mailto:hwh@nju.edu.cn}{\textcolor{blue}{hwh@nju.edu.cn}}}

\AUTHOR{Xinyao Li}
\AFF{Department of Industrial and Systems Engineering, University of Minnesota, Minneapolis, MN 55455, U.S.,
\href{mailto:li003784@umn.edu}{\textcolor{blue}{li003784@umn.edu}}}
}

\ABSTRACT{Contextual decision making uses observed covariates to guide decisions under uncertainty. Two natural approaches are predict-then-optimize (PtO), which learns the conditional expected objective, and generate-then-optimize (GtO), which learns conditional uncertainty and optimizes using generated scenarios. We study when each approach is preferable by examining how the choice of learning target affects statistical accuracy and optimization tractability. Although PtO learns a narrower target, our regret bounds show how differences in smoothness can offset GtO's additional dimensional burden. Exploiting smoothness also raises an optimization tradeoff: within the class of nonnegative-weight PtO methods studied, a lower bound limits the gains from additional smoothness, while more flexible methods may introduce signed weights and lose the guarantee of preserving convexity. GtO preserves scenario-wise convexity, but requires learning a richer object and generating scenarios for optimization. Synthetic experiments and a real case study illustrate how these tradeoffs vary with data availability and problem structure. Together with the scope for reusing a model across tasks, our findings provide practical guidance for choosing between the approaches without identifying a universal winner.
}


\maketitle

\section{Introduction}
\label{sec:introduction}

Modern operating environments provide decision makers with rich contextual information. Retailers observe promotions, recent sales, and local market conditions before replenishing inventory. Logistics providers use congestion and demand signals to plan deliveries and reserve capacity. Service systems use workload forecasts and current operating conditions to make staffing decisions, as illustrated by feature-based approaches to hospital nurse staffing \citep{ban2019big}. Healthcare providers use patient characteristics and clinical information to personalize treatment and allocate care resources \citep{shen2021ranking}. In each setting, the challenge is to turn information available before a decision into an effective response to an uncertain outcome. Contextual stochastic optimization (CSO) provides a framework for doing so by minimizing expected cost conditional on the observed context \citep{bertsimas20,sadana25survey}.

Let $X$ denote the context, $Y$ the uncertain outcome, $z$ the decision, and $c(z;x,y)$ the cost. After observing $X=x_0$, the decision maker seeks to minimize the conditional expected objective
\[
h(z;x_0)=\mathbb{E}[c(z;X,Y)\mid X=x_0].
\]
We focus on convex CSO: the feasible decision set is convex, and the cost is convex in $z$ for each context and outcome. The conditional expected objective is consequently convex as well. The difficulty is that this objective is unknown and must be learned from historical observations of contexts and outcomes.

Two approaches arise naturally. A predictive approach learns the conditional expected objective and optimizes the fitted objective when a new context is observed \citep{bertsimas20}. We refer to this approach as \emph{Predict-then-Optimize} (PtO). A generative approach learns the conditional distribution of $Y$ given $X$, generates scenarios for the observed context, and optimizes their average cost \citep{gong2024deep}. We refer to this approach as \emph{Generate-then-Optimize} (GtO). For instance, in a retail inventory problem, PtO estimates expected inventory cost, whereas GtO generates demand scenarios under which inventory decisions can be evaluated. PtO concentrates learning on the objective of interest. GtO learns a conditional distribution that can support different objectives involving the same uncertain outcome.

The rise of generative AI makes the generative approach increasingly appealing. Advances in generative modeling provide flexible tools for learning conditional distributions and producing scenarios \citep{bondtaylor2022deep,yang2023diffusionsurvey}. Further progress may make these tools easier to develop and use, while creating opportunities to incorporate qualitative context, such as descriptions of operating conditions or expert judgments expressed in language. As these capabilities develop, more organizations may be able to use generative models in operational decision making. Yet the ability to generate realistic scenarios does not by itself establish that learning a conditional distribution is preferable to learning the expected objective. The choice still depends on the data available, the optimization problem to be solved, and the decisions the model is intended to support.

This motivates our central question: \emph{When should contextual decision making be predictive, and when should it be generative?} Both approaches can be applied to many of the same decision problems, but their relative appeal is not immediately clear. Learning the expected objective focuses on what a particular task requires, while learning the conditional distribution provides scenarios that can be used to evaluate decisions. We seek to understand when each approach is useful and to develop practical guidance for choosing between them. As generative AI makes scenario generation more accessible, this choice will become increasingly~important.

The first consideration is statistical: how difficult is it to learn the expected objective or the conditional distribution from the available data? PtO avoids estimating how probability is distributed across outcomes, which can be valuable when data are limited or the uncertain outcome is high-dimensional. GtO must account for both the context and the uncertain outcome when learning the conditional distribution. Dimension alone, however, does not determine the difficulty of learning. Smoothness also matters because it allows observations at nearby contexts to inform one another, a central principle in nonparametric regression and conditional-density estimation \citep{stone1982optimal,li2022minimax}. The conditional distribution may vary smoothly with the context even when the expected objective changes abruptly because costs depend on thresholds or operating rules. Learning the expected objective can therefore involve fewer dimensions without necessarily being easier.

We examine this tradeoff through bounds on regret, defined as the excess expected cost relative to an optimal decision. For PtO, our bound for kernel ridge regression (KRR) describes how greater objective smoothness can mitigate the effect of context dimension. For GtO, the corresponding analysis accounts for the dimension and smoothness of the conditional distribution, with conditional diffusion models providing a concrete instance \citep{fu24}. These results suggest assessing dimension and smoothness together when choosing an approach. PtO is attractive when the expected objective is sufficiently smooth to be learned accurately from the available data. GtO becomes more attractive when the conditional distribution is easier to learn because its greater smoothness compensates for the additional outcome dimension. The bounds explain these competing effects and inform the choice.

The second consideration is computational: how does learning affect the optimization problem that must ultimately be solved? Although the original decision problem is convex, an approximation learned from data need not preserve that property. A broad class of PtO methods fits the expected objective as a weighted sum of observed costs,
\[
\hat h_n(z;x)=\sum_{i=1}^n\omega_{n,i}(x)c(z;x_i,y_i),
\]
where $\{(x_i,y_i)\}_{i=1}^n$ are historical observations \citep{bertsimas20,bertsimas22}. Nonnegative weights preserve convexity. Methods such as KRR can exploit higher-order smoothness but may assign negative weights, so their fitted objectives are no longer guaranteed to be convex. GtO preserves convexity by averaging scenario costs with equal positive weights, but must generate the scenarios and solve the resulting sample average approximation (SAA) problem \citep{ShapiroDentchevaRuszczynski2009}. The choice of learning method thus affects both statistical accuracy and the computation needed to obtain a decision.

Our analysis makes this connection precise. For the class of nonnegative, covariate-based PtO weighting rules considered, we establish a lower bound of order $n^{-4/(d_x+4)}$ on worst-case expected regret, where $d_x$ is the context dimension. The bound holds even for arbitrarily smooth objectives in the constructed problem class. Within this class, the nonnegative weights that preserve convexity limit the benefit that additional smoothness can provide. KRR can benefit from higher-order smoothness, but its negative weights may complicate optimization. The extent of that difficulty depends on the cost function: negative weights do not invariably produce a nonconvex or difficult problem. A predictive method should therefore be assessed together with the optimization problem induced by its fitted objective.

For GtO, our regret bound separates the error in learning the conditional distribution from the error caused by using finitely many generated scenarios. Under the stated assumptions, the bound consists of a term proportional to the conditional-distribution learning error and an SAA term of order $m^{-1/2}$, where $m$ is the number of generated scenarios. Improving the learned distribution reduces the first error, whereas generating more scenarios reduces the second at additional computational cost. Preserving convexity therefore does not by itself imply faster decisions. PtO is attractive when its fitted objective remains easy to optimize and rapid responses are important. GtO is attractive when retaining the convex structure is valuable and the time available permits scenario generation and optimization. This distinction is particularly relevant when decisions must be made repeatedly under tight response-time requirements.

The third consideration is how broadly the learned model will be used. A model serving one objective can concentrate its learning effort on that task. In practice, however, several operational teams may face the same conditional distribution while making decisions with different objectives or constraints. When data analysis and operations reside in different departments, adapting a model to each objective can also require substantial coordination. GtO offers one way to organize this work: a data team develops a common generative model, and operational teams use its scenarios in their own decision problems. PtO can also support reuse through shared weights, features, or estimated objective components. The relevant question is how much adaptation and coordination each additional use requires. A stable, specialized task makes objective learning appealing, whereas a range of tasks sharing the same conditional distribution increases the value of a reusable generative model. The latter may be worth developing even when it offers no advantage for an individual task.

Taken together, these three considerations provide practical guidance on choosing between PtO and GtO. The choice involves balancing what can be learned reliably, what can be optimized within the available time, and how the model will serve the organization's decisions. Neither approach is preferred in all settings, and an advantage along one of these dimensions may be offset by a disadvantage along another. We draw on the theoretical results and these practical considerations to develop rules of thumb that help decision makers assess that balance.

We use two numerical examples, a contextual newsvendor problem and a product placement problem, to examine the tradeoffs and provide evidence supporting this guidance. We then extend the comparison through a real case study of fresh-retail inventory using operational data from a retail chain. With both continuous and categorical features, this case goes beyond the setting of our theoretical analysis and provides further evidence of the usefulness of our managerial insights in practice.

\subsection{Related Literature}
\label{subsec:related-literature}

Our work builds on contextual stochastic optimization, where historical data and contextual information are used to improve decisions under uncertainty. \citet{hannah2010nonparametric} study nonparametric methods for stochastic optimization with observable state information. \citet{ban2019big} develop feature-based methods for the newsvendor problem, and \citet{bertsimas20} provide a general prescriptive analytics framework using side information. \citet{bertsimas22} develop methods based on reproducing kernel Hilbert spaces, while \citet{bertsimas2023dynamic} extend optimization with side information to dynamic settings. \citet{sadana25survey} review the broader literature. The weighting frameworks of \citet{bertsimas20} and \citet{bertsimas22} provide the foundation for our PtO analysis. We use this representation to examine how the choice between objective learning and conditional-distribution learning affects statistical performance and downstream optimization.

A related stream of research studies how training can account for the decisions that a learned model will support. \citet{donti2017task} incorporate optimization into task-based and decision-focused learning, while \citet{elmachtoub2022smart} develop the Smart Predict-then-Optimize framework and the SPO+ surrogate loss. Broader accounts are provided by \citet{mandi2024decision} and \citet{sadana25survey}. In this paper, PtO refers specifically to learning the conditional expected objective; elsewhere, the term also describes predicting inputs to an optimization model. Our comparison concerns what is learned, while decision-focused training concerns how learning is aligned with decision quality. These questions are complementary, since either an objective model or a generative model can, in principle, be trained using decision-focused criteria. We study the choice of learning target within a framework that separates learning from optimization.

The statistical foundations of our comparison come from nonparametric regression and conditional-density estimation, where smoothness and dimension govern estimation difficulty \citep{stone1982optimal,li2022minimax}. Generative models provide practical methods for learning distributions, including variational autoencoders, adversarial models, normalizing flows, and diffusion models; \citet{bondtaylor2022deep} and \citet{yang2023diffusionsurvey} provide reviews. Our conditional-diffusion analysis uses the learning guarantees of \citet{fu24}, and our retail implementation uses TabCSDI \citep{zheng2022diffusion}. We use these existing methods to study the decision consequences of learning a conditional distribution.

The use of generated scenarios also connects our work to stochastic programming. Scenario construction and reduction have long been used to approximate stochastic optimization problems \citep{hoyland2001generating,dupavcova2003scenario}. Recent studies use deep generative models to construct decision-relevant scenarios or incorporate optimization objectives into generative training \citep{pisinger2026deep,wang2026gen,zhao2026diffusion}. Our contribution is a comparative analysis of PtO and GtO that connects regret guarantees with preservation of optimization structure and computational requirements. This analysis supports practical guidance for choosing between the approaches and complements research on improving the performance of either one.

The remainder of the paper is organized as follows. Section~\ref{sec:problem} formulates the problem and performance criteria. Section~\ref{sec:PtO} analyzes weighted PtO methods, and Section~\ref{sec:GtO} develops the GtO regret analysis. Section~\ref{sec:PorG} draws practical guidance from the comparison. Section~\ref{sec:numerical experiments} presents the numerical and empirical studies, and Section~\ref{sec:conclusion} concludes. The appendix contains the proofs, implementation details, and supplementary numerical~results.

\section{Problem Formulation}
\label{sec:problem}

We consider a convex contextual stochastic optimization (CSO) problem of the form
\begin{equation}
z^{\ast}(x) \in \arg\min_{z \in \mathcal{Z}} \Big\{ h(z;x):=  \E\big[c(z;X,Y)\,|\,X = x \big]\Big\},
\label{eq: stochastic programming with covariates}
\end{equation}
where $X$ is a random vector of covariates taking values in $\mathcal{X}\subseteq \R^{d_x}$, $Y$ is a vector of uncertain parameters taking values in $\mathcal{Y}\subseteq \R^{d_y}$ and potentially correlated with $X$, $z$ is the vector of decision variables with the feasible set $\mathcal{Z}\subseteq \R^{d_z}$, and $c(z;x,y)$ is a given cost function that is convex in the decision variable $z$. In the offline stage, we have access only to historical data $\mathcal{D}_n =\{(x_i,y_i)\}_{i=1}^n$, consisting of $n$ observations of $(X,Y)$. In the online stage, upon observing a new covariate realization $X=x_0$, we make an online decision $z^\ast(x_0)$ via solving Problem~\eqref{eq: stochastic programming with covariates}. We assume that the offline observations $\{(x_i,y_i)\}_{i=1}^n$ are independent and identical distributed (i.i.d.) copies of $(X,Y)$, and that the online covariate $X=x_0$ observed at the decision epoch is an independent draw from the same distribution as the training covariates.
We impose the following structural assumptions on the feasible set and the cost function.

\begin{assumption}\label{ass:SAA-struct0}
	    The feasible set $\mathcal{Z}\subset\R^{d_z}$ is nonempty, convex, and compact. We denote its diameter by
		$
		\mathrm{diam}(\mathcal{Z})
		:= \sup\{\|z_1-z_2\| : z_1,z_2\in\mathcal{Z}\} = D < \infty.
		$
        For every $(x,y)\in\mathcal{X}\times\mathcal{Y}$, the cost function $c(\cdot; x,y)$ is convex on $\mathcal{Z}$, uniformly Lipschitz with respect to $z$, i.e., there exists a constant $G<\infty$ such that for all $z_1,z_2\in\mathcal{Z}$ and all $(x,y)\in\mathcal{X}\times\mathcal{Y}$, we have
		$
		|c(z_2; x,y) - c(z_1; x,y)| \le G \|z_2-z_1\|,
		$
		and uniformly bounded, i.e., there exists a constant $M<\infty$ such that for all $z\in\mathcal{Z}$ and $(x,y)\in\mathcal{X}\times\mathcal{Y}$, we have
		$
		|c(z; x,y)| \le M.
		$
\end{assumption}

The assumptions for the feasible set in  Assumption~\ref{ass:SAA-struct0} is standard in stochastic programming and contextual stochastic optimization to guarantee the existence of minimizers, see, e.g., \cite{ShapiroDentchevaRuszczynski2009,bertsimas20,bertsimas22}. The convexity assumption for the cost function in Assumption~\ref{ass:SAA-struct0} restricts our focus for the convex CSO problem.  Assumption~\ref{ass:SAA-struct0} also imposes a uniform Lipschitz and bounded condition on the cost function; such conditions are reasonable technical assumptions and used in the literature of contextual stochastic optimization, see, e.g., \cite{bertsimas20,bertsimas22}. 

Under Assumption~\ref{ass:SAA-struct0}, the conditional objective $h(\cdot;x)$ is convex for every $x\in\mathcal{X}$, uniformly $G$-Lipschitz continuous in $z$ over $x\in\mathcal{X}$, and uniformly bounded.
Consequently, $h(\cdot;x)$ is continuous on $\mathcal{Z}$. Since $\mathcal{Z}$ is compact, the set of optimal solutions in Problem~\eqref{eq: stochastic programming with covariates} is nonempty, compact, and convex for every $x\in\mathcal{X}$. We do not require the optimal solution to be unique.
Moreover, we impose the following regularity conditions on the conditional objective and the covariate distribution.

\begin{assumption}
\label{ass:smoothness_h}
There exists a constant $L_h < \infty$ such that $h(z;x) := \mathbb{E}[c(z;X,Y)| X = x]$ is 
$L_h$-Lipschitz in $x$, uniformly in $z$, i.e.,
$| h(z;x) - h(z;x') | \le L_h \|x - x'\|$ holds for all $z \in \mathcal{Z}$ and for all $x, x' \in \mathcal{X}$.
\end{assumption}

\begin{assumption}
\label{ass:Regularity_density}
The covariates are supported on the hypercube 
$\mathcal{X} = [0,1]^{d_x}$ and admit a density $p$ with respect to 
the Lebesgue measure satisfying
$
0 < p_{\min} \le p(x) \le p_{\max} < \infty$ for all $x \in \mathcal{X}$.
\end{assumption}

For a data-driven decision rule $\hat z_n(x_0)$ constructed from
$\mathcal D_n$, we evaluate its performance through the conditional
optimality gap (regret)
\[
\mathcal{R}_n(x_0)
:=
h\left(\hat z_n(x_0);x_0\right)-h\left(z^\ast(x_0);x_0\right),
\]
and, when appropriate, its expected counterpart
$
\mathcal{R}_n
:=
\E_X[
h(\hat z_n(X);X)-h(z^*(X);X)
].
$
For ease of interpretation and comparison across problem instances, we also consider a normalized performance measure based on the regret. Let $\hat z_{\mathrm{B}}(x)$ denote a benchmark decision rule and define its expected regret by
$
\mathcal{R}_{\mathrm{B}}
:=
\E_X[
h(\hat z_{\mathrm{B}}(X);X)-h(z^\ast(X);X)
].
$
Provided that $\mathcal{R}_{\mathrm{B}}>0$, we define the \textit{relative cost} of $\hat z_n$ as
\begin{equation}
\mathrm{RC}_n
:=
1-\frac{\mathcal{R}_n}{\mathcal{R}_{\mathrm{B}}} \,=\, \frac{\E_X[h(\hat z_{\mathrm{B}}(X);X)] - \E_X[h(\hat z_n(X);X)]}{\E_X[
h(\hat z_{\mathrm{B}}(X);X)] - \E_X[h(z^\ast(X);X)]}.
\label{eq: relative cost}
\end{equation}
Thus, the relative cost measures the fraction of the benchmark's optimality gap eliminated by the data-driven decision rule: $\mathrm{RC}_n=0$ indicates performance identical to the benchmark, whereas $\mathrm{RC}_n=1$ corresponds to the optimal decision. A value of $\mathrm{RC}$ closer to $1$ indicates that the decision rule performs better than the benchmark, while a negative value of $\mathrm{RC}$ indicates that the decision rule performs worse than the benchmark. In the numerical experiments, we use an unconditional SAA --- an SAA method for Problem~\eqref{eq: stochastic programming with covariates} without the condition $\{X=x_0\}$ --- as the benchmark and an oracle solution as the proxy for the optimal decision, yielding the equivalent empirical definition reported in Section~\ref{sec:numerical experiments}. This relative-cost measure has been widely used as a normalized performance metric for evaluating data-driven decision rules in Problem~\eqref{eq: stochastic programming with covariates}; see, e.g.,  \cite{bertsimas20} and \cite{bertsimas22}.

Let $F_{Y|X}(\cdot|x)$ denote the conditional distribution function of $Y$ given $X=x$. The conditional objective can then be written as
$$
h(z;x)=\int_{\mathcal{Y}} c(z;x,y)\,\mathrm{d}F_{Y|X}(y|x).
$$
Thus, the conditional distribution $F_{Y|X}(\cdot|x)$ determines the stochastic information relevant to the decision problem, while data-driven methods may extract and utilize this information at different levels. In particular, one may directly estimate the induced conditional objective $h(z;x)$, or instead learn the underlying conditional distribution $F_{Y|X}(\cdot|x)$ and subsequently construct the objective through generated samples.

These two strategies give rise to two broad paradigms for contextual decision making: \textit{a predictive approach}, which learns the conditional objective $h(z;x)$ and then optimizes it, and \textit{a generative approach}, which learns the conditional distribution $F_{Y|X}(\cdot|x)$ (or a conditional simulator) and then optimizes using generated scenarios. 
This distinction raises a fundamental question: \textit{How should we choose between the two paradigms?} In general, neither is uniformly preferable. 
Their relative pros and cons depend on both the statistical structure of the underlying problem and the computational properties of the resulting optimization problem. Accordingly, in the following sections, we study the predictive and generative paradigms separately and compare them along two central dimensions: \emph{convexity preservation}, which concerns whether the learned representation preserves convexity that facilitate the downstream optimization, and \emph{statistical efficiency}, which concerns how effectively each paradigm exploits the available offline data. Our analysis is intended to identify the problem characteristics under which each paradigm is particularly well suited and thereby provide insightful guidance on the choice between predictive and generative approaches.

\section{Predict-then-Optimize}
\label{sec:PtO}

Predict-then-optimize (PtO) is a widely used framework for contextual stochastic optimization \citep{bertsimas20,bertsimas22,sadana25survey}. Under this paradigm, the conditional objective function $h$ is first estimated from historical data $\mathcal D_n$ and then optimized for a realized covariate $x_0$. Most PtO methods approximate the conditional objective using a weighted empirical average of the observed costs. Specifically, the estimated objective takes the form
\begin{equation}
\hat h_n(z;x):=\sum_{i=1}^{n}\omega_{n,i}(x)\,c(z;x_i,y_i),
\label{eq: PtO objective}
\end{equation}
where the weights $\omega_{n,i}(x_0)$ are determined by the statistical procedure used to estimate the conditional expectation $h(z;x)$, and the decision rule is obtained by solving $\hat z_{n}(x_0)\in\arg\min_{z\in\mathcal Z}\hat h_n(z;x_0)$.

The estimator in Equation~\eqref{eq: PtO objective} is linear in the observed responses and encompasses a broad class of classical regression methods, both parametric and nonparametric.
Examples include ordinary least squares and ridge regression on the parametric side \citep{hastie2009}, and $k$-nearest-neighbor regression, kernel and local polynomial regression, smoothing splines, and kernel ridge regression on the nonparametric side \citep{wahba1990spline,hastie2009}. 
Within this broad class, we focus on nonparametric methods and, following \citet{bertsimas22}, distinguish between local and global methods.

\subsection{Local Methods: Convexity Preservation at the Cost of Statistical Efficiency}

Traditional local approximation methods, such as $k$-nearest neighbors, Nadaraya--Watson kernel smoothing, local averaging, classification and regression trees, and random forests, determines the weights $\omega_{n,i}(x_0)$ according to the proximity between the historical covariates $x_i$ and the realized covariate $x_0$; see \cite{wahba1990spline, hastie2009}. They estimate the conditional objective primarily by pooling information from observations that are locally relevant to $x_0$. Many commonly used local methods naturally generate nonnegative weights. Some exceptions, such as local linear methods, may produce signed weights, but can be modified to enforce nonnegativity \citep{bertsimas20}. 
From an optimization perspective, the property of nonnegative weights is important because it preserves the convexity of estimated objective $\hat{h}(\cdot;x_0)$ whenever the cost $c(\cdot;x,y)$ is convex. 

Although local methods therefore preserve convexity of the downstream optimization problem, their statistical efficiency is inherently limited by the curse of dimensionality. These methods rely on locality, assigning substantial weights only to observations close to $x_0$. As the covariate dimension increases, observations become increasingly sparse in local neighborhoods, making accurate local estimation progressively more difficult. For example, the PtO formulation based on Nadaraya–Watson kernel smoothing estimation with a spherical kernel achieves a decision regret rate of
        \[
           \E\Big[\E\big[ h(\hat z_n(X);X)-h(z^*(X);X) \mid \mathcal D_n \big]\Big]
            \,=\, \widetilde{\mathcal{O}}\!\left(n^{-\frac{1}{d_x+2}}\right), 
        \]
which deteriorates as the covariate dimension $d_x$ increases; see \cite{wang2024generalization}.

The curse of dimensionality underlying local methods can be understood through the usual bias--variance tradeoff. In high-dimensional covariate spaces, a small neighborhood around $x_0$ typically contains few observations, leading to high variance in the estimated conditional objective. Enlarging the neighborhood increases the effective sample size and reduces variance, but at the cost of incorporating observations farther from $x_0$, thereby increasing approximation bias when the conditional objective varies with the covariates. As $d_x$ increases, balancing these two sources of error becomes progressively more difficult, which leads to slower estimation of $h(z;x_0)$ and, consequently, slower convergence of the resulting decision regret.

\subsection{Global Methods: Statistical Efficiency at the Cost of Convexity Preservation}\label{sec: global-PtO}

To improve performance in high-dimensional settings, global regression methods based on reproducing kernel Hilbert spaces (RKHS) have been employed \citep{bertsimas22}. Unlike local methods, RKHS-based approaches determine the weights by solving a regularized optimization problem over a function space, which leads to a kernel ridge regression (KRR) estimator under squared loss. 
Specifically, for each fixed $z\in\mathcal Z$ and a given kernel function $K$, KRR estimates $h(z;\cdot)$ by
\[
\hat h_{n}(z;\cdot)\in\arg\min_{f\in\mathcal H}
\left\{
\frac{1}{n}\sum_{i=1}^n\big(f(x_i)-c(z;x_i,y_i)\big)^2
+\lambda_n\|f\|_{\mathcal H}^2
\right\},
\]
where $\lambda_n>0$ is the regularization parameter and $\mathcal H$ is the RKHS induced by the kernel function $K$. By the representer theorem \citep{wahba1990spline,scholkopf2001generalized}, the resulting estimator can be written as 
\begin{equation}\label{eq:KRRestimator}
\hat h_n(z;x_0)=\sum_{i=1}^n\omega^{\mathrm{KRR}}_{n,i}(x_0)c(z;x_i,y_i), 
\end{equation}
where the KRR weights $\omega^{\mathrm{KRR}}_{n}(x_0)=
(K_n+n\lambda_n I_n)^{-1} k_n(x_0)$, $K_n=(K(x_i,x_j))_{i,j=1}^n$, and $k_n(x_0)=(K(x_1,x_0),\ldots,K(x_n,x_0))^\top$. Thus, $\lambda_n$ controls the tradeoff between goodness of fit and regularization, while the resulting KRR weights are generally not restricted to be nonnegative.

This global approach allows the estimator to exploit information from the entire training sample rather than only observations in a local neighborhood. Consequently, it can leverage the smoothness of the conditional objective $h(z;\cdot)$ with respect to the covariates $x$ for each fixed decision $z$. When $h(z;\cdot)$ exhibits sufficiently high-order smoothness, RKHS-based methods can translate this regularity into faster convergence rates with substantially weaker dependence on the covariate dimension $d_x$ than local methods.

We instantiate the RKHS approach using the Mat\'ern kernel $K=K_\nu$ with smoothness parameter $\nu>0$, and set $s:=\nu+d_x/2$. To express the smoothness condition in terms of classical derivatives, we restrict attention to values of $\nu$ for which $s$ is an integer. It is well known that the RKHS $\mathcal H_\nu$ induced by $K_\nu$ is norm-equivalent to the Sobolev space $H^s(\mathcal X)$ \citep{wendland2004scattered}, equipped with the norm
\[
\|f\|_{H^s(\mathcal X)} = \left( \sum_{|\alpha|\leq s} \int_{\mathcal X} |\partial^\alpha f(x)|^2 \,\mathrm{d}x \right)^{1/2},
\]
where $\alpha=(\alpha_1,\ldots,\alpha_{d_x})$ is a multi-index with $|\alpha|=\alpha_1+\cdots+\alpha_{d_x}$. On the compact domain $\mathcal X$, any function that is $s$ times continuously differentiable belongs to $H^s(\mathcal X)$. The regularity required by the RKHS approach can therefore be interpreted as requiring sufficient smoothness of $h(z,x)$ with respect to $x$.

\begin{assumption}
\label{ass:krr_source}
There exists $B < \infty$ such that for every $z \in \mathcal{Z}$, 
$h(z; \cdot) \in H^s(\mathcal{X})$ with 
$\|h(z; \cdot)\|_{H^s(\mathcal{X})} \le B$.
\end{assumption}

The smoothness index $s$ in Assumption~\ref{ass:krr_source} determines the extent to which KRR can exploit the regularity of $h(z;\cdot)$ to improve statistical efficiency. The following result quantifies this effect for the weighted empirical PtO objective constructed using KRR weights, and the proof is provided in Appendix~\ref{app:proof of prop 1}.

\begin{proposition}
\label{prop:krr_matern_uniform_error}
Suppose Assumptions~\ref{ass:SAA-struct0}--\ref{ass:krr_source} hold with $s = \nu + d_x/2$ and $\nu>0$. Let $\hat{h}_n(z; x)$ be the weighted empirical PtO objective in Equation~\eqref{eq:KRRestimator} constructed using the KRR weights $\omega_{N,i}^{\mathrm{KRR}}(x)$ associated with the Mat\'{e}rn kernel $K_\nu$. Choosing the regularization parameter \( \lambda_n\asymp \log n /n \), we have
\[
\mathbb{E}\!\left[\mathbb{E}\!\left[h(\hat{z}_n(X); X) - h(z^\ast(X); X) \,\Big|\, \mathcal{D}_n\right]\right]  \,=\, \mathcal{O}\!\left(n^{-\frac{\nu}{2\nu + d_x}} (\log n)^{\frac12}\right).
\]
\end{proposition}

The convergence rate $n^{-\nu/(2\nu+d_x)}$ established in Proposition~\ref{prop:krr_matern_uniform_error} formalizes the statistical efficiency of KRR. As the Sobolev smoothness of the conditional expectation $h(z;\cdot) \in H^s(\mathcal{X})$ increases, i.e., as $s = \nu + d_x/2$ grows, the exponent $\nu/(2\nu+d_x)$ approaches $1/2$, thereby progressively weakening the effect of the ambient dimension $d_x$ on the convergence rate. In this sense, KRR-based PtO translates the smoothness of $h(z;\cdot)$ into improved statistical efficiency and mitigates the curse of dimensionality associated with local methods.

However, this gain in statistical efficiency comes at the cost of convexity preservation. Because KRR weights in Equation~\eqref{eq:KRRestimator} are obtained through kernel matrix inversion, they are generally sign-indefinite. The presence of negative weights implies that $\hat{h}_n(z;x)$ is no longer a nonnegative combination of the per-sample cost $c(z;X_i,Y_i)$. Consequently, even when each $c(z;\cdot,\cdot)$ is convex in $z$, the resulting empirical PtO objective in Equation~\eqref{eq:KRRestimator} need not remain convex, potentially complicating downstream optimization and increasing its computational burden (see Section~\ref{subsubsec:Effect of Negative Weights in PtO} for an example).

\citet{bertsimas22} observe that the negative KRR weights are typically small in magnitude and propose clipping them at zero to restore convexity. Specifically, each negative weight is replaced by zero, thereby converting the estimated objective into a nonnegative combination of the per-sample cost functions. However, they do not establish theoretical guarantees for the statistical performance of the resulting clipped KRR estimator. In the next subsection, we show that although this clipped KRR method restores convexity, it also reintroduces the curse of dimensionality.

\subsection{The Convexity--Efficiency Dilemma}\label{sec: dilemma}

The complementary strengths and limitations of local and global methods naturally raises a question: is it possible to design a prediction scheme for PtO that simultaneously preserve convexity and achieve statistical efficiency by overcoming the curse of dimensionality? Equivalently, can we construct a weighting scheme $\omega_{n,i}$ in Equation~\eqref{eq: PtO objective} that guarantees nonnegativity throughout the covariate space while exploiting high-order smoothness to achieve statistical efficiency in high-dimensional settings?

Our theoretical analysis in this subsection reveals that these two objectives are fundamentally mutually exclusive. 
To state the result formally, we adopt a minimax formulation over a class of smooth problem instances and a broad class of PtO procedures characterized only through their weighting rules.

\begin{assumption}[Problem Class for Minimax Lower Bound]
\label{ass:problem_class_pto}
Fix a nonnegative integer $\beta$. Let $\mathcal{P}_\beta$ denote a class of \emph{problem instances} $(P,c)$, where $P$ is a joint distribution of $(X,Y)$ supported on $\mathcal{X}\times\mathcal{Y}$ and $c:\mathcal{Z}\times\mathcal{X}\times\mathcal{Y}\to\mathbb{R}$ is a measurable cost function. For each $(P,c)\in\mathcal{P}_\beta$, define the induced conditional objective
$
h(z;x):=\E_{P}\!\left[c(z;X,Y)| X=x\right].
$
We assume the following conditions hold:
\begin{itemize}
\item[(i)] For every fixed $z\in\mathcal Z$, the map $x\mapsto h(z;x)$ is $\beta$-times continuously differentiable on $\mathcal X$, with all partial derivatives up to order $\beta$ uniformly bounded over $\mathcal X$, where $\beta=0$ means continuity.

\item[(ii)] The covariates are supported on the hypercube $\mathcal{X}=[0,1]^{d_x}$ and admit a density $p$ such that $0<p_{\min}\le p(x)\le p_{\max}<\infty$ for all $x\in\mathcal{X}.$

\item[(iii)] The conditional variance of $Y$ given $X=x$ is uniformly lower bounded, i.e., $\Var_{P}(Y| X=x)\ge \sigma^2$ for all $x\in\mathcal{X}$ and some constant $\sigma>0$.
\end{itemize}
\end{assumption}

\begin{assumption}[Admissible PtO Weights]
\label{ass:weights_pto}
A PtO algorithm is said to use admissible weights if, for every sample size $n$, every query point $x\in\mathcal{X}$, and every realization of $x_{1:n}$, it assigns weights $\{\omega_{n,i}(x)\}_{i=1}^n$ satisfying:
\begin{itemize}
\item[(i)] Each $\omega_{n,i}(x)$ depends only on $(x,x_{1:n})$, and does not depend on $y_{1:n}$.
\item[(ii)] $\omega_{n,i}(x)\ge 0$ for all $i=1,\dots,n.$
\item[(iii)] $\sum_{i=1}^n \omega_{n,i}(x) > 0$ for all $x \in\mathcal{X}.$
\end{itemize}
\end{assumption}

\begin{theorem}[Minimax Lower Bound for Convexity-Preserving PtO]
\label{thm:pto_lower_bound}
Consider the setting $d_y=d_z=1$ with $\mathcal Y=\mathcal Z=[-R,R]$ for some sufficiently large $R>0$. Fix a nonnegative integer $\beta$, let $\mathcal{P}_\beta$ be the problem class in Assumption~\ref{ass:problem_class_pto}. Consider any PtO algorithm that, given data $\mathcal{D}_n=\{(x_i,y_i)\}_{i=1}^n$ and a query covariate $x$, constructs the empirical objective as Equation~\eqref{eq: PtO objective}, where the weights satisfy Assumption~\ref{ass:weights_pto}, and outputs a decision \( \hat z_n(x)\in\arg\min_{z\in\mathcal{Z}} \hat h_n(z;x). \) Then there exists a constant $c_0>0$, independent of $n$, such that
\begin{equation}
\sup_{(P,c)\in\mathcal{P}_\beta} \E \left[ \E \left[ h(\hat z_n(X);X) - h(z^\ast(X);X) \,\big|\, \mathcal{D}_n \right] \right] \ \ge\ c_0\, n^{-\frac{4}{d_x+4}}.
\label{eq:PtO lower bound}
\end{equation}
\end{theorem}

Theorem~\ref{thm:pto_lower_bound} establishes a formal ``no-free-lunch'' result for PtO procedures using nonnegative, covariate-based weights satisfying Assumption~\ref{ass:weights_pto}. Even when the conditional objective is $\beta$-times continuously differentiable for arbitrarily large $\beta$, the worst-case expected regret of any such procedure cannot decay faster than $n^{-4/(d_x+4)}$.
This limitation already arises with a one-dimensional outcome and decision, demonstrating that these weighting schemes can fail to fully exploit high-order smoothness even in this simple setting.
This lower-bound rate coincides with that of second-order local methods and deteriorates rapidly as the covariate dimension $d_x$ increases. Thus, nonnegativity constitutes a fundamental statistical bottleneck: although it preserves convexity, it prevents the estimator from attaining the faster convergence rates associated with global smoothers over highly smooth function classes.

The proof constructs a finite adversarial subclass of $\mathcal P_\beta$, consisting of instances whose conditional objectives are quadratic and strongly convex in the decision variable, such that, for any PtO algorithm whose weights satisfy Assumption~\ref{ass:weights_pto}, at least one instance in this subclass incurs expected regret at least of order $n^{-4/(d_x+4)}$. Taking the maximum over this finite subclass and subsequently the supremum over $\mathcal P_\beta$ yields Equation~\eqref{eq:PtO lower bound}.  
When \(\nu=2\), the exponent in Proposition~\ref{prop:krr_matern_uniform_error}'s upper bound is half that in Theorem~\ref{thm:pto_lower_bound}'s lower bound. This difference may be explained by the different proof arguments. Proposition~\ref{prop:krr_matern_uniform_error} bounds regret by the uniform objective-estimation error without exploiting curvature, whereas Theorem~\ref{thm:pto_lower_bound} uses a quadratic construction in which regret is proportional to the squared decision error. Neither bound is guaranteed to be tight, and considering a more general adversarial subclass may yield a sharper lower bound. The complete proof of Theorem~\ref{thm:pto_lower_bound} is provided in Appendix~\ref{app: proof of theorem 1}.

To illustrate this dilemma, we consider a contextual newsvendor problem in which a retailer decides how much inventory to order each day. The resulting cost depends on both the order quantity \(z\) and the realized demand \(Y\). Specifically, the retailer incurs an overage cost when the order quantity exceeds demand and a shortage cost when demand exceeds the order quantity. At the time of ordering, the retailer does not observe demand \(Y\), but has access to some covariates \(X\), such as seasonal conditions, promotional activities, and weather, that may help predict demand. The detailed experimental setting, including the cost parameters, conditional demand distribution is provided in Section \ref{sec:contextual newsvendor}. Here, we use the structure-aware formulation established in Section \ref{sec:contextual newsvendor}, under which the objective function is infinitely differentiable in $x$. We fix the training sample size at $n=2,000$ and vary the covariate dimension \(d_x\). We evaluate decision performance using the relative cost defined in Equation~\eqref{eq: relative cost}, with unconditional SAA as the benchmark. 

Figure~\ref{fig:illustrative example} compares traditional local PtO methods with two variants of KRR-based PtO. PtO-KRR uses the original KRR weights, which can be negative, while PtO-KRR$^+$ truncates negative weights to zero and renormalizes the remaining weights to sum to one, following \citet{bertsimas22}. Traditional local methods, including $k$-nearest neighbors and kernel smoothing, use nonnegative weights and therefore preserve convexity. Their performance, however, deteriorates substantially as $d_x$ increases, illustrating the severe effect of the curse of dimensionality on local PtO methods. By contrast, PtO-KRR exploits the existing global smoothness and maintains relatively stable performance as $d_x$ increases. PtO-KRR$^+$ restores convexity through nonnegative weights but exhibits a similarly pronounced deterioration, indicating that clipping eliminates much of the high-dimensional advantage of the original KRR estimator. In summary, the theoretical and numerical results reveal a fundamental dilemma within standard PtO frameworks: preserving convexity through nonnegative weights and attaining high-dimensional statistical efficiency through high-order smoothness cannot, in general, be achieved simultaneously.

\begin{figure}[t]
    \FIGURE
    {\includegraphics[width=0.7\linewidth]{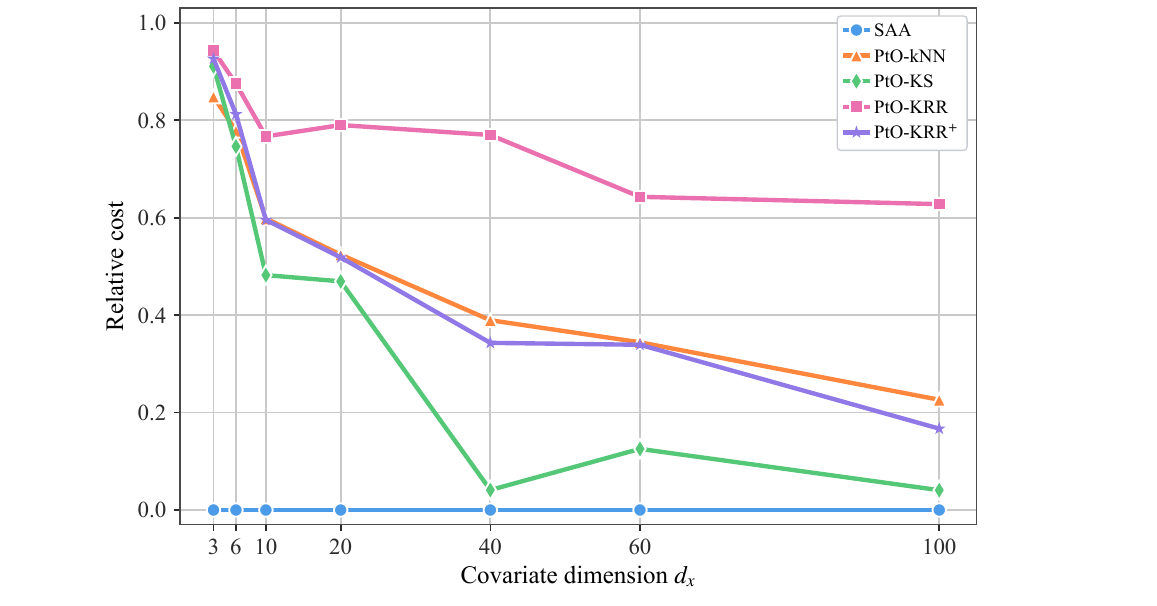}}
    {An illustrative example showing the performance of predictive approaches
    \label{fig:illustrative example}}
    {}
\end{figure}

\section{Generate-then-Optimize}
\label{sec:GtO}

We next introduce an alternative paradigm, Generate-then-Optimize (GtO), for solving the CSO problem~\eqref{eq: stochastic programming with covariates}.
GtO is a standard simulation-based method for solving stochastic optimization problems, which can be naturally extended to the CSO problem~\eqref{eq: stochastic programming with covariates}. Denote $F_{Y|X}(\cdot|x)$ as the conditional distribution of $Y$ conditioned on $X=x$. On the offline stage, we use the historical data $\mathcal{D}_n =\{(x_i,y_i)\}_{i=1}^n$ to learn the conditional distribution or a simulator for $Y|\{X=x\}$, denoted as $\hat{F}_{Y|X}(\cdot|x)$. On the online stage, when a new covariate $X=x_0$ arrives, we generate $m$ i.i.d. samples from $\hat{F}_{Y|X}(\cdot|x_0)$, denoted as $\hat{Y}_1,\hat{Y}_2,\ldots,\hat{Y}_m\sim \hat{F}_{Y|X}(\cdot|x_0)$. Then, we can use SAA to construct a sample objective function $\hat{h}$ and solve the simulation-based CSO problem accordingly as
\begin{equation}\label{eq: GtO opt}
	\hat{z}_{n,m}^\ast(x_0) \in \arg\min_{z \in \mathcal{Z}} \left\{ \hat{h}(z;x_0):= \frac{1}{m}\sum_{j=1}^m c\big(z;x_0,\hat{Y}_j\big) \right\}.
\end{equation}

\subsection{Convexity Preservation}\label{subsec:convexity-preservation}
Assumption~\ref{ass:SAA-struct0} supposes that the cost function $c(z;x,y)$ is convex in $z$ for each fixed $(x,y)$. The following proposition shows that GtO preserve the convexity structure of the cost function. 

\begin{proposition}\label{prop:convexity-preservation}
	If $c(\cdot;x,y)$ is convex on $\mathcal{Z}$ for all $(x,y)$, then for any realized simulator $\hat F_{Y|X}(\cdot| x_0)$, the SAA objective
	$\hat h(z;x_0) := \frac{1}{m}\sum_{j=1}^m c(z;x_0,\hat Y_j)$, where $\hat{Y}_1,\hat{Y}_2,\ldots,\hat{Y}_m\sim \hat F_{Y|X}(\cdot| x_0)$, is convex on $\mathcal{Z}$.
\end{proposition}

The proof of Proposition~\ref{prop:convexity-preservation} is immediate from the fact that nonnegative linear combinations of convex functions are convex. Proposition~\ref{prop:convexity-preservation} shows that, under GtO, the optimization problem~\eqref{eq: GtO opt} solved online is exactly a convex program with $\hat h(\cdot;x_0)$ being convex. Thus, the powerful machinery of convex optimization can be used without modification and classical stability results for convex stochastic programming apply; see, e.g., \citet{ShapiroDentchevaRuszczynski2009}.

\subsection{Conditional Generative Models for GtO}
\label{subsec:generative_models_gto}

The GtO framework is agnostic to the particular generative architecture used to construct $\hat{F}_{Y|X}(\cdot|x)$. Its essential requirement is a learned conditional distribution from which scenarios can be generated at a given covariate value. A conditional generator may be represented abstractly as $\hat{Y}=G_{\hat{\theta}}(x,\xi)$, where $\xi\sim P_{\xi}$, $G_{\hat{\theta}}$ is learned from $\mathcal{D}_n$, and the exogenous random variable $\xi$ supplies the randomness required for scenario generation. The distribution induced by $G_{\hat{\theta}}(x,\xi)$ defines $\hat{F}_{Y|X}(\cdot|x)$. This formulation encompasses both classical conditional distribution estimators that permit sampling and modern deep generative models. We focus on the latter.

A variety of conditional generative models can be incorporated into GtO. Representative examples include conditional variational autoencoders, which use latent variables to represent complex conditional distributions; conditional generative adversarial networks, which learn a conditional sampler through adversarial training; and conditional normalizing flows, which construct flexible conditional densities through invertible transformations. Autoregressive models can generate outcomes sequentially and are particularly useful for discrete, sequential, or otherwise structured outcomes. Transformer-based language models also belong to this class when outcomes are represented as sequences. We refer to \citet{bondtaylor2022deep} for a review of these model families. Conditional diffusion models provide another approach, generating scenarios through a learned reverse diffusion process \citep{yang2023diffusionsurvey}. Related flow-matching models learn continuous transformations from a reference distribution to the target distribution through a vector-field regression~objective.

We use conditional diffusion models (CDMs) as a benchmark implementation of GtO. They can flexibly represent high-dimensional and multimodal conditional distributions, and their denoising-based training avoids the adversarial minimax optimization used by conditional GANs. Recent studies also provide explicit statistical guarantees for conditional distribution estimation, making CDMs suitable for connecting our theoretical analysis with an implementable generative model \citep{fu24,tang2024conditional}. A potential limitation is the cost of generating samples through a discretized reverse-time process. Accelerated samplers and distillation methods can reduce this burden, although sampling cost remains an important consideration in downstream optimization; see \citet{yang2023diffusionsurvey} for a review. Accordingly, we use CDMs in the subsequent theoretical and numerical analyses, while other conditional generative models can be incorporated into the same GtO framework.

We briefly describe the CDM. Fix a covariate value $x$, and let $\{Y_t:t\geq0\}$ denote a forward Ornstein--Uhlenbeck process initialized from the target conditional distribution:
\[
\mathrm{d} Y_t=-\frac{1}{2}Y_t\,\mathrm{d}t+\mathrm{d}B_t, \quad \mbox{with $Y_0\sim F_{Y|X}(\cdot|x)$ and $t\in(0,T]$},
\]
where $B_t$ is a standard Brownian motion. Equivalently,
\(
Y_t=a_tY_0+b_t\varepsilon,
~
a_t=e^{-t/2},
~
b_t=\sqrt{1-e^{-t}},
~
\varepsilon\sim\mathcal{N}(0,I).
\)
Let $q_t(\cdot|x)$ denote the conditional density of $Y_t$ given $X=x$. The key object is the time-dependent conditional score \( s^\ast(y,x,t) := \nabla_y\log q_t(y|x), \) which is estimated from $\mathcal D_n$ by denoising score matching. The conditional and unconditional score functions can be estimated jointly by minimizing the empirical denoising score-matching risk \( \hat{s} \in \arg\min_{s\in\mathcal S} \frac{1}{n} \sum_{i=1}^n \ell(x_i,y_i;s), \) where 
\[ \ell(x_i,y_i;s) := \frac{1}{T-t_0} \int_{t_0}^{T} \E_{\substack{ \tau, \tilde Y_t }} \left[ \left\| s\big(\tilde Y_t,\tau(x_i),t\big)- \nabla_{\tilde y} \log\varphi_t(\tilde Y_t\mid y_i) \right\|_2^2 \right] \mathrm{d}t,
\]
where $\tilde Y_t\sim \mathcal N(a_ty_i,b_t^2I)$, $\tau\sim \mathsf{Unif}\{\varnothing,\mathrm{id}\}$, $\tau(x_i)=x_i$ when $\tau=\mathrm{id}$ and $\tau(x_i)=\varnothing$ when $\tau=\varnothing$, and $\varphi_t(\cdot\mid y_i)$ denotes the Gaussian transition density $\mathcal N(a_ty_i,b_t^2I)$. Notice that, since \( \nabla_{\tilde y} \log\varphi_t(\tilde y\mid y_i) = -\frac{\tilde y-a_ty_i}{b_t^2}, \) writing $\tilde Y_t=a_ty_i+b_t\varepsilon$ with $\varepsilon\sim\mathcal N(0,I)$ yields the equivalent regression target \( \nabla_{\tilde y} \log\varphi_t(\tilde Y_t\mid y_i) = -\frac{\varepsilon}{b_t}; \) The lower endpoint $t_0>0$ is an early-stopping time introduced to avoid the singular behavior of the score near $t=0$, and $\mathcal S$ is a unified class consisting of conditional and unconditional ReLU score networks; See \cite{yang2023diffusionsurvey} and \cite{fu24} for more learning details. Given the estimated conditional score, scenarios are generated from the reverse-time stochastic differential equation 
\begin{equation}
\mathrm{d} \bar Y_t = \left[ \frac{1}{2}\bar Y_t + \tilde{s}\big(\bar Y_t,x,T-t\big) \right]\mathrm{d}t + \mathrm{d}\bar B_t, \quad \mbox{with $\bar Y_0\sim\mathcal N(0,I)$ and $t\in(0,T-t_0]$},
\label{eq: reverse-time process}
\end{equation}
where $\bar B_t$ is another Brownian motion independent of $B_t$ and $\tilde{s}(y,x,t):=(1+\eta)\hat{s}(y,x,t)-\eta\hat{s}(y,\varnothing,t)$ with a guidance strength $\eta\geq 0$. For the distribution-estimation interpretation and the theoretical guarantees below, we take $\eta=0$; positive values of $\eta$ correspond to an optional classifier-free guidance heuristic. At the population level, with the exact conditional score $s^\ast$, the output distribution of the reverse-time process~\eqref{eq: reverse-time process} approaches $F_{Y| X}(\cdot| x)$ as the terminal time $T$ becomes large and the early-stopping time $t_0$ becomes small. In practice, using a finite $T$ and a positive $t_0$, replacing the exact score $s^\ast$ with an estimated score $\tilde{s}$, and numerically discretizing the reverse process yield the learned conditional distribution $\hat F_{Y| X}(\cdot| x)$ from which GtO scenarios are generated. Under suitable assumptions and parameter choices, \citet{fu24} and \citet{tang2024conditional} establish convergence guarantees for such CDM estimators.

The regret analysis developed below deliberately abstracts away from these architectural details. It requires only that the learned model support conditionally independent scenario generation and that the discrepancy between $\hat{F}_{Y|X}(\cdot|x)$ and $F_{Y|X}(\cdot|x)$ admit an appropriate statistical bound. The specific generative architecture affects the resulting regret bound only through the offline model error, while the online SAA error is governed by the number $m$ of generated scenarios. CDMs provide one concrete implementation for which theoretical distributional learning rates are available, but the GtO framework and the regret decomposition are not restricted to this particular model class.

\subsection{Statistical Efficiency of GtO}
Now, we analyze the regret 
\[
\mathcal{R}(x_0) := h\left(\hat{z}_{n,m}^\ast(x_0);x_0\right) - h\left(z^\ast(x_0);x_0\right) \ge 0.
\]
To isolate the impact of the model error from the SAA error, we define an auxiliary decision rule as
\begin{equation}\label{eqn:nosaa}
	\bar{z}_n^{\ast}(x_0) \in \arg\min_{z \in \mathcal{Z}} \Big\{ \bar{h}(z;x_0):=  \E\big[c(z;X,\hat{Y})\,|\,X = x_0 \big]\Big\},
\end{equation}
where $\hat{Y}\sim\hat{F}_{Y|X}(\cdot|x_0)$. Then we have the basic decomposition
\begin{align}
	&\mathcal{R}(x_0)  \nonumber \\
	&\,=\, \left[h(\hat{z}_{n,m}^\ast(x_0);x_0) - \bar h(\hat{z}_{n,m}^\ast(x_0);x_0)\right] + \left[\bar h(\hat{z}_{n,m}^\ast(x_0);x_0) - \bar h(\bar z_n^\ast(x_0);x_0)\right] + \left[\bar h(\bar z_n^\ast(x_0);x_0) - h(z^\ast(x_0);x_0)\right] \nonumber \\
	&\,\le\, \left[\bar h(\hat{z}_{n,m}^\ast(x_0);x_0) - \bar h(\bar z_n^\ast(x_0);x_0)\right] + 2 \sup_{z\in\mathcal{Z}} \big|h(z;x_0)-\bar h(z;x_0)\big|. \label{eq:error_decomposition0}
\end{align}
The first term in Equation~\eqref{eq:error_decomposition0} is the error induced by the SAA under the distribution $\hat{F}_{Y|X}(\cdot|x_0)$, while the second term in Equation~\eqref{eq:error_decomposition0} is the model error caused by the difference between $\hat{F}_{Y|X}(\cdot|x_0)$ and $F_{Y|X}(\cdot|x_0)$. We will discuss the two terms in Equation~\eqref{eq:error_decomposition0} separately, and then summarize their combined impact on the regret $\mathcal{R}(x_0)$.

\subsubsection{The SAA Error.}

We first analyze the contribution of the SAA approximation to the regret in Equation~\eqref{eq:error_decomposition0}. Fix the covariate value $x_0$, and treat the offline data $\mathcal{D}_n$ and the learned distribution or simulator $\hat{F}_{Y|X}(\cdot|x_0)$ as given. All expectations in this subsection are taken with respect to the simulation draws $\hat{Y}_1,\hat{Y}_2,\dots,\hat{Y}_m$ conditional on $\mathcal{D}_n$ and $X=x_0$.

\begin{lemma}\label{lem:SAA-error0}
	Suppose Assumption~\ref{ass:SAA-struct0} holds. Fix a covariate value $x_0$ and a realization of the learned conditional distribution or simulator $\hat{F}_{Y|X}(\cdot|x_0)$.
	Then there exists a constant $C_{\mathrm{SAA}}<\infty$, depending only on $(G,M,D,d_z)$, such that for all $m\ge 2$,
	\begin{equation}\label{eq:SAA-regret-bound}
		\E\left[ \bar h(\hat z_{n,m}^\ast(x_0);x_0) - \bar h(\bar z_n^\ast(x_0);x_0) \,\big|\, \mathcal{D}_n,X=x_0 \right] \;\le\; C_{\mathrm{SAA}}\,m^{-1/2}.
	\end{equation}
\end{lemma}

Lemma~\ref{lem:SAA-error0} shows that, once the learned distribution or simulator $\hat{F}_{Y|X}(\cdot|x_0)$ is fixed, the SAA-based decision $\hat{z}_{n,m}^\ast$ is nearly optimal for the population objective $\bar{z}_n^{\ast}(x_0)$, with regret decaying at the typical rate $m^{-1/2}$ which comes from the fact that $\hat{h}$ is an empirical average of $m$ i.i.d.\ samples. The proof follows the standard SAA/M-estimation pattern: obtain a uniform bound on the empirical objective and then transfer it to the regret; see, e.g., \citet[Chapter~5]{vanderVaart1998asymptotic} and \citet[Chapter~5]{ShapiroDentchevaRuszczynski2009}. The existing SAA results typically provide almost-sure or in-probability convergence of regrets and solutions, and asymptotic distributions; see, e.g., \citet{ShapiroDentchevaRuszczynski2009,BanholzerFliegeWerner2022}. Our result here is a non-asymptotic $L^1$-error bound for the SAA regret with explicit $m^{-1/2}$ dependence.
We provide the proof in Appendix~\ref{app: proof of lemma 1}.

\subsubsection{The Model Error.}

We now turn to the contribution of model misspecification, i.e., the error caused by using the learned conditional distribution $\hat{F}_{Y|X}$ instead of the true distribution $F_{Y|X}$. 

\begin{assumption}\label{ass: TV-rate}
	There exist constants $C_{\mathrm{M}}>0$, $\alpha_0>0$, and $\alpha_1\ge 0$ such that for all $n\geq 2$,
	\begin{equation}\label{eq:TV-rate}
\E\Big[\E\Big[\mathrm{TV}\Big(\hat{F}_{Y|X}(\cdot|X),F_{Y|X}(\cdot|X)\Big) \,\big|\, \mathcal{D}_n \Big]\Big]
		\;\le\;
		C_{\mathrm{M}}\,n^{-\alpha_0}(\log n)^{\alpha_1},
	\end{equation}
	where $\mathrm{TV}$ denotes the total variation distance, the inner expectation is with respect to the covariate $X$, and the outer expectation is with respect to the training data $\mathcal{D}_n$.
\end{assumption}

Assumption~\ref{ass: TV-rate} requires that the learned conditional distribution $\hat{F}_{Y|X}(\cdot|x)$ converges to the true conditional distribution $F_{Y|X}(\cdot|x)$ in total variation, on average over the covariate distribution, at some explicit rate $n^{-\alpha_0}$ up to logarithmic factors. Such rates are available in both parametric and nonparametric settings. In parametric conditional density models with a finite-dimensional parameter estimated at the usual $n^{-1/2}$ rate \citep{vanderVaart1998asymptotic}, one can show that Assumption~\ref{ass: TV-rate} holds with $\alpha_0=1/2$ and $\alpha_1=0$ under some regularity conditions.
In nonparametric conditional density estimation, minimax rates in integrated $L_1$ (and hence total variation) risk of the form $n^{-\beta/(d_x+d_y+2\beta)}$ with log factors are attainable under $\beta$-H\"older smoothness assumptions; see, e.g., \citet{Tsybakov2009,li2022minimax,fu24}. 

\smallskip
\begin{remark}\label{rem: TV bound}
\textit{
Recently, \citet{fu24} study the sample complexity of using CDMs to estimate the conditional distribution $F_{Y|X}(\cdot|x)$ from data. Under their baseline conditions, the conditional density $p_{Y|X}(y|x)$ satisfies a uniform Gaussian-type upper-tail bound and $\beta$-H\"older smooth (i.e., belongs to a $\beta$-H\"older ball). Together with a uniform bound on its Kullback--Leibler divergence from the standard Gaussian distribution and appropriate choices of the early-stopping and terminal times, their result yields
\[
\alpha_0 = \frac{\beta}{4(d_x+d_y+\beta)} \quad \text{and} \quad  \alpha_1 = \max\left\{ 9,\, \frac{d_y}{2}+\frac{\beta}4+1 \right\}.
\]
Furthermore, under their stronger conditional-density condition, \( p_{Y|X}(y|x) = \textrm{exp}(-C_2\|y\|_2^2/2) \cdot f(y,x), \) where $f$ is $\beta$-H\"older smooth and is uniformly bounded away from zero, the rate improves to
\[
\alpha_0 = \frac{\beta}{d_x+d_y+2\beta} \quad \text{and} \quad \alpha_1 = \max\left\{ \frac{19}2,\, \frac{\beta+2}2 \right\}.
\]
Consequently, for fixed $d_x$ and $d_y$, the convergence rate $\alpha_0$ approach $1/2$ as $\beta\to\infty$. Similarly, \citet{tang2024conditional} establish a related rate under compact-support and smooth conditional-density assumptions, while allowing the smoothness levels in the covariate and response directions to differ. Assumption~\ref{ass: TV-rate} therefore summarizes the statistical accuracy of the offline conditional-distribution learning step in a form directly applicable to our regret analysis.
}
\end{remark}

\begin{lemma}\label{lem:model-error0}
	Suppose Assumptions~\ref{ass:SAA-struct0} and~\ref{ass: TV-rate} hold. Then
	\begin{equation}\label{eq:model-error-bound-regret}
		\E\Big[\E\Big[\sup_{z\in\mathcal{Z}}|h(z;X) - \bar{h}(z;X) | \,\big|\, \mathcal{D}_n \Big]\Big]
		\;\le\; 2M\,C_{\mathrm{M}}\,n^{-\alpha_0}(\log n)^{\alpha_1},
	\end{equation}
	where the inner expectation is with respect to $X$ and the outer expectation is with respect to $\mathcal{D}_n$.
\end{lemma}

Lemma~\ref{lem:model-error0} shows that, under the average total variation convergence of Assumption~\ref{ass: TV-rate}, the regret due purely to model misspecification (i.e., optimizing under $\hat{F}_{Y|X}(\cdot|x)$ instead of $F_{Y|X}(\cdot|x)$) decays at the same rate as the simulator itself in total variation, up to constants depending on the cost envelope $M$. We arrange the proof in Appendix~\ref{app: proof of lemma 2}.

\smallskip
\begin{remark}[Alternative Distributional Discrepancies]
\label{rem:alternative-discrepancies}
\textit{
Assumption~\ref{ass: TV-rate} is stated in terms of total variation because, under the uniform boundedness condition $|c(z;x,y)|\le M$, total variation directly controls the discrepancy between the true and learned conditional objectives as shown in Appendix~\ref{app: proof of lemma 2}. More generally, the same analysis applies to a probability metric or statistical divergence $\rho$ whenever there exist constants $L_\rho<\infty$ and $\gamma\in(0,1]$ such that
\begin{equation}
\sup_{z\in\mathcal Z} \left| \E_{Y\sim P}[c(z;x,Y)] - \E_{Y\sim Q}[c(z;x,Y)] \right| \,\le\, L_\rho \rho(P,Q)^\gamma
\label{eq: bound with other discrepancy}
\end{equation}
for all relevant distributions $P$ and $Q$. Furthermore, if
\begin{equation}
\E \Big[ \E \Big[ \rho\Big( \hat F_{Y|X}(\cdot|X), F_{Y|X}(\cdot|X) \Big) \,\big|\, \mathcal D_n \Big] \Big] \,\le\, C_{\mathrm M}n^{-\alpha_0}(\log n)^{\alpha_1},
\label{eq: Distributional learning rates}
\end{equation}
then the proof of Lemma~\ref{lem:model-error0} and Jensen's inequality yield
\[
\E\Big[\E\Big[\sup_{z\in\mathcal{Z}}|h(z;X) - \bar{h}(z;X) | \,\big|\, \mathcal{D}_n \Big]\Big] \,\le\, L_\rho C_{\mathrm M}^{\gamma} n^{-\gamma\alpha_0} (\log n)^{\gamma\alpha_1}.
\]
}

\textit{
For bounded costs, Condition~\eqref{eq: bound with other discrepancy} holds with $(L_\rho,\gamma)=(2M,1)$ for total variation, $(2\sqrt{2}M,1)$ for the Hellinger distance, and $(\sqrt{2}M,1/2)$ for the Kullback--Leibler divergence. If $c(z;x,\cdot)$ is uniformly $L_y$-Lipschitz, it also holds with $(L_\rho,\gamma)=(L_y,1)$ for the $p$-Wasserstein distance with $p\ge1$. }

\textit{
Distributional learning rates~\eqref{eq: Distributional learning rates} under these discrepancies are available for several of the conditional generative models discussed in Section~\ref{subsec:generative_models_gto}. In particular, explicit total-variation and Wasserstein rates have been established for CDMs \citep{fu24,tang2024conditional}. For adversarial generators, \citet{liu2021wasserstein} establish nonasymptotic distributional guarantees for a Wasserstein-based conditional generator, which can be viewed as a conditional WGAN-type approach. For latent-variable models, \citet{kumar2025likelihood} derive Hellinger and Wasserstein convergence rates for a likelihood-based conditional deep generative model and implement the estimator using a conditional VAE architecture. Closely related to flow-matching models, \citet{huang2023conditional} establish nonasymptotic guarantees for conditional stochastic interpolation, which constructs conditional samplers through learned probability-flow ODEs or stochastic differential equations. Comparable distributional learning rates are, to our knowledge, less developed for general conditional normalizing flows and Transformer-based autoregressive generators. Importantly, these guarantees apply to the specific estimators and regularity conditions studied in the respective works, rather than uniformly to the corresponding generative-model architecture classes.
}
\end{remark}

\subsubsection{Regret Bound for GtO.}

We now combine the SAA regret bound (Lemma~\ref{lem:SAA-error0}) and the model regret bound (Lemma~\ref{lem:model-error0}) to obtain an overall bound on the regret of the GtO procedure.
Taking expectations over the randomness in both the training data and the online covariate on both sides of Equation~\eqref{eq:error_decomposition0}, and applying Lemmas~\ref{lem:SAA-error0} and~\ref{lem:model-error0}, yields the desired regret bound.

\begin{theorem}\label{thm:regret-GtO0}
	Suppose Assumptions~\ref{ass:SAA-struct0} and~\ref{ass: TV-rate} hold.
	Then there exists a constant $C<\infty$, depending only on $(G,M,D,d_z,C_{\mathrm{M}},\alpha_0,\alpha_1)$, such that for all integers $n,m\ge 2$,
	\begin{equation}\label{eq:regret-bound-GtO}
		\E\Big[\E\Big[h(\hat{z}_{n,m}^\ast(X);X) - h(z^\ast(X);X) \,\big|\, \mathcal{D}_n \Big]\Big] \;\le\; C\left( n^{-\alpha_0}(\log n)^{\alpha_1} + m^{-1/2}
		\right),
	\end{equation}
	where, conditional on $\mathcal D_n$, the inner expectation is with respect to the online covariate $X$ and the conditional simulation draws, and the outer expectation is with respect to $\mathcal{D}_n$.
\end{theorem}

Theorem~\ref{thm:regret-GtO0} provides a clean decomposition of the average regret of the GtO procedure into: an \emph{offline learning error} term, $n^{-\alpha_0}(\log n)^{\alpha_1}$, which reflects how accurately the conditional distribution $F_{Y|X}(\cdot|x)$ can be learned from $n$ samples in total variation, and an \emph{online SAA error} term, $m^{-1/2}$, which reflects the Monte Carlo error from using $m$ simulated samples at decision time.
In the situation where $\alpha_0=1/2$ (Remark~\ref{rem: TV bound}), this yields the canonical rate
\[
\E\Big[\E\Big[h(\hat{z}_{n,m}^\ast(X);X) - h(z^\ast(X);X) \,\big|\, \mathcal{D}_n \Big]\Big] \,=\, \mathcal{O}\Big(n^{-1/2}(\log n)^{\alpha_1} + m^{-1/2} \Big) \,=\, \widetilde{\mathcal{O}}\Big(n^{-1/2} + m^{-1/2}\Big).
\]
Theorem~\ref{thm:regret-GtO0} makes explicit the trade-off between investing in more training data (larger $n$) and using more simulation at decision time (larger $m$). In particular, if one balances the two error terms by choosing $m$ as a suitable function of $n$, the overall regret can be minimized. We arrange the proof in Appendix~\ref{app: proof of theorem 2}.

\smallskip
\begin{remark}\label{remark: strongly convex}
\textit{
Furthermore, if the cost function $c(z;x,y)$ is \textit{strongly convex} on $\mathcal{Z}$, then we can improve the result~\eqref{eq:regret-bound-GtO} to 
	\begin{equation}\label{eq:decision-error-bound0}
		\E\Big[\E\Big[\|\hat{z}_{n,m}^\ast(X) - z^\ast(X)\| \,\big|\, \mathcal{D}_n \Big]\Big]
		\;\le\;
		C\left(
		n^{-\alpha_0}(\log n)^{\alpha_1}
		+
		m^{-1/2}
		\right), 
	\end{equation}
for some constant $C$ depending on the setting of strongly convex cost function; See Appendix~\ref{app:stronglyconvex} for the detailed discussion. 
}
\end{remark}

\section{Choosing Between Predictive and Generative Approaches}\label{sec:PorG}

Our analysis yields three managerial insights for choosing between PtO and GtO, concerning the statistical difficulty of the learning target, the optimization structure retained after learning, and the range of decisions the learned model will support. These insights explain why neither approach is uniformly preferable and provide practical guidance on where an organization can invest to improve its decision process.

\subsection{Learning Less Is Not Always Easier}

PtO learns the conditional expected objective $h(z;x)$, whereas GtO learns the conditional uncertainty underlying that objective. Although the former is a narrower target, its statistical difficulty depends on its smoothness. To illustrate this distinction, consider the specific guarantees established earlier for KRR-based PtO and CDM-based GtO. Under suitable regularity conditions, suppressing logarithmic factors and the finite-scenario approximation error gives
\begin{eqnarray*}
    \mathbb{E}\left[ \mathcal{R}_{\mathrm{PtO}} \right]\lesssim n^{-\alpha_P} \quad &\text{where}& \quad \alpha_P=\beta_h/(2\beta_h+d_x),\\
    \mathbb{E}\left[ \mathcal{R}_{\mathrm{GtO}} \right]\lesssim n^{-\alpha_G} \quad &\text{where}& \quad \alpha_G=\beta_F/(2\beta_F+d_x+d_y).    
\end{eqnarray*}
Here, $\beta_h=s_h-d_x/2>0$ represents the expected-objective smoothness parameter in the KRR bound, and $\beta_F$ is the conditional-density smoothness parameter in the CDM bound. 

These expressions reveal a tradeoff between dimension and smoothness. GtO bears the additional dimension $d_y$, but its learning target can be smoother than the expected objective. PtO avoids that additional dimension, while its convergence rate also depends on the regularity of the expected objective. Thus, identifying the smaller learning target is not enough; one must also identify which target is easier to learn.

\subsection{Statistical Flexibility Has an Optimization Consequence}

Convexity is an important source of tractability in optimization, but whether it survives learning depends on how the estimated objective is constructed. Both PtO and GtO use weighted combinations of scenario costs:
\[
\hat h_P(z;x)=\sum_{i=1}^{n}\omega_{n,i}(x)\, c(z;X_i,Y_i),
\qquad
\hat h_G(z;x)=\frac{1}{m}\sum_{j=1}^{m}c(z;x_0,\tilde Y_j).
\]
When $c(\cdot;x,y)$ is convex for every $(x,y)$, nonnegative weights preserve convexity. GtO uses equal positive weights and therefore inherits this property. PtO, however, faces the dilemma established by our earlier lower-bound result: within the class of nonnegative, covariate-based weighting rules studied there, worst-case expected regret cannot decay faster than $n^{-4/(d_x+4)}$, even for arbitrarily smooth objectives in the constructed problem class. Thus, nonnegative weighting preserves convexity but limits the statistical gains available from additional smoothness. More flexible regression can exploit that smoothness, but may introduce negative weights and lose the general convexity guarantee.

Negative weights do not necessarily make the fitted objective nonconvex, nor does nonconvexity necessarily make it difficult to optimize. Their computational effect depends on the problem's structure. The numerical study illustrates this distinction: the signed weighted newsvendor objective in Section~\ref{sec:contextual newsvendor} remains efficiently solvable, whereas signed weights can invalidate the usual linear programming reformulation of the product placement problem in Section~\ref{sec:product-placement}. Thus, the relevant concern is whether the fitted objective retains the structure needed by the solution method.

Convexity alone also does not determine computational cost. For GtO, the bound $\mathbb{E}[\mathcal{R}_{\mathrm{GtO}}]\lesssim\delta_n+m^{-1/2}$ highlights the role of scenario count: increasing $m$ reduces scenario approximation error but increases the scenario generation cost and enlarges the optimization problem, without eliminating model error $\delta_n$. PtO's computational cost likewise depends on its fitted representation and the available solver. The two approaches should therefore be compared by the decision quality they deliver within the available computational budget.

\subsection{The Appropriate Learning Target Depends on the Scope of Reuse}

The preceding discussion focuses on a single decision problem. An organization may instead face a collection of tasks that share the same conditional uncertainty:
\[
\min_{z\in\mathcal Z_k}h_k(z;x), \qquad h_k(z;x)=\mathbb{E}[c_k(z;X,Y)\mid X=x], \qquad k=1,\ldots,K.
\]
The costs and constraints may differ across tasks, even though the relevant uncertain outcomes are the same.

This distinction changes the value of what is learned. A model tailored to $h_k$ can focus its statistical capacity on one task. A conditional generator can support multiple tasks through $\hat h_{G,k}(z;x)=m^{-1}\sum_{j=1}^{m}c_k(z;x,\tilde Y_j)$. Once generated, the same scenarios can be evaluated under different objectives and constraints. Learning a richer object can therefore have value beyond its performance on any single task.

This potential benefit is particularly relevant when data analysis and operations reside in different departments. A shared generator provides an interface through which the data team supplies conditional scenarios and operational teams formulate their own decision problems. Changes in an objective need not require redesigning the learning target, provided the modeled uncertainty remains appropriate. This can reduce the coordination required to adapt analytics to new uses.

Reuse is nevertheless not exclusive to GtO. PtO methods can share weights, features, or estimated components across tasks. Nor does a generator remain valid automatically when a new task changes the underlying uncertainty --- for example, when a pricing decision changes demand. The practical question is how much additional learning, validation, and coordination each new use requires.

Accordingly, the comparison should reflect the intended scope of deployment. For a stable, specialized task, exploiting objective-specific structure may be especially valuable. For a collection of evolving tasks with common uncertainty, a reusable generative model may justify its greater statistical and computational demands. The single-task regret bounds clarify those demands; they do not quantify the organizational benefits of reuse.

\subsection{Practical Rules of Thumb}

Taken together, these insights suggest choosing between PtO and GtO by considering what must be learned, how the resulting decision problem will be solved, and how broadly the model will be used. PtO's specialization is valuable when it reduces the learning burden while retaining a tractable decision problem. GtO's broader representation is valuable when regularity in the uncertainty, preservation of optimization structure, or reuse across tasks compensates for its additional learning and computational demands. Table~\ref{tab:pto-gto-guidance} summarizes these considerations as practical rules of thumb. When the factors point in different directions, the final choice should be guided by decision quality and computational cost in the intended application.

\begin{table}[htbp]
\centering
\caption{Practical guidance for choosing between PtO and GtO.}
\label{tab:pto-gto-guidance}
\small
\begin{tabular}{p{0.29\textwidth}p{0.25\textwidth}p{0.34\textwidth}}
\hline
Factor & Low & High \\
\hline
Importance of convexity & No clear preference & GtO or convexity-preserving PtO \\
Objective smoothness & GtO & PtO \\
Distribution smoothness & PtO & GtO \\
Covariate dimension & No clear preference & Smoothness-dependent \\
Outcome dimension & No clear preference & PtO \\
Task diversity & PtO & GtO \\
Importance of online speed & No clear preference & PtO \\
\hline
\end{tabular}
\par\smallskip
\begin{minipage}{0.95\textwidth}
\footnotesize
\textit{Note.} Preferences indicate tendencies, all else equal, rather than guaranteed performance rankings. Objective smoothness refers to the components that must be learned; task diversity assumes shared underlying uncertainty. Smoothness and dimension comparisons are subject to the assumptions of the corresponding theoretical results. Online speed includes both conditional scenario generation and optimization.
\end{minipage}
\end{table}

\section{Numerical Study}
\label{sec:numerical experiments}

We compare PtO and GtO in three settings. A contextual newsvendor experiment examines the roles of training sample size, covariate dimension, and the structure of the learning target. A product placement experiment examines the computational consequences of signed prediction weights and the effect of outcome dimension. A fresh-retail inventory study evaluates the methods with real data at different levels of pooling. Together, the experiments compare decision quality and the computation required to obtain those decisions.

\subsection{General Setup}

The PtO benchmarks are $k$-nearest neighbors (KNN), kernel smoothing (KS), kernel ridge regression (KRR), and clipped KRR (KRR$^+$), which replaces negative weights by zero and renormalizes the remaining weights. For $n\geq10{,}000$, KRR and clipped KRR use a Nystr\"om approximation \citep{williams2001nystrom}; the table headings retain the KRR designation. In the synthetic experiments, GtO uses a conditional diffusion model (GtO-CDM) \citep{fu24}. Unconditional SAA and an oracle based on the true data-generating process provide benchmarks. Hyperparameters are selected by validation decision cost and fixed before test~evaluation. 

Synthetic decision performance is measured by the relative cost in Section~\ref{sec:problem}: $\mathrm{RC}=(C_{\mathrm{SAA}}-C)/(C_{\mathrm{SAA}}-C_{\mathrm{oracle}})$. A larger value is better, with zero corresponding to SAA and one to the oracle. The retail study instead reports cost relative to SAA because the true conditional demand distribution is unavailable. Online time excludes model training and includes query-specific computations: prediction weights and optimization for PtO, and conditional scenario generation and optimization for GtO.

All experiments are implemented in Python 3.12.4 and run on the University of Minnesota Supercomputing Institute cluster with NVIDIA A40 GPUs and on Nanjing University servers with NVIDIA H100 GPUs. Gurobi 13.0.1 solves the product placement linear programs (LPs) and the convex subproblems in the difference-of-convex algorithm (DCA). Online time starts when a new context is observed and excludes model training. It includes computing prediction weights and solving the decision problem for PtO, and generating conditional scenarios and solving the decision problem for GtO.

\subsection{Contextual Newsvendor Problem}
\label{sec:contextual newsvendor}

We begin with a contextual newsvendor problem, in which a retailer uses observed covariates to determine an order quantity before demand is realized. A distinctive feature of our setting is a shortage penalty that changes discontinuously with the covariates, even though the conditional demand model varies smoothly. This allows us to examine how the comparison between PtO and GtO depends on whether PtO directly learns the objective or exploits its known structure. Details of the demand and cost construction are provided in Appendix~\ref{app:num-newsvendor}.

\subsubsection{Experimental Design.}
\label{sububsec:NP Problem setup}
After observing covariates $X=x_0$, the decision maker chooses an order quantity to solve
\begin{equation}
\min_{z\geq0}\;\mathbb{E}\left[c_h(z-Y)^++b(X)(Y-z)^+\mid X=x_0\right],
\label{eq:NP problem}
\end{equation}
where $c_h=90$. Conditional demand follows a three-component Gaussian mixture censored at zero. Covariates affect the component means, probabilities, and variances; average demand and signal-to-noise ratios are calibrated to remain comparable across dimensions. The known shortage penalty $b(x)$ takes 11 levels between 30 and 120 and is a step function of a covariate-based priority score. 

We vary $n\in\{800,2000,5000,10000,20000,30000\}$ and $d_x\in\{3,6,10,20,40,60,100\}$, holding $d_y=1$. Each setting uses 100 test contexts and 30 independent replications. To distinguish direct objective fitting from exploitation of known structure, we implement PtO in two ways:
\[
\begin{aligned}
\hat h_n^{\mathrm{direct}}(z;x_0)
&=\sum_{i=1}^n\omega_{n,i}(x_0)\left[c_h(z-y_i)^++b(x_i)(y_i-z)^+\right],\\
\hat h_n^{\mathrm{struct}}(z;x_0)
&=\sum_{i=1}^n\omega_{n,i}(x_0)\left[c_h(z-y_i)^++b(x_0)(y_i-z)^+\right].
\end{aligned}
\]
The direct formulation estimates the conditional expected objective from historical costs, including the discontinuities in $b(x)$. The structure-aware formulation estimates expected overage and shortage separately and applies the known penalty at the current context. The two learning targets of this formulation, $\E[(z-Y)^+| X=x]$ and $\E[(Y-z)^+| X=x]$, are both smooth in $x$ under the demand model, whereas the overall objective learned by the direct formulation is discontinuous. This comparison therefore examines the value of learning a smoother objective in PtO by exploiting known cost structure while holding the underlying decision problem fixed.

\subsubsection{Decision Quality and Computation.}

Figure~\ref{fig:newsvendor-regime-map} compares GtO with the best-performing PtO benchmark in each setting. Each cell reports the difference in relative cost, GtO minus the best PtO; positive values favor GtO. The best PtO score summarizes the benchmark results rather than specifying an additional method for deployment. Tables~\ref{tab:newsvendor-relative-cost-direct} and~\ref{tab:newsvendor-relative-cost} in Appendix~\ref{app:num-newsvendor-results} report the individual methods' scores for the direct and structure-aware formulations, respectively. The GtO scores are identical across the two~formulations.

\begin{figure}[htbp]
\FIGURE
{
\includegraphics[width=\textwidth]{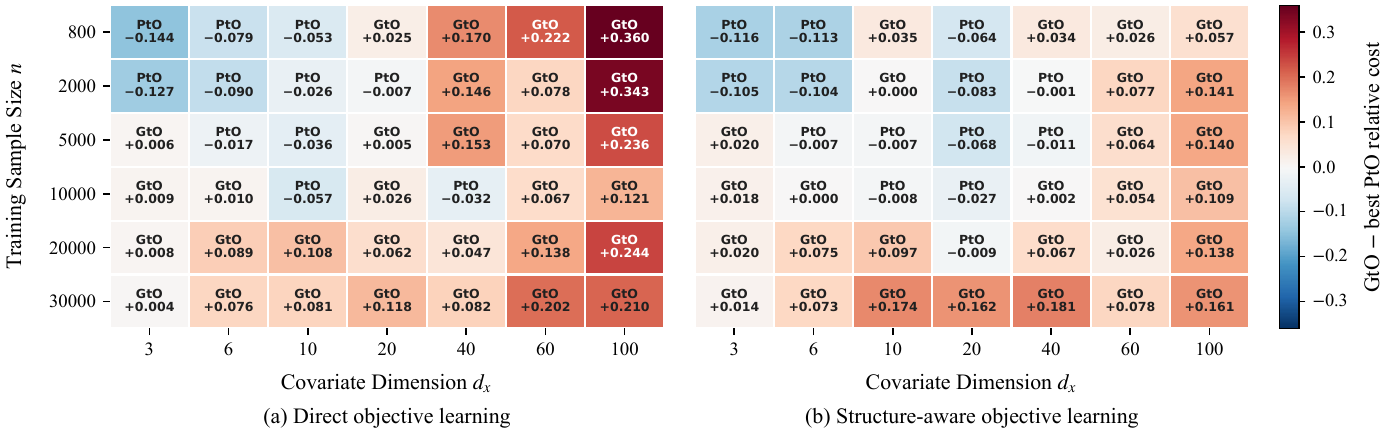}
}
{Relative performance of PtO and GtO across $n$ and $d_x$ under the direct and structure-aware PtO formulations
\label{fig:newsvendor-regime-map}}
{}
\end{figure}

Three patterns emerge. First, PtO is competitive with limited training data, while GtO's relative performance improves as $n$ increases. At $n=800$, the best PtO benchmark outperforms GtO in three of the seven covariate dimensions under each formulation; at $n=2000$, it does so in four of seven. At $n=20{,}000$, GtO performs better in all seven dimensions under direct learning and six of seven under structure-aware learning. At $n=30{,}000$, GtO performs better throughout. Thus, the larger-sample results favor GtO, although the ranking at smaller sample sizes depends on the context dimension and formulation.

Second, increasing the covariate dimension generally reduces decision performance, with a larger decline for PtO in several settings. At $n=30{,}000$, increasing $d_x$ from 3 to 100 reduces the best direct PtO score from 0.9899 to 0.6564 and the best structure-aware PtO score from 0.9804 to 0.7060, compared with a decline from 0.9943 to 0.8665 for GtO. The pattern is not monotonic in every setting, but these results illustrate the difficulty of high-dimensional objective fitting even when known cost structure is incorporated. They do not imply that high covariate dimension intrinsically favors GtO.

Third, incorporating objective structure changes the comparison, but does not uniformly improve finite-sample performance. At $n=5000$ and $d_x=20$, KRR's relative cost increases from 0.7670 under direct learning to 0.8404 under structure-aware learning, reversing its ranking against GtO's 0.7723. A similar reversal occurs at $d_x=40$, where KRR increases from 0.6402 to 0.8048, compared with GtO's 0.7935. Across all seven dimensions at this sample size, PtO leads in two settings under direct learning and four under structure-aware learning. However, at $n=30{,}000$ and $d_x=10$, KRR's score falls from 0.8993 to 0.8062. Exploiting the smoothness of the learning targets therefore provides an opportunity to improve PtO, rather than a guarantee of better decisions in every setting.

The gains in decision quality can require substantially more online computation. Table~\ref{tab:newsvendor-online-summary} reports the structure-aware results at $n=10{,}000$: PtO takes approximately 0.013 to 0.042 seconds per decision, whereas GtO takes 0.76 to 2.80 seconds. The empirical objectives are one-dimensional and piecewise linear, allowing breakpoint-based optimization even with signed weights when the objective is bounded below. Conditional scenario generation adds computation for GtO. The complete timing grids in Tables~\ref{tab:runtime-sec-direct} and~\ref{tab:runtime-sec-structure-aware} show that this latency difference persists under both formulations. Thus, the more favorable GtO decision performance in many settings comes with a higher online computational cost.

\begin{table}[htbp]
\TABLE
{Online time for the structure-aware newsvendor formulation at $n=10{,}000$ (seconds).
\label{tab:newsvendor-online-summary}}
{
\small
\begin{tabular}{@{}crrrrrr@{}}
\toprule
$d_x$ & SAA & PtO-KNN & PtO-KS & PtO-KRR & PtO-KRR$^+$ & GtO-CDM \\
\midrule
3   & 0.0214 & 0.0392 & 0.0331 & 0.0338 & 0.0251 & 0.7556 \\
6   & 0.0136 & 0.0268 & 0.0208 & 0.0294 & 0.0199 & 1.0644 \\
10  & 0.0221 & 0.0350 & 0.0373 & 0.0405 & 0.0285 & 1.3747 \\
20  & 0.0139 & 0.0218 & 0.0206 & 0.0146 & 0.0150 & 1.6678 \\
40  & 0.0224 & 0.0250 & 0.0420 & 0.0207 & 0.0202 & 2.6496 \\
60  & 0.0134 & 0.0211 & 0.0235 & 0.0214 & 0.0125 & 2.7344 \\
100 & 0.0216 & 0.0301 & 0.0298 & 0.0210 & 0.0314 & 2.7950 \\
\bottomrule
\end{tabular}
}
{\textit{Note.} Entries are taken from Table~\ref{tab:runtime-sec-structure-aware}. Online time includes conditional scenario generation for GtO.}
\end{table}

\subsection{Product Placement Problem}
\label{sec:product-placement} 

We next consider a product placement problem, in which a retailer allocates inventory across stores before demand is realized and subsequently transfers stock among stores to meet demand \citep{bertsimas22}. This setting introduces multivariate uncertainty, with the initial inventory decision having the same dimension as the demand vector. More importantly, this setting allows us to examine the computational consequences of signed PtO weights as the decision dimension increases. In the newsvendor problem, the decision is one-dimensional, so even when KRR produces negative weights, the piecewise-linear objective can be globally minimized by evaluating its breakpoints. In contrast, the product placement problem also has a piecewise-linear objective, but with a higher decision dimension, the number of candidate vertices can grow rapidly. Enumerating these candidates therefore becomes computationally impractical as $d_z$ increases. We therefore need to exploit the structure of the optimization problem to using specific optimization algorithm to solve it.

The product placement problem has a two-stage linear structure. With nonnegative scenario weights, the second-stage optimization problems can be incorporated into the first-stage optimization, yielding a single linear program (LP) that can be efficiently solved even in higher dimensions. However, signed KRR weights invalidate this reformulation and may result in a nonconvex optimization problem. We therefore use a difference-of-convex algorithm (DCA) with linear programming subproblems; the derivation is provided in Appendix~\ref{app:num-placement-solver}. Since DCA does not generally guarantee global optimality, the reported KRR results may also be affected by optimization error. Unlike the one-dimensional newsvendor problem, the computational cost of handling signed weights can become substantial as the decision dimension $d_z$ increases. Details of the network, demand model, and optimization formulations are provided in Appendix~\ref{app:num-placement}.

\subsubsection{Experimental Design.}

Let $z$ denote initial inventory at the $d_y$ stores and $y$ the realized demand vector. The recourse cost is
\[
Q(z,y)=\min_{f,p}\left\{g^\top f+r^\top p:\ z-Af+p\geq y,\ f\geq0,\ p\geq0\right\},
\]
where $f$ denotes transfers, $p$ unmet demand, $A$ the network incidence matrix, and $g$ and $r$ the transportation and shortage cost vectors. The retailer minimizes $\alpha^\top z+\mathbb{E}[Q(z,Y)\mid X=x_0]$ over $0\leq z\leq U$, where $\alpha$ denotes placement costs and $U$ inventory upper bounds. Placement, transportation, and shortage costs are 2, 1, and 10 per unit, respectively. We extend the newsvendor demand model to correlated multivariate demand, fix $n=10{,}000$ and $d_x=10$, and vary $d_y\in\{1,2,5,10,15\}$.

\subsubsection{Decision Quality and Solution Time.}
\label{subsubsec:Effect of Negative Weights in PtO}

Table~\ref{tab:product-placement-results} reports decision performance and online time together. We first compare the PtO methods. KRR consistently improves on clipped KRR: its relative cost lies between approximately 0.65 and 0.70, compared with 0.42--0.65 for clipped KRR. The improvement comes with longer solution times in every setting. At $d_y=15$, KRR achieves a relative cost of 0.6602 in 272.54 seconds, whereas clipped KRR achieves 0.4227 in 22.50 seconds. KNN and KS are faster still, but provide lower decision quality than raw KRR. These comparisons among PtO methods expose a clear tradeoff between decision quality and solution time.

\begin{table}[htbp]
\TABLE
{Decision performance and online time for product placement.
\label{tab:product-placement-results}
\label{tab:product-placement-relative-cost}
\label{tab:product-placement-running-time}}
{
\small
\begin{minipage}{\textwidth}
\centering
\textit{Panel A: Relative cost (higher is better).}\par\smallskip
\begin{tabular}{@{}crrrrrr@{}}
\toprule
$d_y$ & SAA & PtO-KNN & PtO-KS & PtO-KRR & PtO-KRR$^+$ & GtO-CDM \\
\midrule
1  & 0.0000 & 0.5913 & 0.1154 & 0.6618 & 0.4476 & 0.9268 \\
2  & 0.0000 & 0.5802 & 0.2294 & 0.6807  & 0.6129 & 0.9104 \\
5  & 0.0000 & 0.5750 & 0.1979 & 0.6983 & 0.6454 & 0.9243\\
10 & 0.0000 & 0.5632 & 0.1432 & 0.6454 & 0.5623 & 0.9263\\
15 & 0.0000 & 0.5391 & 0.1537 & 0.6602 & 0.4227 & 0.9182\\
\bottomrule
\end{tabular}
\par\medskip
\textit{Panel B: Online time in seconds (lower is better).}\par\smallskip
\begin{tabular}{@{}crrrrrr@{}}
\toprule
$d_y$ & SAA & PtO-KNN & PtO-KS &  PtO-KRR & PtO-KRR$^+$ & GtO-CDM \\
\midrule
1  & 0.0048 & 0.0914 & 0.0915 & 1.2329  & 0.1333 & 2.9496 \\
2  & 0.1872 & 0.4537 & 0.3212 & 4.8291  & 0.4981 & 3.8220 \\
5  & 0.2524 & 0.9574 & 1.2131 & 29.9059 & 2.6315 & 14.6501 \\
10 & 1.1088 & 2.1841 & 2.9841 & 57.9449  & 9.0928 & 75.1714 \\
15 & 3.8801 & 3.9280 & 3.9460 & 272.5360  & 22.5041 & 180.9274 \\
\bottomrule
\end{tabular}
\end{minipage}
}
{\textit{Note.} KRR uses DCA; the other methods use a LP formulation. Online time excludes training and includes query-specific computations.}
\end{table}

Now we compare GtO with PtO algorithms in Table~\ref{tab:product-placement-results}, in particular, we look at GtO-CDM. It achieves the highest relative cost in every tested setting, approximately 0.91--0.93, but is not uniformly the fastest approach. Its online times are comparable to raw PtO-KRR in several settings, and their ordering changes with dimension. At $d_y=1$, PtO-KRR takes 1.23 seconds compared with 2.95 seconds for GtO-CDM; at $d_y=15$, GtO-CDM takes 180.93 seconds compared with PtO-KRR's 272.54 seconds. Nonnegative-weight PtO remains faster than either at the latter dimension. Even with the increased computational cost associated with signed weights, PtO-KRR achieves substantially lower decision quality than GtO-CDM despite their comparable online times in several settings. This comparison further highlights the computational consequences of losing convexity in PtO. Further specialization of the PtO-KRR solver might improve its computational efficiency, but would require additional algorithmic effort.

As shown in Table~\ref{tab:product-placement-results}, increasing $d_y$ substantially raises online computation, whereas relative decision performance varies comparatively little over the tested range. In particular, GtO's relative cost stays above 0.91 as $d_y$ increases from 1 to 15; the experiment does not exhibit a crossover favoring PtO. Similar stability of PtO performance with respect to outcome dimension has been reported by \citet{bertsimas22}. Here, increasing $d_y$ also enlarges the placement and recourse problems, so the runtime changes capture more than the cost of learning higher-dimensional uncertainty. Over the tested range, outcome dimension has a more pronounced effect on computation than on the relative decision-performance ranking.

\subsection{Real Case Study: Fresh Retail Inventory}
\label{sec:real case}

We compare PtO and GtO in a fresh-retail inventory setting using the FreshRetailNet-LT dataset released by Dingdong Fresh \citep{wang2025freshretailnet}.\footnote{\url{https://huggingface.co/datasets/Dingdong-Inc/FreshRetailNet-LT}} The application combines several features common in operational decision making: demand varies across locations and operating conditions, stockouts censor observed sales, and contextual information includes both continuous and categorical variables. It therefore provides an opportunity to examine the usefulness of the two approaches beyond the setting of our theoretical analysis and to assess whether similar insights emerge with real operational data.

\subsubsection{Data and Experimental Design.}

FreshRetailNet-LT contains daily sales records for 576 perishable products across 1,057 stores in 18 major Chinese cities. We focus on product 156 because of its broad store coverage and long sales history, and select the 50 stores with the longest common observation period. These stores span 11 cities and are observed from June 5, 2023 to July 20, 2025, yielding a balanced panel of 777 days and 38,850 store-day observations.

Stockouts occur in approximately 20\% of the observations and may censor sales. We apply the TimesNet-based demand recovery procedure described by \citet{wang2025freshretailnet} to the hourly sales and stock-status data, then aggregate recovered demand to daily totals for training and evaluation. Since procurement costs and actual selling prices are unavailable, we normalize the unit overage cost to one and specify the shortage penalty as $b(x)=4((\operatorname{discount}(x)-0.7)/0.3)^+$, where $\operatorname{discount}(x)$ denotes the ratio of the discounted selling price to the regular price, which is included as a covariate in the dataset. The resulting inventory problem is
\[
\min_{z\geq0}\;
\mathbb{E}\left[(z-Y)^++b(X)(Y-z)^+\mid X=x_0\right].
\]
This penalty makes shortages more costly when the product sells closer to its regular price. The evaluation therefore measures decision performance using recovered demand and a specified cost function.

We construct 37 exogenous features from store location, calendar effects, promotions, discounts, and weather. City, weekday, and month use one-hot encoding; holidays and promotional activity use binary indicators; the remaining features are continuous. Following \citet{ban2019big}, we also incorporate the preceding 14 days of demand, summarized by their sample mean and the 13 differences between consecutive order statistics \citep{liyanage2005practical}. The resulting covariate vector has 51 components: 32 binary or one-hot indicators and 19 continuous variables.

All settings use the same chronological split. After the initial 14-day history window, training runs from June 19, 2023 to March 22, 2025, validation covers March 23--April 13, 2025, and testing covers April 14--July 20, 2025. Each store contributes 643 training days, 22 validation days, and 98 test days.

We consider three store settings to examine how the comparison changes as learning expands from a single location to a broader network. This reflects a practical modeling choice for retailers: whether to build a model using one store's history or pool observations across stores. Pooling provides more training data but also introduces greater heterogeneity. The one-store setting uses the highest-sales store among the 50 selected stores; the 11-store setting uses the highest-sales store in each city; and the 50-store setting uses all selected stores. Their respective training samples contain 643, 7,073, and 32,150 observations. These settings allow us to examine the joint effects of data availability and pooling scope, rather than isolate a sample-size effect.

This application extends beyond the continuous-covariate setting of our theoretical analysis because the contextual features include both continuous and categorical variables. We therefore use learning models that accommodate these mixed feature types. For GtO, we use TabCSDI \citep{zheng2022diffusion}, a diffusion model for mixed-type tabular data, to generate conditional demand scenarios. For PtO, we add random forests (RF) to the existing surface-fitting benchmarks. For KRR, details of incorporating one-hot-encoded categorical variables into the kernel are provided in Appendix~\ref{app:krr-categorical}. The subsequent optimization steps remain unchanged. The purpose is to examine whether the two approaches remain useful with learning models suited to the application and whether the tradeoffs identified earlier persist beyond the assumptions of the theory. Hyperparameters are selected using validation decision cost, as in the preceding~experiments.

Because the true conditional demand distribution is unavailable, we evaluate each method using its average realized cost on the test set, computed from recovered demand. For method $M$ and store setting $s$, we report normalized cost $\widetilde C_{M,s}=\overline C_{M,s}/\overline C_{\mathrm{SAA},s}$. SAA consequently has normalized cost one, and lower values indicate better performance.

\subsubsection{Decision Quality and Online Computation.}

Panel A of Table~\ref{tab:real-case-results} shows that the relative performance of PtO and GtO changes across the pooling settings. With one store, PtO-RF achieves the lowest normalized cost, 0.4134, followed by PtO-KRR at 0.4509. GtO improves on SAA but performs worse than every PtO benchmark, with a normalized cost of 0.8324. With 11 stores, GtO achieves 0.4752 compared with 0.5106 for RF, the best PtO benchmark, reducing cost by approximately 6.9\%. With 50 stores, GtO achieves 0.5006 compared with 0.5512 for PtO-KRR, a reduction of approximately 9.2\%.

\begin{table}[htbp]
\TABLE
{Decision performance and online time for fresh-retail inventory.
\label{tab:real-case-results}
\label{tab:realcase-performance}
\label{tab: Running time for real case}}
{
\small
\begin{minipage}{\textwidth}
\centering
\textit{Panel A: Normalized test cost (lower is better).}\par\smallskip
\begin{tabular}{@{}crrrrrrr@{}}
\toprule
Stores & SAA & PtO-KNN & PtO-KS & PtO-KRR & PtO-KRR$^+$ & PtO-RF & GtO-TabCSDI \\
\midrule
1  & 1.000000 & 0.733880 & 0.703226 & 0.450895 & 0.652547 & \textbf{0.413442} & 0.832392 \\
11 & 1.000000 & 0.726613 & 0.930284 & 0.522770 & 0.667793 & 0.510626 & \textbf{0.475210} \\
50 & 1.000000 & 0.688577 & 0.890121 & 0.551178 & 0.716294 & 0.563453 & \textbf{0.500594} \\
\bottomrule
\end{tabular}
\par\medskip
\textit{Panel B: Online time in seconds (lower is better).}\par\smallskip
\begin{tabular}{@{}crrrrrrr@{}}
\toprule
Stores & SAA & PtO-KNN & PtO-KS & PtO-KRR & PtO-KRR$^+$ & PtO-RF & GtO-TabCSDI \\
\midrule
1  & 0.0002 & 0.0028 & 0.0008 & 0.0115 & 0.0072 & 0.0121 & 1.0461 \\
11 & 0.0011 & 0.0048 & 0.0013 & 0.0149 & 0.0155 & 0.0962 & 7.1255 \\
50 & 0.0203 & 0.0189 & 0.0397 & 0.0297 & 0.0284 & 0.1063 & 18.0873 \\
\bottomrule
\end{tabular}
\end{minipage}
}
{\textit{Note.} Boldface identifies the lowest normalized cost in each store setting. Online times exclude model training, include query-specific computation, and are averaged over all test points.}
\end{table}

The shift toward GtO with broader pooling resembles the sample-size pattern in the synthetic experiments. However, because pooling changes both the training sample and the store population, the performance differences across settings cannot be attributed to sample size alone.

Panel B reports the computational tradeoff. GtO's online time increases from 1.05 seconds in the one-store setting to 7.13 seconds with 11 stores and 18.09 seconds with 50 stores. All PtO methods remain below 0.11 seconds in the three settings. With 50 stores, for example, GtO's 9.2\% cost reduction relative to PtO-KRR accompanies an online time of 18.09 seconds, compared with 0.030 seconds for PtO-KRR. Thus, PtO performs better in the smallest setting on both decision quality and latency, while broader pooling introduces a tradeoff between GtO's lower decision cost and PtO's faster response. These findings extend the empirical comparison to mixed contextual features and recovered demand, while illustrating similar tradeoffs between data availability, decision quality, and computation.

\section{Conclusion}
\label{sec:conclusion}

In this paper, we study a fundamental representation choice in contextual stochastic optimization: should one learn the conditional objective needed for a particular decision problem, or learn the underlying conditional distribution and generate scenarios for subsequent optimization? We formalize these two approaches as Predict-then-Optimize (PtO) and Generate-then-Optimize (GtO). Our main message is that neither paradigm is uniformly preferable. PtO benefits from statistical parsimony and avoids online scenario generation, whereas GtO learns a richer stochastic representation that can preserve the structure of the downstream optimization problem and be reused across decision tasks. The appropriate choice therefore depends on the statistical, computational, and operational characteristics of the problem.

Our theoretical analysis makes this trade-off explicit for convex contextual stochastic optimization. For a broad class of local PtO methods, nonnegative weighting preserves the convexity of the original problem but limits the ability of the estimator to exploit high-order smoothness in high-dimensional contexts. Global PtO methods such as KRR can achieve greater statistical efficiency by exploiting smoothness, but their signed weights need not preserve convexity. GtO avoids this tension because its downstream problem is an ordinary sample average approximation with nonnegative scenario weights and therefore inherits the convexity of the primitive stochastic program. Its statistical performance, however, depends on learning the richer conditional distribution. We then decompose the regret of GtO into an offline learning error and an online optimization error, providing a framework for directly comparing its statistical efficiency with that of PtO.

These results lead to a practical characterization of when each paradigm is attractive. PtO tends to be preferable when the conditional objective is sufficiently smooth, the uncertainty dimension is large, offline data are limited, the downstream task is stable, or online computation is particularly costly. GtO becomes more attractive when the conditional distribution is substantially smoother than the induced objective, the context dimension is large relative to the uncertainty dimension, preserving the structure of downstream optimization is important, sufficient training data are available, or the learned uncertainty model is intended to support multiple decision problems. Our numerical and empirical studies complement the theoretical analysis by illustrating how data availability, covariate dimension, learning formulation, and downstream problem structure affect the relative decision quality and online computational costs of PtO and GtO.

This study leaves several directions for future research. First, our analysis focuses on convex contextual stochastic optimization; extending the comparison to nonconvex, discrete, and multistage decision problems may reveal different trade-offs between statistical representation and optimization structure. Second, both PtO and GtO in this paper follow a sequential learning and optimization paradigm. An important direction is to study the predictive vs. generative representation choice within integrated or decision-focused learning frameworks. Third, reducing the online computational cost of GtO through faster conditional generators, more efficient scenario construction, or adaptive sampling may further expand the range of settings in which generative decision making is practically attractive. More broadly, this paper hopes that viewing contextual decision making through the choice of what statistical object to learn provides a useful perspective for designing data-driven optimization methods.

\bibliographystyle{informs2014}
\bibliography{GtO}

\begin{thebibliography}{8}
\providecommand{\natexlab}[1]{#1}
\providecommand{\url}[1]{\texttt{#1}}
\providecommand{\urlprefix}{URL }

\bibitem[{Dirksen(2015)}]{dirksen2015tail}
Dirksen S (2015) Tail bounds via generic chaining. \emph{Electron. J. Probab.} 20(53):1--29.

\bibitem[{Hu et~al.(2020{\natexlab{a}})Hu, Chen, \protect\BIBand{} He}]{hu2020sample}
Hu Y, Chen X, He N (2020{\natexlab{a}}) Sample complexity of sample average approximation for conditional stochastic optimization. \emph{SIAM J. Optim.} 30(3):2103--2133.

\bibitem[{Hu et~al.(2020{\natexlab{b}})Hu, Zhang, Chen, \protect\BIBand{} He}]{hu2020biased}
Hu Y, Zhang S, Chen X, He N (2020{\natexlab{b}}) Biased stochastic first-order methods for conditional stochastic optimization and applications in meta learning. Larochelle H, Ranzato M, Hadsell R, Balcan MF, Lin HT, eds., \emph{Adv. Neural Inform. Proc. Systems 33} (Curran Associates, New York), 2759--2770.

\bibitem[{Sadana et~al.(2025)Sadana, Chenreddy, Delage, Forel, Frejinger, \protect\BIBand{} Vidal}]{sadana25survey}
Sadana U, Chenreddy A, Delage E, Forel A, Frejinger E, Vidal T (2025) A survey of contextual optimization methods for decision-making under uncertainty. \emph{Eur. J. Oper. Res.} 320(2):271--289.

\bibitem[{Shapiro et~al.(2009)Shapiro, Dentcheva, \protect\BIBand{} Ruszczy{\'n}ski}]{ShapiroDentchevaRuszczynski2009}
Shapiro A, Dentcheva D, Ruszczy{\'n}ski A (2009) \emph{Lectures on Stochastic Programming: Modeling and Theory} (SIAM, Philadelphia, PA).

\bibitem[{Tuo \protect\BIBand{} Zou(2026)}]{tuo2024asymptotic}
Tuo R, Zou L (2026) Asymptotic theory for linear functionals of kernel ridge regression. \emph{Statist. Sinica} Forthcoming.

\bibitem[{van~der Vaart(1998)}]{vanderVaart1998asymptotic}
van~der Vaart AW (1998) \emph{Asymptotic Statistics} (Cambridge University Press, Cambridge, UK).

\bibitem[{Vershynin(2018)}]{Vershynin2018}
Vershynin R (2018) \emph{High-Dimensional Probability: An Introduction with Applications in Data Science} (Cambridge University Press, Cambridge, UK).

\end{thebibliography}

\ECTYPE{}

\ECRUNAUTHOR{Hong, Huang, and Li}
\ECRUNTITLE{Contextual Decision Making: Predictive or Generative}

\ECRRHFirstLine{\bfseries\theECRUNAUTHOR:}
\ECRRHSecondLine{\itshape\theECRUNTITLE}

\ECLRHFirstLine{\bfseries\theECRUNAUTHOR:}
\ECLRHSecondLine{\itshape\theECRUNTITLE}

\ECSwitch

\ECHead{Appendix}

\section{Proofs in Section~\ref{sec:PtO}}

In this section, we prove Proposition~\ref{prop:krr_matern_uniform_error} and Theorem~\ref{thm:pto_lower_bound} respectively.  

\subsection{Proof of Proposition~\ref{prop:krr_matern_uniform_error}}
\label{app:proof of prop 1}

To analyze the regret $\mathcal R_n(x_0)$, we adopt the decomposition
\begin{align}
&\mathcal{R}_n(x_0) \nonumber \\
&\,=\, \left[h(\hat z_n(x_0);x_0)-\hat h_n(\hat z_n(x_0);x_0)\right] +  \left[\hat h_n(\hat z_n(x_0);x_0)-\hat h_n(z^{\ast}(x_0);x_0)\right] + \left[\hat h_n(z^{\ast}(x_0);x_0)-h(z^{\ast}(x_0);x_0)\right] \nonumber\\
&\,\leq\, 2 \sup_{z \in \mathcal{Z}} \abs{h(z;x_0) -\hat{h}_n(z;x_0)}.
\label{eq:PtO_error_decomposition0}
\end{align}
Equation~\eqref{eq:PtO_error_decomposition0} shows that controlling the regret of the weighted PtO procedure reduces to bounding the approximation error between the estimated conditional objective $\hat h_n$ and the true conditional objective $h$, uniformly over $\mathcal Z$. For each fixed decision $z\in\mathcal Z$, $\hat h_n(z;\cdot)$ is the KRR estimator of $h(z;\cdot)$, constructed using the observed costs $c(z;X_i,Y_i)$ as regression responses. We therefore begin by analyzing the KRR estimation error for a fixed decision and then extend the resulting bound uniformly over the decision set.

For a fixed decision $z\in\mathcal Z$, define the observed cost, its conditional mean, and the centered observation noise by
\[ 
U_{z,i}:=c(z;X_i,Y_i), \qquad h_z(x):=h(z;x)=\mathbb E\left[U_{z,i}\mid X_i=x\right], \qquad \xi_{z,i}:=U_{z,i}-h_z(X_i).
\]
Then, \( U_{z,i}=h_z(X_i)+\xi_{z,i} \text{ and } \mathbb E[\xi_{z,i}\mid X_i]=0. \) Moreover, Assumption~\ref{ass:SAA-struct0} gives $|U_{z,i}|\leq M$ and $|h_z(X_i)|\leq M$, and hence $|\xi_{z,i}|\leq2M$. Therefore, by Hoeffding's lemma,
\begin{equation}
\mathbb E\left[ \exp(t\xi_{z,i}) \,\middle|\, X_i \right] \leq \exp(2M^2t^2), \quad t\in\mathbb R.
\label{eq:krr-conditional-subgaussian}
\end{equation}
Thus, conditional on $X_{1:n}$, the noise variables $\xi_{z,1},\ldots,\xi_{z,n}$ are independent, mean-zero, and sub-Gaussian with a sub-Gaussian bound independent of $z$. 

To separate the effect of observation noise from the remaining estimation error, we consider the conditional expectation of the KRR estimator given the observed covariates. Define $\bar h_{n,z}(x):=\mathbb E[\hat h_n(z;x)| X_{1:n}]$. Let $\mathbf c(z;X_{1:n},Y_{1:n}) = (c(z;X_1,Y_1),\ldots,c(z;X_n,Y_n))^\top$ denote the vector of observed responses. By Equation~\eqref{eq:KRRestimator}, the KRR estimator admits the closed-form expression
\[
\hat h_n(z;x) = k_n(x)^\top (K_n+n\lambda_n I_n)^{-1}\mathbf c(z,X_{1:n},Y_{1:n}).
\]
Since $K_n$ and $k_n(x)$ are determined by $X_{1:n}$ and $\mathbb E[\mathbf c(z,X_{1:n},Y_{1:n})| X_{1:n}]=(h_z(X_1),\ldots,h_z(X_n))^\top$, taking the conditional expectation gives
\[
\bar h_{n,z}(x) = k_n(x)^\top(K_n+n\lambda_n I_n)^{-1}
(h_z(X_1),\ldots,h_z(X_n))^\top.
\]
Thus, $\bar h_{n,z}$ coincides with the KRR estimator fitted to the noiseless responses $h_z(X_1),\ldots,h_z(X_n)$. We therefore decompose the objective approximation error as
\begin{equation}
\left\| \hat h_n(z;\cdot)-h_z \right\|_\infty \leq \underbrace{\left\|\hat h_n(z;\cdot)-\bar h_{n,z}\right\|_\infty}_{=:S_{n,z}} + \underbrace{\left\|\bar h_{n,z}-h_z\right\|_\infty}_{=:B_{n,z}}.
\label{eq:krr-error-decomposition}
\end{equation}
The first term $S_{n,z}$ is the stochastic error arising from the observation noise, while the second term $B_{n,z}$ is the conditional bias. We first establish a bound for $S_{n,z}$ on a high-probability event for the random design $X_{1:n}$ in Lemma~\ref{lem:krr-stochastic-supnorm}. We then bound $B_{n,z}$ on the same event and combine the two bounds to obtain a fixed-decision sup-norm error bound in Lemma~\ref{lem:fixed-decision-krr}. Finally, Lemma~\ref{lem:krr_uniformization} extends this result uniformly over $\mathcal Z$ and establishes the expected uniform approximation error needed to control the regret.

Let $d:=d_x$. For a given realization $X_{1:n}$, define the empirical norm $\|f\|_n:=(n^{-1}\sum_{i=1}^n f(X_i)^2)^{1/2}$. 
For a $t \geq 1$, define the event
\[
\mathcal E_n(t) := \left\{ \|f\|_{L^2} \leq \max\left\{ C_1\|f\|_n,\, C_2t n^{-s/d}\|f\|_{H^s} \right\} \text{ for all }f\in H^s(\mathcal X) \right\},
\]
where $C_1$ and $C_2$ are the constants in Theorem~S4.1 in the Supplementary Material of \citeEC{tuo2024asymptotic}.

\begin{lemma}[$S_{n,z}$ bound on a good-design event]
\label{lem:krr-stochastic-supnorm}
Suppose the conditions of Proposition~\ref{prop:krr_matern_uniform_error} hold. Then, for any fixed $A>0$, there exist constants $C_A'>0$, $C_A<\infty$, and an integer $N_A\geq 1$, all independent of $n$ and $z$, such that the following two statements hold for every $n\geq N_A$. (i) Let $t_n:=C_A'(\log n)^{s/d}$. The event $\mathcal E_n(t_n)$ holds with a large probability,
\begin{equation}
\mathbb P\bigl(\mathcal E_n(t_n)^c\bigr) \,\leq\, \frac{1}{2}n^{-A}.
\label{eq:krr-design-event-polynomial-tail}
\end{equation}
(ii) For every $z\in\mathcal Z$ and every realization $X_{1:n}\in\mathcal E_n(t_n)$,
\begin{equation} \mathbb P \left( S_{n,z} > C_A n^{-1/2} \lambda_n^{-\frac{d}{4s}} \sqrt{\log n} \,\middle|\, X_{1:n} \right) \,\leq\, \frac{1}{2}n^{-A}.
\label{eq:krr-stochastic-polynomial-tail}
\end{equation}
\end{lemma}

The proof of Lemma~\ref{lem:krr-stochastic-supnorm} relies on the smoothing-regime bound in Lemma~S5.1 in the Supplementary Material of \citetEC{tuo2024asymptotic}. Applying this bound in our setting requires verifying several conditions. To make these conditions explicit and the subsequent argument easier to follow, we first restate Lemma~S5.1 of \citetEC{tuo2024asymptotic} in our notation as Lemma~\ref{lem:noiseless-krr-bound}.

\begin{lemma}[Lemma~S5.1 of \citealtEC{tuo2024asymptotic}]
\label{lem:noiseless-krr-bound}
Let $g\in\mathcal H_\nu$ and $\hat g$ be the KRR estimator fitted to the responses $g(X_1),\ldots,g(X_n)$ with regularization parameter $\lambda_n>0$. Suppose there exist constants $C_g>0$ and $\delta\in(0,1]$ such that
\begin{equation}
|\langle g,f\rangle_{\mathcal H_\nu}| \,\leq\, C_g\|f\|_{L^2}^{\delta}\|f\|_{\mathcal H_\nu}^{1-\delta} \quad \text{for all } f\in\mathcal H_\nu.
\label{eq:noiseless-krr-condition}
\end{equation}
Suppose further that the design $X_{1:n}$ satisfies \( \|f\|_{L^2} \leq \max\{ a\|f\|_n,\, b\|f\|_{\mathcal H_\nu} \} \text{ for all } f\in\mathcal H_\nu \) and for some constants $a,b>0$ satisfying $b\leq a\lambda_n^{1/2}$. Then
\begin{equation}
\|\hat g-g\|_n \,\leq\, 2C_g a^\delta\lambda_n^{(1+\delta)/2}.
\label{eq:noiseless-krr-bound}
\end{equation}
\end{lemma}

With Lemma~\ref{lem:noiseless-krr-bound} in place, we now prove Lemma~\ref{lem:krr-stochastic-supnorm}.
\smallskip

\begin{proof}{Proof of Lemma \ref{lem:krr-stochastic-supnorm}}
Under Assumption~\ref{ass:Regularity_density}, Theorem~S4.1 of \citeEC{tuo2024asymptotic} implies that there exist constants $C_1,C_2,C_3,C_4>0$, independent of $n$ and the design points, such that for every $t\geq1$, \( \mathbb P(\mathcal E_n(t)^c) \leq C_3\exp\{-C_4t^{d/s}\}. \) For the prescribed $A>0$, let \( t_n=C_A'(\log n)^{s/d}, \) where $C_A'>0$ is chosen sufficiently large so that \( q_A:=C_4(C_A')^{d/s}>A. \) Then \( \mathbb P(\mathcal E_n(t_n)^c) \leq C_3\exp\{-C_4(C_A')^{d/s}\log n\} = C_3n^{-q_A}. \) Since $q_A>A$, there exists an integer $N_{A}\geq 1$ such that $C_3n^{-(q_A-A)}\leq 1/2$ for all $n\geq N_{A}$. Consequently, for every $n\geq N_{A}$, \( C_3n^{-q_A}= C_3n^{-(q_A-A)}n^{-A} \leq \frac{1}{2}n^{-A}. \) Therefore, \( \mathbb P(\mathcal E_n(t_n)^c) \leq \frac{1}{2}n^{-A}, \)
 which establishes Equation~\eqref{eq:krr-design-event-polynomial-tail} for all $n\geq N_{A}$.

We next prove Equation~\eqref{eq:krr-stochastic-polynomial-tail}. We follow the argument in the proof of Theorem~S2.3 of \citeEC{tuo2024asymptotic}, particularly their Equations~(S5.61)--(S5.62), to bound the stochastic error $S_{n,z}$. For each fixed $z\in\mathcal Z$, define the stochastic process $\mathbb G_{n,z}(x):=\hat h_n(z;x)-\bar h_{n,z}(x)$, indexed by $x\in\mathcal X$. Then $S_{n,z}=\sup_{x\in\mathcal X}|\mathbb G_{n,z}(x)|.$ Fix any $x_0 \in \mathcal{X}$. By the triangle inequality,
\begin{equation}
S_{n,z} \,\leq\, \sup_{x\in\mathcal X} |\mathbb G_{n,z}(x)-\mathbb G_{n,z}(x_0)| + |\mathbb G_{n,z}(x_0)|.
\label{eq:stochastic_error_decomposition}
\end{equation}
In the following, we bound the two terms in the right hand side of Equation~\eqref{eq:stochastic_error_decomposition} separately. For the first term, we equip $\mathcal X$ with a semimetric that controls the increments of $\mathbb G_{n,z}$ and apply a generic chaining inequality of \citeEC{dirksen2015tail}. The second term is controlled directly using a sub-Gaussian tail bound. 

Now we bound the first term in Equation~\eqref{eq:stochastic_error_decomposition}. Conditional on $X_{1:n}$, the process $\mathbb G_{n,z}$ is centered. Let $\boldsymbol{\xi}_z:=(\xi_{z,1},\ldots,\xi_{z,n})^\top$ denote the noise vector defined above. By the closed-form KRR representation and $U_{z,i}=h_z(X_i)+\xi_{z,i}$, we have \( \mathbb G_{n,z}(x) = k_n(x)^\top(K_n+n\lambda_n I_n)^{-1}\boldsymbol{\xi}_z. \) We now construct a semimetric using the KRR representation. For each $x\in\mathcal X$, define the kernel section $\phi_x(\cdot):=K_\nu(\cdot,x)$, and let $\hat\phi_x$ denote the KRR estimator fitted to the responses $\phi_x(X_1),\ldots,\phi_x(X_n)$. Since $\phi_x(X_{1:n})=k_n(x)$, the closed-form KRR representation gives \( \hat\phi_x(\cdot) =k_n(x)^\top (K_n+n\lambda_n I_n)^{-1} K_\nu(\cdot,X_{1:n})^\top. \) Evaluating at the training points and subtracting $\phi_x(X_{1:n})=k_n(x)$ yields
\[
(\hat\phi_x-\phi_x)(X_{1:n}) \,=\, \left[K_n(K_n+n\lambda_n I_n)^{-1}-I_n\right]k_n(x) \,=\, -n\lambda_n(K_n+n\lambda_n I_n)^{-1}k_n(x).
\]
Consequently, \( \mathbb G_{n,z}(x) = -\frac{1}{n\lambda_n}\sum_{i=1}^n(\hat\phi_x-\phi_x)(X_i)\xi_{z,i}. \) For $x,x'\in\mathcal X$, define $\phi_{x,x'}:=\phi_x-\phi_{x'}$ and $\hat\phi_{x,x'}:=\hat\phi_x-\hat\phi_{x'}$. Then
\begin{equation}
\mathbb G_{n,z}(x)-\mathbb G_{n,z}(x') \,=\, -\frac{1}{n\lambda_n} \sum_{i=1}^n (\hat\phi_{x,x'}-\phi_{x,x'})(X_i)\xi_{z,i}.
\label{eq:stochastic increment}
\end{equation}
Since the noise variables $\xi_{z,i}$ are conditionally independent and sub-Gaussian, the conditional fluctuations of these increments in Equation~\eqref{eq:stochastic increment} are controlled by $\|\hat\phi_{x,x'}-\phi_{x,x'}\|_n$. We therefore define the following semimetric on $\mathcal X$ to control the increments of $\mathbb G_{n,z}$:
\begin{equation}
\mathfrak d_n(x,x') \,:=\, 2\sqrt{2}M\lambda_n^{-1}n^{-1/2} \|\hat\phi_{x,x'}-\phi_{x,x'}\|_n.
\label{eq:krr-semimetric}
\end{equation}
This is a semimetric on $\mathcal X$. By Equation~\eqref{eq:krr-conditional-subgaussian} and the conditional independence of $\xi_{z,1},\ldots,\xi_{z,n}$,
\[
\begin{aligned}
\mathbb E\left[ \exp\left\{ \theta\bigl(\mathbb G_{n,z}(x)-\mathbb G_{n,z}(x')\bigr) \right\} \,\middle|\, X_{1:n} \right] \,\leq\, \exp\left\{ \frac{2M^2\theta^2}{n\lambda_n^2} \left\| \hat\phi_{x,x'}-\phi_{x,x'} \right\|_n^2 \right\} \,=\, \exp\left\{ \frac{\theta^2\mathfrak d_n^2(x,x')}{4} \right\}.
\end{aligned}
\]
A Chernoff bound therefore yields
\[
\mathbb P \left( \left| \mathbb G_{n,z}(x)-\mathbb G_{n,z}(x') \right| \geq u\,\mathfrak d_n(x,x') \,\middle|\, X_{1:n} \right) \,\leq\, 2e^{-u^2}, \quad u\geq0.
\]

In the terminology of \citeEC{dirksen2015tail}, a stochastic process $\{X_t:t\in T\}$ is a $\psi_2$-process with respect to a semimetric $\rho$ if its increments satisfy \( \mathbb P\left( |X_t-X_s| \geq u\,\rho(t,s) \right) \leq 2e^{-u^2} \text{ for all } s,t\in T, u\geq0.  \) Thus, conditional on $X_{1:n}$, $\{\mathbb G_{n,z}(x):x\in\mathcal X\}$ is a $\psi_2$-process with respect to $\mathfrak d_n$, and the semimetric is independent of $z$. This verifies the conditions required for Theorem~3.2 of \citeEC{dirksen2015tail}. Applying that theorem, we get a generic chaining inequality for the increment of the stochastic process $\{\mathbb G_{n,z}(x):x\in\mathcal X\}$. That is, for any fixed $x_0\in\mathcal X$ and every $u\geq1$,
\begin{equation}
\mathbb P\left( \sup_{x\in\mathcal X} \left| \mathbb G_{n,z}(x)-\mathbb G_{n,z}(x_0) \right| \geq e^{1/2} \left\{ C_D\gamma_2(\mathcal X,\mathfrak d_n) + uD_D\Delta_{\mathfrak d_n}(\mathcal X) \right\} \,\middle|\, X_{1:n} \right) \,\leq\, \exp(-u^2/2),
\label{eq:theorem 3.2 of Dirksen}
\end{equation}
where $C_D,D_D>0$ are universal constants. Here \( \Delta_{\mathfrak d_n}(\mathcal X) := \sup_{x,x'\in\mathcal X} \mathfrak d_n(x,x') \) is the diameter of $\mathcal X$ under the semimetric $\mathfrak d_n$, and $\gamma_2(\mathcal X,\mathfrak d_n)$ is the generic-chaining functional, which admits the standard entropy bound
\begin{equation}
\gamma_2(\mathcal X,\mathfrak d_n) \,\leq\, C\int_0^{\Delta_{\mathfrak d_n}(\mathcal X)} \sqrt{\log N(\mathcal X,\mathfrak d_n,\varepsilon)} \,\mathrm{d}\varepsilon,
\label{eq:entropy bound}
\end{equation}
where $C>0$ is a universal constant and $N(\mathcal X,\mathfrak d_n,\varepsilon)$ denotes the covering number of $\mathcal X$ under $\mathfrak d_n$ at radius $\varepsilon$. We next bound these two quantities on $\mathcal{E}_n(t_n)$.

We first bound the diameter $\Delta_{\mathfrak d_n}(\mathcal X)$. Recall that $\phi_x(\cdot)=K_\nu(\cdot,x)$. By the reproducing property, for every $f\in\mathcal H_\nu$, $\langle\phi_x,f\rangle_{\mathcal H_\nu}=f(x)$. Since $s>d/2$, the Sobolev interpolation inequality and the norm equivalence between $\mathcal H_\nu$ and $H^s(\mathcal X)$ imply
$$
|\langle\phi_x,f\rangle_{\mathcal H_\nu}| \,\leq\, C\|f\|_{L^2}^{1-\frac{d}{2s}}\|f\|_{\mathcal H_\nu}^{\frac{d}{2s}} \quad \text{for all } f\in\mathcal H_\nu,
$$
where $C$ is independent of $x$. Thus, condition~\eqref{eq:noiseless-krr-condition} holds for $g=\phi_x$ with $\delta=1-d/(2s)$ and $C_g=C$, uniformly over $x\in\mathcal X$. Moreover, on $\mathcal E_n(t_n)$, the norm equivalence between $\mathcal H_\nu$ and $H^s(\mathcal X)$ gives
\[
\|f\|_{L^2} \,\leq\, \max\left\{ C_1\|f\|_n,\, \widetilde C_2t_n n^{-s/d}\|f\|_{\mathcal H_\nu} \right\}, \quad f\in\mathcal H_\nu,
\]
for some constant $\widetilde C_2>0$. Since $t_n n^{-s/d}=o(\lambda_n^{1/2})$, we have $\widetilde C_2t_n n^{-s/d}\leq C_1\lambda_n^{1/2}$ for all sufficiently large $n$. Hence, the conditions of Lemma~\ref{lem:noiseless-krr-bound} are satisfied. Applying that lemma, we obtain
\begin{equation}
\|\hat\phi_x-\phi_x\|_n \,\leq\, C\lambda_n^{1-\frac{d}{4s}}
\label{eq:krr-kernel-section-error}
\end{equation}
uniformly over $x\in\mathcal X$. For any $x,x'\in\mathcal X$, \( \hat\phi_{x,x'}-\phi_{x,x'} = (\hat\phi_x-\phi_x) - (\hat\phi_{x'}-\phi_{x'}), \) and hence, by Equation~\eqref{eq:krr-kernel-section-error},
\begin{equation}
\left\| \hat\phi_{x,x'}-\phi_{x,x'} \right\|_n \,\leq\, C \lambda_n^{1-\frac{d}{4s}}.
\label{eq:eq:krr-semimetric-holder}
\end{equation}
Substituting this bound into Equation~\eqref{eq:krr-semimetric} yields \( \mathfrak d_n(x,x') \leq C n^{-1/2} \lambda_n^{-\frac{d}{4s}} \) uniformly over $x,x'\in\mathcal X$. Therefore,
\begin{equation}
\Delta_{\mathfrak d_n}(\mathcal X) \,\leq\, C n^{-1/2} \lambda_n^{-\frac{d}{4s}}.
\label{eq:krr-metric-diameter}
\end{equation}

We next bound $\gamma_2(\mathcal X,\mathfrak d_n)$. Fix any $\tau\in(0,\min\{\nu,1\})$. Since $s=\nu+d/2$, the Sobolev embedding and the norm equivalence between $\mathcal H_\nu$ and $H^s(\mathcal X)$ imply that
\[ |f(x)-f(x')| \,\leq\, C\|f\|_{\mathcal H_\nu}\|x-x'\|^\tau, \quad f\in\mathcal H_\nu,~ x,x'\in\mathcal X.
\]
By the reproducing property,
\[
\|\phi_{x,x'}\|_{\mathcal H_\nu} \,=\, \sup_{\|f\|_{\mathcal H_\nu}\leq1} |\langle\phi_{x,x'},f\rangle_{\mathcal H_\nu}| \,=\, \sup_{\|f\|_{\mathcal H_\nu}\leq1} |f(x)-f(x')| \,\leq\, C\|x-x'\|^\tau.
\]
Since $\hat\phi_{x,x'}$ is the KRR estimator fitted to the responses $\phi_{x,x'}(X_1),\ldots,\phi_{x,x'}(X_n)$, comparing its regularized empirical objective with that of $\phi_{x,x'}$ gives \( \|\hat\phi_{x,x'}-\phi_{x,x'}\|_n^2 + \lambda_n\|\hat\phi_{x,x'}\|_{\mathcal H_\nu}^2 \leq \lambda_n\|\phi_{x,x'}\|_{\mathcal H_\nu}^2. \) Consequently, \( \|\hat\phi_{x,x'}-\phi_{x,x'}\|_n \leq C\lambda_n^{1/2}\|x-x'\|^\tau. \) Substituting this bound into Equation~\eqref{eq:krr-semimetric} yields
\begin{equation}
\mathfrak d_n(x,x') \,\leq\, Cn^{-1/2}\lambda_n^{-1/2}\|x-x'\|^\tau, \quad x,x'\in\mathcal X.
\label{eq:krr-semimetric-euclidean}
\end{equation}

By Equation~\eqref{eq:krr-semimetric-euclidean}, any Euclidean cover of $\mathcal X$ with radius \( r_\varepsilon := (\varepsilon /(Cn^{-1/2}\lambda_n^{-1/2}) )^{1/\tau} \) is also an $\varepsilon$-cover under $\mathfrak d_n$. Hence, \( N(\mathcal X,\mathfrak d_n,\varepsilon) \leq N(\mathcal X,\|\cdot\|,r_\varepsilon). \) Since $\mathcal X$ is a bounded subset of $\mathbb R^d$, the standard Euclidean covering number bound gives
\begin{equation}
\log N(\mathcal X,\mathfrak d_n,\varepsilon) \,\leq\, d\log\left[ 1+C\left((n^{-1/2}\lambda_n^{-1/2})/\varepsilon \right)^{1/\tau} \right].
\label{eq:krr-covering-number}
\end{equation}
For convenience, let $a_n:=n^{-1/2}\lambda_n^{-d/(4s)}$ and $b_n:=n^{-1/2}\lambda_n^{-1/2}$. Combining Equations~\eqref{eq:entropy bound}, \eqref{eq:krr-metric-diameter}, and~\eqref{eq:krr-covering-number}, we obtain
\[
\begin{aligned}
\gamma_2(\mathcal X,\mathfrak d_n)
&\,\leq\, C\int_0^{Ca_n} \left({\log\left[1+C\left({b_n}/{\varepsilon}\right)^{1/\tau}\right]}\right)^{1/2}\,\mathrm{d} \varepsilon \,\leq\, Ca_n\int_0^1 \left({\log\left[1+C\left({b_n}/({a_nv})\right)^{1/\tau}\right]}\right)^{1/2}\,\mathrm{d}v\\
&\,\leq\, Ca_n\int_0^1 \left({ \log\left[1+C\left({b_n}/{a_n}\right)^{1/\tau}\right] +\frac{1}{\tau}\log(1/v) }\right)^{1/2}\,\mathrm{d}v\\
&\,\leq\, Ca_n\int_0^1 \left[ \left({\log\left[1+C\left({b_n}/{a_n}\right)^{1/\tau}\right]}\right)^{1/2} +\left(\frac{1}{{\tau}}{\log(1/v)}\right)^{1/2} \right]\mathrm{d}v\\
&\,=\, Ca_n \left[ \left({\log\left(1+C\lambda_n^{-\frac{\nu}{2s\tau}}\right)}\right)^{1/2} +\frac{\sqrt{\pi}}{2\sqrt{\tau}} \right]\\
&\,\leq\, Ca_n \left[ \sqrt{\log(1+C)+\frac{\nu}{2s\tau}\log(1/\lambda_n)} +\frac{\sqrt{\pi}}{2\sqrt{\tau}} \right]
\,\leq\,  Ca_n\sqrt{\log(C'/\lambda_n)}.
\end{aligned}
\]
In the preceding inequalities, $C>0$ denotes a generic constant that may differ from one occurrence to another but is independent of $n$, $z$, and the design points. The last inequality holds for all sufficiently large $n$, since $\lambda_n\to0$ and $\tau>0$ is fixed. Therefore,
\begin{equation}
\gamma_2(\mathcal X,\mathfrak d_n) \,\leq\, Cn^{-1/2}\lambda_n^{-\frac{d}{4s}} \sqrt{\log(C'/\lambda_n)}.
\label{eq:krr-gamma2-bound}
\end{equation}

Fix any $x_0\in\mathcal X$. Substituting Equations~\eqref{eq:krr-metric-diameter} and~\eqref{eq:krr-gamma2-bound} into Equation~\eqref{eq:theorem 3.2 of Dirksen}, we obtain, for every realization $X_{1:n}\in\mathcal E_n(t_n)$ and every $u\geq1$,
\begin{equation}
\mathbb P\left( \sup_{x\in\mathcal X} \left| \mathbb G_{n,z}(x)-\mathbb G_{n,z}(x_0) \right| > C n^{-1/2} \lambda_n^{-\frac{d}{4s}} \left[ \sqrt{\log(C'/\lambda_n)}+u \right] \,\middle|\, X_{1:n} \right) \,\leq\, e^{-cu^2},
\label{eq:krr-increment-tail}
\end{equation}
where $c,C>0$ are independent of $z$.

Now we bound the second term in Equation~\eqref{eq:stochastic_error_decomposition}. From \( \mathbb G_{n,z}(x_0) = -\frac{1}{n\lambda_n} \sum_{i=1}^n (\hat\phi_{x_0}-\phi_{x_0})(X_i)\xi_{z,i}, \) Equation~\eqref{eq:krr-conditional-subgaussian} implies, conditional on $X_{1:n}$, \[ 
\mathbb E\left[ \exp\left\{ \theta\mathbb G_{n,z}(x_0) \right\} \,\middle|\, X_{1:n} \right] \,\leq\, \exp\left\{ \frac{2M^2\theta^2}{n\lambda_n^2} \|\hat\phi_{x_0}-\phi_{x_0}\|_n^2 \right\}. 
\] 
Using Equation~\eqref{eq:krr-kernel-section-error}, for every $X_{1:n}\in\mathcal E_n(t_n)$, \( \frac{2M}{\sqrt n\,\lambda_n} \|\hat\phi_{x_0}-\phi_{x_0}\|_n \leq C n^{-1/2} \lambda_n^{-\frac{d}{4s}}. \) Therefore, the standard sub-Gaussian tail bound gives, for every $u\geq1$,
\begin{equation}
\mathbb P\left( |\mathbb G_{n,z}(x_0)| > C n^{-1/2} \lambda_n^{-\frac{d}{4s}} u \,\middle|\, X_{1:n} \right) \,\leq\, 2e^{-cu^2}.
\label{eq:krr-pointwise-stochastic-tail}
\end{equation}
Equations ~\eqref{eq:stochastic_error_decomposition}, ~\eqref{eq:krr-increment-tail} and~\eqref{eq:krr-pointwise-stochastic-tail}, together with a union bound, imply that for every realization $X_{1:n}\in\mathcal E_n(t_n)$ and every $u\geq1$,
\begin{equation}
\mathbb P\left( S_{n,z} > C n^{-1/2} \lambda_n^{-\frac{d}{4s}} \left[ \sqrt{\log(C'/\lambda_n)}+u \right] \,\middle|\, X_{1:n} \right) \,\leq\, 3e^{-cu^2}.
\label{eq:krr-stochastic-tail}
\end{equation}
Choose $u_n=C_A''\sqrt{\log n}$, where $C_A''>0$ is sufficiently large. Since $\lambda_n\asymp(\log n)/n$, \( \log(C'/\lambda_n)=O(\log n), \) and $C_A''$ can be chosen so that \( 3e^{-cu_n^2} \leq \frac{1}{2}n^{-A} \) for all sufficiently large $n$. Hence, after enlarging $C_A$ if necessary, Equation~\eqref{eq:krr-stochastic-tail} yields
\begin{equation*}
\mathbb P \left( S_{n,z} > C_A n^{-1/2} \lambda_n^{-\frac{d}{4s}} \sqrt{\log n} \,\middle|\, X_{1:n} \right) \,\leq\, \frac{1}{2}n^{-A}
\end{equation*}
for every realization $X_{1:n}\in\mathcal E_n(t_n)$ and all sufficiently large $n$. Together with Equation~\eqref{eq:krr-design-event-polynomial-tail}, this proves the two assertions of Lemma~\ref{lem:krr-stochastic-supnorm}. \Halmos
\end{proof}

\smallskip

Having established a bound for the stochastic error $S_{n,z}$ on the good-design event $\mathcal E_n(t_n)$, we next control the conditional bias $B_{n,z}$ on the same event and combine the two bounds to yield the fixed-decision sup-norm result in Lemma~\ref{lem:fixed-decision-krr}.

\begin{lemma}[Fixed-decision sup-norm bound for KRR]
\label{lem:fixed-decision-krr}
Suppose the conditions of Proposition~\ref{prop:krr_matern_uniform_error} hold. Then, for every fixed $A>0$, there exist constants $C_A<\infty$ and $n_A<\infty$, independent of $z$, such that for all $n\geq n_A$ and all $z\in\mathcal Z$,
\begin{equation}
\mathbb P_{\mathcal D_n} \left( \left\| \hat h_n(z;\cdot)-h(z;\cdot) \right\|_\infty > C_A \left( \frac{\log n}{n} \right)^{\frac{\nu}{2\nu+d_x}} \right) \,\leq\, n^{-A}.
\label{eq:fixed-krr-main-bound}
\end{equation}
\end{lemma}

\smallskip

\smallskip
\begin{proof}{Proof of Lemma \ref{lem:fixed-decision-krr}}

Recall that $d=d_x$ and $s=\nu+d/2$. By Assumption~\ref{ass:krr_source} and the norm equivalence between $\mathcal H_\nu$ and $H^s(\mathcal X)$, there exists a constant $B_{\mathcal H}<\infty$, independent of $z$, such that
\begin{equation}
\sup_{z\in\mathcal Z} \|h_z\|_{\mathcal H_\nu} \,\leq\, B_{\mathcal H}.
\label{eq:krr-rkhs-uniform-bound}
\end{equation}

Conditional on \(X_{1:n}\), \(\bar h_{n,z}\) is the KRR estimator based on the responses \(h_z(X_1),\ldots,h_z(X_n)\). By the defining optimality inequality for KRR,
\begin{equation}
\|\bar h_{n,z}-h_z\|_n^2 + \lambda_n\|\bar h_{n,z}\|_{\mathcal H_\nu}^2 \,\leq\, \lambda_n\|h_z\|_{\mathcal H_\nu}^2,
\label{eq:krr-noiseless-basic}
\end{equation}
and hence \( \|\bar h_{n,z}-h_z\|_n \leq \lambda_n^{1/2}\|h_z\|_{\mathcal H_\nu}, \text{ and } \|\bar h_{n,z}-h_z\|_{\mathcal H_\nu} \leq 2\|h_z\|_{\mathcal H_\nu}. \) Since $t_n n^{-s/d}=o(\lambda_n^{1/2})$, the definition of $\mathcal E_n(t_n)$ and the norm equivalence between $H^s(\mathcal X)$ and $\mathcal H_\nu$ imply, on $\mathcal E_n(t_n)$,
$$
\|\bar h_{n,z}-h_z\|_{L^2} \,\leq\, C\lambda_n^{1/2}\|h_z\|_{\mathcal H_\nu}.
$$
Applying the Sobolev interpolation inequality and Equation~\eqref{eq:krr-rkhs-uniform-bound}, we obtain the following bound for $B_{n,z}$ on $\mathcal E_n(t_n)$:
\begin{equation}
\begin{aligned}
\|\bar h_{n,z}-h_z\|_\infty
\,\leq\, C \|\bar h_{n,z}-h_z\|_{L^2}^{1-\frac{d}{2s}} \|\bar h_{n,z}-h_z\|_{H^s}^{\frac{d}{2s}} \,\leq\, C\lambda_n^{\frac12-\frac{d}{4s}} \|h_z\|_{\mathcal H_\nu}\,\leq\, C\lambda_n^{\frac12-\frac{d}{4s}}.
\end{aligned}
\label{eq:krr-bias-high-prob}
\end{equation}
where $C$ is independent of $z$.

We now compare $S_{n,z}$ and $B_{n,z}$. Since $\lambda_n\asymp(\log n)/n$, we know $$
\lambda_n^{\frac12-\frac{d}{4s}} \,\asymp\, \left( \frac{\log n}{n} \right)^{\frac12-\frac{d}{4s}} \quad \text{and} \quad 
n^{-1/2}
\lambda_n^{-\frac{d}{4s}}
\sqrt{\log n} \,\asymp\, \left( \frac{\log n}{n}
\right)^{\frac12-\frac{d}{4s}}.
$$
Furthermore, because $s=\nu+d/2$, we have
\begin{equation}
\frac12-\frac{d}{4s} \,=\,
\frac{\nu}{2\nu+d}.
\label{eq:krr-exponent-identity}
\end{equation}
It follows from Equations~\eqref{eq:krr-bias-high-prob}, \eqref{eq:krr-stochastic-polynomial-tail}, \eqref{eq:krr-error-decomposition} and~\eqref{eq:krr-exponent-identity} that, for every realization of $X_{1:n}$ belonging to $\mathcal E_n(t_n)$,
\begin{equation}
\mathbb P \left( \left\| \hat h_n(z;\cdot)-h_z \right\|_\infty > C_A \left( \frac{\log n}{n} \right)^{\frac{\nu}{2\nu+d}} \,\middle|\, X_{1:n} \right) \,\leq\, \frac{1}{2}n^{-A}.
\label{eq:krr-conditional-final-bound}
\end{equation}
Combining Equation~\eqref{eq:krr-conditional-final-bound} with the design-event probability bound~\eqref{eq:krr-design-event-polynomial-tail} yields
\begin{equation}
\mathbb P_{\mathcal D_n}
\left(
\left\|
\hat h_n(z;\cdot)-h_z
\right\|_\infty
>
C_A
\left(
\frac{\log n}{n}
\right)^{\frac{\nu}{2\nu+d}}
\right)
\,\leq\,
n^{-A}.
\label{eq:krr-unconditional-final-bound}
\end{equation}
Recalling that $d=d_x$, Equation~\eqref{eq:krr-unconditional-final-bound} is exactly Equation~\eqref{eq:fixed-krr-main-bound}. All constants above depend only on $A$, $M$, the uniform Sobolev bound in Assumption~\ref{ass:krr_source}, the density bounds in Assumption~\ref{ass:Regularity_density}, and the kernel and domain parameters, and therefore can be chosen independently of $z$. This completes the proof. \Halmos
\end{proof}
\smallskip

Lemma~\ref{lem:fixed-decision-krr} establishes a high-probability sup-norm error bound for the estimated objective function at each fixed $z\in\mathcal Z$, with constants independent of $z$. To control the supremum over $\mathcal Z$ in Equation~\eqref{eq:PtO_error_decomposition0}, we need to establish a bound that holds simultaneously for all $z\in\mathcal Z$. We therefore use a covering-number argument to extend the fixed-decision result and obtain the expected uniform approximation error in Lemma~\ref{lem:krr_uniformization}.

\begin{lemma}[Uniform approximation error for KRR]
\label{lem:krr_uniformization}
Suppose the conditions of Proposition~\ref{prop:krr_matern_uniform_error} hold. Then, the KRR-weighted objective $\hat h_n$ in Equation~\eqref{eq:KRRestimator} satisfies
\begin{equation}
\label{eq: uniform approx error of h}
\mathbb E_{\mathcal D_n,X} \left[ \sup_{z\in\mathcal Z} \Big|\hat h_n(z;X)-h(z;X)\Big| \right] = \mathcal{O}\left( n^{-\frac{\nu}{2\nu+d_x}}\sqrt{\log n} \right).
\end{equation}
\end{lemma}

\smallskip
\begin{proof}{Proof of Lemma \ref{lem:krr_uniformization}}

Fix $\varepsilon>0$. By the compactness of $\mathcal Z$ in Assumption~\ref{ass:SAA-struct0}, there exists an $\varepsilon$-net $V_\varepsilon\subset\mathcal Z$ with $|V_\varepsilon|\le C\varepsilon^{-d_z}$ for all sufficiently small $\varepsilon$, where $C$ is independent of $\varepsilon$ (Chapter~4.2 of \citealtEC{Vershynin2018}). For any $z\in\mathcal Z$, choose $v\in V_\varepsilon$ such that $\|z-v\|\le\varepsilon$.
By the uniform Lipschitz condition in Assumption~\ref{ass:SAA-struct0}, we have $|h_z(x)-h_v(x)|\le G\varepsilon$ for all $x\in\mathcal X$. Moreover, by Equation~\eqref{eq:KRRestimator},
\[
|\hat h_n(z;x)-\hat h_n(v;x)| \,=\, \left|\sum_{i=1}^n\omega_{n,i}^{\rm KRR}(x)\big\{c(z;X_i,Y_i)-c(v;X_i,Y_i)\big\}\right| \,\le\, G\varepsilon\sum_{i=1}^n|\omega_{n,i}^{\rm KRR}(x)|.
\]

We next bound the $\ell_1$-norm of the KRR weights. Since $K_n\succeq0$ and $\lambda_n>0$, we have $\|(K_n+n\lambda_n I_n)^{-1}\|_{\rm op}\le (n\lambda_n)^{-1}$, where $\|A\|_{\rm op}:=\sup_{\|u\|_2=1}\|Au\|_2$ denotes the operator norm induced by the Euclidean norm. Moreover, the boundedness of $K_\nu$ on $\mathcal X$ implies $\sup_{x\in\mathcal X}\|k_n(x)\|_2\le C\sqrt n$. By the Cauchy--Schwarz inequality, we have \( \sum_{i=1}^n|\omega_{n,i}^{\rm KRR}(x)| \le \sqrt n\|\omega^{\rm KRR}(x)\|_2. \) Using the closed-form representation of the KRR weights and the definition of the operator norm, we obtain
\[
\begin{aligned}
\sup_{x\in\mathcal X}\sum_{i=1}^n|\omega_{n,i}^{\rm KRR}(x)|
&\,\le\, \sqrt n\sup_{x\in\mathcal X}\|\omega^{\rm KRR}(x)\|_2 \,=\, \sqrt n\sup_{x\in\mathcal X}\|(K_n+n\lambda_n I_n)^{-1}k_n(x)\|_2\\
&\,\le\, \sqrt n\|(K_n+n\lambda_n I_n)^{-1}\|_{\rm op}\sup_{x\in\mathcal X}\|k_n(x)\|_2 \,\le\, \sqrt n\cdot\frac{1}{n\lambda_n}\cdot C\sqrt n \,=\, \frac{C}{\lambda_n}.
\end{aligned}
\]

For any $z\in\mathcal Z$, choose $v\in V_\varepsilon$ such that $\|z-v\|\le\varepsilon$. By the triangle inequality and the preceding bounds, for every $x\in\mathcal X$,
\[
\begin{aligned}
|\hat h_n(z;x)-h_z(x)|
&\,\le\, |\hat h_n(v;x)-h_v(x)| +|\hat h_n(z;x)-\hat h_n(v;x)| +|h_v(x)-h_z(x)|\\
&\,\le\, |\hat h_n(v;x)-h_v(x)| +G\varepsilon\left(1+\frac{C}{\lambda_n}\right).
\end{aligned}
\]
Taking the supremum over $x\in\mathcal X$ and $z\in\mathcal Z$, we obtain
\begin{equation}
\label{eq:krr-uniformization-discretization}
\sup_{z\in\mathcal Z}\|\hat h_n(z;\cdot)-h_z\|_\infty \,\le\, \max_{v\in V_\varepsilon}\|\hat h_n(v;\cdot)-h_v\|_\infty + G\varepsilon\left(1+\frac{C}{\lambda_n}\right).
\end{equation}
Choose $\varepsilon_n=\lambda_n\left({\log n}/{n}\right)^{\frac{\nu}{2\nu+d_x}}$. Since $\lambda_n\asymp(\log n)/n$, the discretization term in Equation~\eqref{eq:krr-uniformization-discretization} satisfies \( G\varepsilon_n\left(1+{C}/{\lambda_n}\right) \le C\left({\log n}/{n}\right)^{\frac{\nu}{2\nu+d_x}}. \) Moreover, since $\lambda_n\asymp(\log n)/n$, we have 
\[
\varepsilon_n^{-1} \,=\, \mathcal{O}\left(\left({n}/{\log n}\right)^{1+\frac{\nu}{2\nu+d_x}}\right).
\]
Therefore, the covering-number bound gives \( |V_{\varepsilon_n}| \le C\varepsilon_n^{-d_z} \le Cn^{d_z(1+\frac{\nu}{2\nu+d_x})}. \) Applying the fixed-decision bound \eqref{eq:fixed-krr-main-bound} to all points in $V_{\varepsilon_n}$ and using the union bound, we obtain
\[
\mathbb P_{\mathcal D_n}\left( \max_{v\in V_{\varepsilon_n}}\|\hat h_n(v;\cdot)-h_v\|_\infty >C_A\left(\frac{\log n}{n}\right)^{\frac{\nu}{2\nu+d_x}} \right) \,\le\, |V_{\varepsilon_n}|n^{-A} \,\le\, Cn^{d_z\left(1+\frac{\nu}{2\nu+d_x}\right)-A}.
\]
Combining this bound with Equation~\eqref{eq:krr-uniformization-discretization} and the choice of $\varepsilon_n$, we obtain
\begin{equation}
\label{eq:krr-uniformization-tail}
\mathbb P_{\mathcal D_n}\left( \sup_{z\in\mathcal Z}\|\hat h_n(z;\cdot)-h_z\|_\infty >(C_A+C)\left(\frac{\log n}{n}\right)^{\frac{\nu}{2\nu+d_x}} \right) \,\le\, Cn^{d_z\left(1+\frac{\nu}{2\nu+d_x}\right)-A}.
\end{equation}
Since Equation~\eqref{eq:fixed-krr-main-bound} holds for any fixed $A>0$, we choose $A>d_z(1+\nu/(2\nu+d_x))+2$. Then, for some $\delta>0$, the right-hand side of Equation~\eqref{eq:krr-uniformization-tail} is bounded by $Cn^{-2-\delta}$.

We next derive the expected uniform approximation error. By Assumption~\ref{ass:SAA-struct0}, we have $|c(z;x,y)|\le M$ and $|h_z(x)|\le M$ uniformly over $z\in\mathcal Z$ and $x\in\mathcal X$. Together with the preceding bound on the KRR weights, this gives \( \sup_{z\in\mathcal Z}\|\hat h_n(z;\cdot)-h_z\|_\infty \le M\left(1+{C}/{\lambda_n}\right). \) Combining this deterministic bound with Equation~\eqref{eq:krr-uniformization-tail}, we obtain
\[
\mathbb E_{\mathcal D_n}\left[ \sup_{z\in\mathcal Z}\|\hat h_n(z;\cdot)-h_z\|_\infty \right] \,\le\, (C_A+C)\left(\frac{\log n}{n}\right)^{\frac{\nu}{2\nu+d_x}} + CM(1+\lambda_n^{-1})n^{-2-\delta} \,=\, \mathcal O\left(\left(\frac{\log n}{n}\right)^{\frac{\nu}{2\nu+d_x}}\right),
\]
where the last equality follows from $\lambda_n\asymp(\log n)/n$ and
\[
(1+\lambda_n^{-1})n^{-2-\delta} \,=\, \mathcal O\left(\frac{n^{-1-\delta}}{\log n}\right) \,=\, o\left(\left(\frac{\log n}{n}\right)^{\frac{\nu}{2\nu+d_x}}\right).
\]

Finally, because \( \sup_{z\in\mathcal Z}|\hat{h}_n(z;X) - h(z;X)| \le \sup_{z\in\mathcal Z}\|\hat{h}_n(z;\cdot) - h(z;\cdot)\|_\infty\), we have
\[
\mathbb E_{\mathcal D_n,X} \left[ \sup_{z\in\mathcal Z} |\hat h_n(z;X)-h(z;X)| \right] \,=\, \mathcal O\left(\left(\frac{\log n}{n}\right)^{\frac{\nu}{2\nu+d_x}}\right).
\]
Since \(\nu/(2\nu+d_x)<1/2\), we obtain Equation~\eqref{eq: uniform approx error of h}. Therefore, we complete the proof. \Halmos
\end{proof}
\smallskip

It remains to combine the uniform approximation bound in Lemma~\ref{lem:krr_uniformization} with the regret decomposition in Equation~\eqref{eq:PtO_error_decomposition0} to prove Proposition~\ref{prop:krr_matern_uniform_error}.

\smallskip
\proof{Proof of Proposition \ref{prop:krr_matern_uniform_error}}
For each \(x\in\mathcal{X}\), let $ z^*(x)\in\arg\min_{z\in\mathcal Z}h(z;x).$ Such a minimizer exists because \(\mathcal Z\) is compact and \(h(\cdot;x)\) is continuous. Similarly, \(\hat z_n(x)\) exists because \(\hat h_n(\cdot;x)\) is continuous on compact \(\mathcal Z\). By Equation~\eqref{eq:PtO_error_decomposition0}, we have \( h(\hat z_n(X);X)-h(z^*(X);X) \le 2\sup_{z\in\mathcal Z} |\hat h_n(z;X)-h(z;X)|. \) Taking expectations over \((\mathcal D_n,X)\) on the both sides yields
\[
\mathbb E_{\mathcal D_n,X} \Big[ h(\hat z_n(X);X)-h(z^*(X);X) \Big] \,\le\, 2\, \mathbb E_{\mathcal D_n,X} \left[ \sup_{z\in\mathcal Z} |\hat h_n(z,X)-h(z;X)| \right].
\]
Applying Lemma~\ref{lem:krr_uniformization} to the right-hand side of the above inequality, we complete the proof.
\Halmos
\endproof

\subsection{Proof of Theorem~\ref{thm:pto_lower_bound}}\label{app: proof of theorem 1}

We prove Theorem~\ref{thm:pto_lower_bound} by constructing a finite subclass of problem instances in $\mathcal P_\beta$. The subclass consists of a common quadratic cost function and a finite collection of data-generating models for $(X,Y)$, whose induced conditional objectives are infinitely differentiable in $x$. We show that, for any PtO procedure satisfying Assumption~\ref{ass:weights_pto}, the expected regret averaged over this subclass is bounded below by a constant multiple of $n^{-4/(d_x+4)}$. Since the worst-case expected regret over $\mathcal P_\beta$ is at least this average, the desired lower bound~follows.

\smallskip
\begin{proof}{Proof of Theorem~\ref{thm:pto_lower_bound}}
For notational simplicity, write $d=d_x$. To construct the finite subclass of problem instances we consider the following collection of conditional mean functions: \( \mathcal F_0:=\{f_1,\ldots,f_d,f_{\mathrm q}\}, \) where, for $x=(x_1,\ldots,x_d)\in[0,1]^d$, \( f_j(x)=x_j \text{ for all } j=1,\ldots,d, \text{ and } f_{\mathrm q}(x)=\frac12\|x\|^2. \)

For each $f\in\mathcal F_0$, we construct a problem instance by specifying the joint distribution of $(X,Y)$. Let \( X\sim\operatorname{Unif}([0,1]^d). \) Fix a constant $\sigma>0$, which determines the noise level in the outcome. Let $\varepsilon$ be independent of $X$ and satisfy \( \mathbb P(\varepsilon=\sigma) = \mathbb P(\varepsilon=-\sigma) = \frac12. \) For each $f\in\mathcal F_0$, define \( Y=f(X)+\varepsilon, \) and let $P_f$ denote the resulting joint distribution of $(X,Y)$. Since $\varepsilon$ has mean zero and variance $\sigma^2$,
\[
\mathbb E_f\left[Y\mid X=x\right] \,=\, f(x), \qquad \operatorname{Var}_f\left(Y\mid X=x\right) \,=\, \sigma^2
\]
for every $x\in[0,1]^d$. Let \( F_{\max}:=\max\left\{1,\frac d2\right\} \text{ and } R:=F_{\max}+\sigma. \) Since \(0\le f(x)\le F_{\max}\) for every \(f\in\mathcal F_0\) and \(x\in[0,1]^d\), and \(Y=f(X)+\varepsilon\) with \(\varepsilon\in\{-\sigma,\sigma\}\), the support of \(Y\) under every \(P_f\) is contained in $[-R,R]$. We therefore define \( \mathcal Y=[-R,R] \) as the outcome space for \(Y\).

For the problem instances constructed here, we take the feasible set to be \( \mathcal Z=[-R,R]. \) For all instances in the subclass, use the same cost function
\[
c(z;x,y) \,=\, \frac12 z^2-zy, \quad \text{for all } z\in\mathcal Z.
\]
The cost does not depend explicitly on $x$. It is convex in $z$, and
for $z_1,z_2\in\mathcal Z$ and $y\in\mathcal Y$,
\[
|c(z_2;x,y)-c(z_1;x,y)| \,=\, |z_2-z_1| \left| \frac{z_1+z_2}{2}-y \right|  \,\le\, 2R|z_2-z_1|,
\]
which is uniformly Lipschitz in $z$. Moreover, \(|c(z;x,y)| \le \frac12 R^2+R^2 = \frac32R^2. \) Hence all instances satisfy Assumption~\ref{ass:SAA-struct0} with fixed constants $D=2R$, $G=2R$, and $M=3R^2/2$.

The induced conditional objective under $P_f$ is \( h_f(z;x) = \mathbb E_f\left[ c(z;X,Y)| X=x \right] = \frac12 z^2-zf(x). \) Since every $f\in\mathcal F_0$ is a polynomial of degree at most two, $x\mapsto h_f(z;x)$ is $\beta$-times continuously differentiable for every nonnegative integer $\beta$, with all derivatives up to order $\beta$ uniformly bounded over $z\in\mathcal Z$ and $x\in[0,1]^d$. In addition, the covariate density is identically one on $[0,1]^d$, and the conditional variance of $Y$ is exactly $\sigma^2$. Therefore, \( \{(P_f,c):f\in\mathcal F_0\}\subset\mathcal P_\beta. \)
It is therefore sufficient to lower bound the average regret over this finite subclass. In particular,
\begin{eqnarray}
\lefteqn{ \sup_{(P,c)\in\mathcal P_\beta} \mathbb E \left[ \mathbb E \left[ h(\hat z_n(X);X)-h(z^\ast(X);X) \,\big|\, \mathcal{D}_n \right] \right] }\nonumber\\
&\ge & \max_{f\in\mathcal F_0} \mathbb E_f \left[ h_f(\hat z_n(X);X) - \min_{z\in\mathcal Z}h_f(z;X) \right] \,\ge\, \frac{1}{|\mathcal F_0|} \sum_{f\in\mathcal F_0} \mathbb E_f \left[ h_f(\hat z_n(X);X) - \min_{z\in\mathcal Z}h_f(z;X) \right].~~~~~
\end{eqnarray}
We next lower bound the last quantity. The analysis is divided into four steps.

\smallskip

Step 1: We first show that, under the constructed quadratic objective, the decision regret is proportional to the squared estimation error. We then apply the standard bias--variance decomposition to express the expected regret in terms of the squared bias and variance of the weighted estimator.

Fix a query point $x\in[0,1]^d$ and training covariates $X_{1:n}=(X_1,\ldots,X_n)$. By Assumption~\ref{ass:weights_pto}, the weights depend only on $(x,X_{1:n})$ and are then deterministic after conditioning on $x$ and $X_{1:n}$. Define the normalized weights \( \widetilde\omega_{n,i}(x) := \omega_{n,i}(x)/\sum_{j=1}^n\omega_{n,j}(x)\). Assumption~\ref{ass:weights_pto} implies \( \widetilde\omega_{n,i}(x)\ge0\) and \(\sum_{i=1}^n\widetilde\omega_{n,i}(x)=1.\)

Under $P_f$, the empirical PtO objective can be written as 
\[
\hat h_{n,f}(z;x) \,=\, \sum_{i=1}^n \omega_{n,i}(x) \left( \frac12 z^2-zy_i \right) \,=\, \Bigg(\sum_{i=1}^n\omega_{n,i}(x)\Bigg) \Bigg[ \frac12z^2 - z\sum_{i=1}^n \widetilde\omega_{n,i}(x)y_i \Bigg]. 
\]
Since every $y_i\in[-R,R]$ and the normalized weights form a convex combination, \( \sum_{i=1}^n\widetilde\omega_{n,i}(x)y_i\in[-R,R]=\mathcal Z. \) Hence the minimizer over the compact feasible set is exactly $ \hat z_n(x) = \sum_{i=1}^n \widetilde\omega_{n,i}(x)y_i.$ Similarly, since $f(x)\in\mathcal Z$, $ z_f^*(x) = \arg\min_{z\in\mathcal Z} h_f(z;x) = f(x).$ Consequently,
$$
h_f(\hat z_n(x);x)-h_f(z_f^*(x);x) \,=\, \frac12 \left[ \hat z_n(x)-f(x) \right]^2.
$$

Define the conditional bias
\begin{equation}
B_f(x) \,:=\, \sum_{i=1}^n \widetilde\omega_{n,i}(x)f(X_i) -f(x)
\end{equation}
and the variance proxy $ S(x) := \sum_{i=1}^n \widetilde\omega_{n,i}^2(x).$ Since \( Y_i=f(X_i)+\varepsilon_i, \) where $\varepsilon_1,\ldots,\varepsilon_n$ are conditionally independent, mean zero, and have variance $\sigma^2$, we have \( \hat z_n(x)-f(x) = B_f(x) + \sum_{i=1}^n \widetilde\omega_{n,i}(x)\varepsilon_i. \) Therefore, we have the decomposition,
\begin{align}
\mathbb E_f \left[ h_f\big(\hat z_n(x);x\big)-h_f\big(z_f^*(x);x\big) \,\big|\, X_{1:n} \right] \,=\,  \frac12 \mathbb E_f \Bigg[ \Big( B_f(x) + \sum_{i=1}^n \widetilde\omega_{n,i}(x)\varepsilon_i \Big)^2 \,\Big|\, X_{1:n} \Bigg] \,=\, \frac12 B_f(x)^2 + \frac{\sigma^2}{2} S(x). \label{eq:lower-bound-bias-variance}
\end{align}

\smallskip

Step 2: To establish a lower bound on the average squared bias over the finite subclass, we exploit its construction to relate this average to a geometric quantity that depends only on the training covariates and the PtO weights, rather than on the individual problem instances. Specifically, define the spatial dispersion of the weights around the query point by
\[
V(x) \,:=\, \sum_{i=1}^n
\widetilde\omega_{n,i}(x)\|X_i-x\|^2.
\]
A smaller $V(x)$ indicates that the weights are concentrated on observations close to the query point, whereas a larger $V(x)$ indicates that more weight is assigned to observations
farther away.

For the quadratic function $f_{\mathrm q}$, the conditional bias is \( B_{\mathrm q}(x) = \frac 12 ( \sum_{i=1}^n \widetilde\omega_{n,i}(x)\|X_i\|^2 - \|x\|^2 ), \) whereas for $f_j(x)=x_j$, the conditional bias is \( B_j(x) = \sum_{i=1}^n \widetilde\omega_{n,i}(x)X_{i,j} - x_j. \) Expanding $\|X_i-x\|^2$ and using $\sum_i\widetilde\omega_{n,i}(x)=1$ gives the exact identity $ V(x) = 2 B_{\mathrm q}(x) - 2 \sum_{j=1}^d x_j B_j(x). $ Since $x\in[0,1]^d$, we have $|x_j|\le1$, and hence \( \frac 12 V(x) \le |B_{\mathrm q}(x)| + \sum_{j=1}^d|B_j(x)|. \) By the Cauchy--Schwarz inequality,
\[ 
\Bigg( |B_{\mathrm q}(x)| + \sum_{j=1}^d|B_j(x)| \Bigg)^2 \,\le\, (d+1) \Bigg( B_{\mathrm q}(x)^2 + \sum_{j=1}^d B_j(x)^2 \Bigg).
\]
Therefore,
\begin{equation}
\frac{1}{d+1} \sum_{f\in\mathcal F_0} B_f(x)^2 \,\ge\, \frac{1}{4(d+1)^2} V(x)^2 \,=:\, C_1V(x)^2.
\label{eq:lower-bound-bias-dispersion}
\end{equation}

Define the average conditional regret over the finite subclass by
\[
\overline R_n(x) \,:=\, \frac{1}{d+1} \sum_{f\in\mathcal F_0} \mathbb E_f \left[ h_f(\hat z_n(x);x)-h_f(z_f^*(x);x) \mid X_{1:n} \right].
\]
Combining Equations~\eqref{eq:lower-bound-bias-variance} and \eqref{eq:lower-bound-bias-dispersion}, we obtain
\begin{equation}
\overline R_n(x) \,\ge\, \frac{C_1}{2}V(x)^2 + \frac{\sigma^2}{2}S(x).
\label{eq:lower-bound-regret-bv}
\end{equation}

\smallskip

Step 3: In this step, we give a lower bound on $\overline R_n(x)$ by establishing a tradeoff between the spatial dispersion $V(x)=\sum_{i=1}^n\widetilde\omega_{n,i}(x)\|X_i-x\|^2$ and the variance proxy $S(x)=\sum_{i=1}^n\widetilde\omega_{n,i}^2(x)$. By the Cauchy--Schwarz inequality, reducing $S(x)$ requires spreading the weights across more observations, which, combined with the spatial distribution of the training covariates, forces a larger $V(x)$ on a suitable event. This tradeoff yields a lower bound on $\overline R_n(x)$.

For a fixed query point $x$, let \( B(x,r) := \{u\in\mathbb R^d:\|u-x\|\le r\} \) and define \( N_n(x,r) := \sum_{i=1}^n \mathbb I\{X_i\in B(x,r)\}. \) Let $v_d$ denote the volume of the Euclidean unit ball in $\mathbb R^d$. Since $X$ is uniform on $[0,1]^d$, \( \mathbb P(X_i\in B(x,r)) \le v_dr^d \) uniformly over $x\in[0,1]^d$. For $k=1,\ldots,n$, define $r_k := ( \frac{k}{4ev_dn} )^{1/d}.$ Then \( p_k := \mathbb P(X_i\in B(x,r_k)) \le \frac{k}{4en}. \) Since $N_n(x,r_k)\sim\operatorname{Bin}(n,p_k)$, it is stochastically dominated by \( Z_k \sim \operatorname{Bin} ( n,\frac{k}{4en} ).\) Writing $\mu_k=\mathbb E[Z_k]=k/(4e)$ and applying the standard Chernoff bound,
$$
\mathbb P\bigl(N_n(x,r_k)\ge k\bigr) \,\le\, \mathbb P(Z_k\ge k) \,\le\, \left( \frac{e\mu_k}{k} \right)^k \,=\, 4^{-k}.
$$

Define the event \( \mathcal E(x) := \bigcap_{k=1}^n \{N_n(x,r_k)<k\}. \) By the union bound, \( \mathbb P(\mathcal E(x)^c) \le \sum_{k=1}^n4^{-k} < \frac13 \) and
\begin{equation}
\mathbb P(\mathcal E(x)) \,\ge\, \frac23.
\label{eq:lower-bound-sparse-event-prob}
\end{equation}
Let $R_{(k)}(x)$ denote the distance from $x$ to its $k$th nearest training covariate. Then, \( \mathcal E(x) = \left\{ R_{(k)}(x)>r_k,~ k=1,\ldots,n \right\}. \) This event ensures that the $k$th nearest training covariate cannot be too close to the query point $x$, preventing the training covariates from being excessively concentrated around $x$. Consequently, spreading the weights across more observations to reduce $S(x)$ necessarily assigns substantial weight to observations farther from $x$, providing a lower bound on the spatial dispersion $V(x)$. We now formalize this intuition by considering two cases according to the value of $S(x)$.

\textit{Case 1: $S(x)\le1/8$.} By Cauchy--Schwarz inequality, \( 1 = ( \sum_{i=1}^n\widetilde\omega_{n,i}(x) )^2 \le nS(x), \) so $S(x)\ge1/n$. Define \( K := \lfloor \frac{1}{4S(x)} \rfloor. \) On $\{S(x)\le1/8\}$, we have $K\ge2$, while $S(x)\ge1/n$ ensures $K\le n$. Let $I_K$ denote the indices of the $K$ nearest training covariates to $x$. By Cauchy--Schwarz inequality,
\begin{equation}
W_{\mathrm{in}} \,:=\, \sum_{i\in I_K} \widetilde\omega_{n,i}(x) \,\le\, \sqrt{ K \sum_{i\in I_K} \widetilde\omega_{n,i}^2(x) } \,\le\, \sqrt{KS(x)} \,\le\, \frac12.
\end{equation}
Hence the total weight outside the $K$ nearest neighbors satisfies \( W_{\mathrm{out}}\ge\frac12. \) On the event $\mathcal E(x)$, every point outside $I_K$ is at distance at least $R_{(K)}(x)>r_K$ from $x$. Therefore,
\begin{equation}
V(x) \,=\, \sum_{i=1}^n \widetilde\omega_{n,i}(x) \|X_i-x\|^2 \,\ge\, W_{\mathrm{out}}r_K^2 \,\ge\, \frac12 \left( \frac{K}{4ev_dn} \right)^{2/d}.
\label{eq:lower-bound-dispersion-knn}
\end{equation}
Since $1/(4S(x))\ge2$, the inequality $\lfloor u\rfloor\ge u/2$ for $u\ge2$ gives \( K \ge \frac{1}{8S(x)}. \) Substituting this into Equation~\eqref{eq:lower-bound-dispersion-knn} yields
\begin{equation}
V(x)^2 \,\ge\, c_2 \bigl(nS(x)\bigr)^{-4/d},
\label{eq:lower-bound-dispersion-S}
\end{equation}
where \( c_2 := \frac14 (32ev_d)^{-4/d} >0. \)
Combining Equations~\eqref{eq:lower-bound-regret-bv} and \eqref{eq:lower-bound-dispersion-S}, on $\{S(x)\le1/8\}\cap\mathcal E(x)$ we have
$$ \overline R_n(x) \,\ge\, A n^{-4/d}S(x)^{-4/d} + B S(x), $$
where \( A:=\frac{C_1c_2}{2}, \text{ and } B:=\frac{\sigma^2}{2}. \) Let $q:=4/d$. For $s>0$, define \( g_n(s) = A n^{-q}s^{-q}+Bs. \) Elementary calculus gives the minimizer \( s_n^* = ( qA/B )^{1/(q+1)} n^{-q/(q+1)}, \) and hence
\begin{equation}
\inf_{s>0}g_n(s) \,=\, c_3 n^{-q/(q+1)} \,=\, c_3 n^{-4/(d+4)},
\end{equation}
where \(c_3:=A\left( qA/B \right)^{-q/(q+1)}+B\left(qA/B \right)^{1/(q+1)}>0.\) Thus, on $\{S(x)\le1/8\}\cap\mathcal E(x)$, $\overline R_n(x) \ge c_3 n^{-4/(d+4)}.$

\textit{Case 2: $S(x)>1/8$.}
In this case, the variance term in Equation~\eqref{eq:lower-bound-regret-bv} alone gives \( \overline R_n(x) \ge \frac{\sigma^2}{2}S(x) > \frac{\sigma^2}{16}. \) Since $n^{-4/(d+4)}\le1$ for $n\ge1$, $\overline R_n(x) \ge \frac{\sigma^2}{16} n^{-4/(d+4)}.$ The two cases partition the sample space. Therefore, with \( c_4 := \min \{ c_3,\frac{\sigma^2}{16} \} >0, \) we obtain, on $\mathcal E(x)$, $\overline R_n(x) \ge c_4 n^{-4/(d+4)}.$

\smallskip

Step 4: We finally complete the minimax lower bound. 
Taking expectation with respect to the common training-covariate law and using $\overline R_n(x)\ge0$ together with Equation~\eqref{eq:lower-bound-sparse-event-prob},
\begin{equation}
\mathbb E_{X_{1:n}} [\overline R_n(x)] \,\ge\, \mathbb E_{X_{1:n}} \left[ \overline R_n(x) \mathbb I_{\mathcal E(x)} \right] \,\ge\, c_4n^{-4/(d+4)} \mathbb P(\mathcal E(x)) \,\ge\, \frac23c_4n^{-4/(d+4)}.
\label{eq:lower-bound-pointwise-expectation}
\end{equation}
The bound is uniform in $x\in[0,1]^d$. Let $X$ be an independent test covariate drawn from the same uniform distribution. Integrating Equation~\eqref{eq:lower-bound-pointwise-expectation} with respect to $X$ gives
$$ 
\frac{1}{d+1} \sum_{f\in\mathcal F_0} \mathbb E_f \left[ h_f\big(\hat z_n(X);X\big) - h_f\big(z_f^*(X);X\big) \right] \,\ge\, \frac23c_4n^{-4/(d+4)}. 
$$
Consequently,
\[
\max_{f\in\mathcal F_0} \mathbb E_f \left[ h_f\big(\hat z_n(X);X\big) - h_f\big(z_f^*(X);X\big) \right] \,\ge\, \frac23c_4n^{-4/(d+4)}.
\] Since
\( \{(P_f,c):f\in\mathcal F_0\} \subset\mathcal P_\beta, \)
we conclude that
\[
\sup_{(P,c)\in\mathcal P_\beta} \mathbb E \left[ h\big(\hat z_n(X);X\big) - h\big(z^*(X);X\big) \right] \,\ge\, c_0 n^{-4/(d+4)},
\]
where \( c_0:=\frac23c_4>0. \) Finally, by the tower property,
\[
\mathbb E \left[ h\big(\hat z_n(X);X\big)-h\big(z^*(X);X\big) \right] \,=\, \mathbb E \left[ \mathbb E \left[ h\big(\hat z_n(X);X\big)-h\big(z^*(X);X\big) \mid D_n \right] \right].
\]
Therefore,
\[
\sup_{(P,c)\in\mathcal P_\beta} \mathbb E \left[ \mathbb E \left[ h\big(\hat z_n(X);X\big) - h\big(z^*(X);X\big) \mid D_n \right] \right] \,\ge\, c_0n^{-4/(d_x+4)}.
\]
Therefore, we complete the proof. \Halmos
\end{proof}

\section{Proofs in Section~\ref{sec:GtO}}

In this section, we prove Lemmas~\ref{lem:SAA-error0}--\ref{lem:model-error0}, Theorem~\ref{thm:regret-GtO0}, and Remark~\ref{remark: strongly convex} respectively.

\subsection{Proof of Lemma~\ref{lem:SAA-error0}}\label{app: proof of lemma 1}

\proof{Proof.}
We suppress the conditioning on $\mathcal{D}_n$ and $X=x_0$ in the notation and write expectations and probabilities with respect to the simulation draws $\hat{Y}_1,\hat{Y}_2,\dots,\hat{Y}_m$. For any $z\in\mathcal{Z}$, we have $\bar{h}(z;x_0) \le \hat{h}(z;x_0) + \sup_{z\in\mathcal{Z}} |\hat{h}(z;x_0)-\bar{h}(z;x_0) |$ and $\hat{h}(z;x_0)  \le \bar{h}(z;x_0) + \sup_{z\in\mathcal{Z}} |\hat{h}(z;x_0) -\bar{h}(z;x_0)|$. Then,
\begin{equation*}\label{eq:difh}
\begin{aligned}
	\bar{h}(\hat z_{n,m}^\ast(x_0);x_0) &\,\leq\, \hat{h}(\hat z_{n,m}^\ast(x_0);x_0) + \sup_{z\in\mathcal{Z}} |\hat{h}(z;x_0) - \bar{h}(z;x_0) | \\
	&\,\leq\, \hat{h}(\bar z_n^\ast(x_0);x_0) + \sup_{z\in\mathcal{Z}} |\hat{h}(z;x_0) - \bar{h}(z;x_0)  |\ \leq\ \bar{h}(\bar z_n^\ast(x_0);x_0) + 2 \sup_{z\in\mathcal{Z}} |\hat{h}(z;x_0) - \bar{h}(z;x_0) |
\end{aligned}
\end{equation*}
which implies 
\begin{equation}\label{eq:SAA-regret-reduction}
	\bar h(\hat z_{n,m}^\ast(x_0);x_0) - \bar h(\bar z_n^\ast(x_0);x_0)
	\;\le\;
	2\sup_{z\in\mathcal{Z}} |\hat h(z;x_0) - \bar h(z;x_0)|.
\end{equation}
Define the empirical objective error process
\(
H_m(z)
:=
\hat h(z;x_0) - \bar h(z;x_0) 
=
\frac{1}{m}\sum_{j=1}^m c(z; x_0,\hat{Y}_j) - \E[c(z; x_0,\hat{Y})],
\)
where $\hat{Y}, \hat{Y}_1,\hat{Y}_2,\ldots,\hat{Y}_m \sim \hat{F}_{Y|X}(\cdot|x_0)$ are i.i.d.
Taking expectations on both sides of Equation~\eqref{eq:SAA-regret-reduction},
\begin{equation}\label{eq:SAA-reduction0}
\E\left[ \bar h(\hat z_{n,m}^\ast(x_0);x_0) - \bar h(\bar z_n^\ast(x_0);x_0) \right]
\;\le\;
2\; \E\left[ \sup_{z\in\mathcal{Z}} |H_m(z)|\right].
\end{equation}
Thus it remains to bound  $\E[\sup_{z\in\mathcal{Z}} |H_m(z)|]$.

We bound the expected supremum by a multiscale covering argument. Fix a reference point $z_{\mathrm{ref}}\in\mathcal Z$ and define $r_k:=D2^{-k}$ for $k\in\mathbb{N}$, set $\mathcal N_0:=\{z_{\mathrm{ref}}\}$. Without loss of generality, we assume $D>0$. For each $k\geq1$, choose an $r_k$-net $\mathcal N_k\subseteq\mathcal Z$, and write $N_k:=|\mathcal N_k|$. The volumetric covering bound (see Proposition~4.2.12 of \citealtEC{Vershynin2018}) gives
\[
N_k \,\leq\, \left(1+\frac{2D}{r_k}\right)^{d_z} \,=\, \left(1+2^{k+1}\right)^{d_z} \,\leq\,  2^{(k+2)d_z}.
\]
For every $z\in\mathcal Z$, choose $\pi_k(z)\in\mathcal N_k$ such that $\|z-\pi_k(z)\|\leq r_k$, and set $\pi_0(z):=z_{\mathrm{ref}}$. Then
\[
\|\pi_k(z)-\pi_{k-1}(z)\|
\,\leq\, 
r_k+r_{k-1}
\,=\, 
3r_k,
\quad k\geq1.
\]
Define the nonempty finite set
\(
\mathcal A_k:=
\{
(u,v)\in\mathcal N_k\times\mathcal N_{k-1}:
\|u-v\|\leq3r_k
\}.
\)
In particular, $|\mathcal A_k|\leq N_kN_{k-1}$.

We first control the empirical increments on $\mathcal A_k$.
For $(u,v)\in\mathcal A_k$, define
\(
W_j(u,v):=
c(u;x_0,\hat Y_j)-c(v;x_0,\hat Y_j),~j=1,2,\ldots,m.
\)
By Assumption~\ref{ass:SAA-struct0},
\(
|W_j(u,v)|
\leq G\|u-v\|
\leq3Gr_k.
\)
Moreover,
\[
H_m(u)-H_m(v)
\,=\,
\frac1m\sum_{j=1}^m
\bigl(W_j(u,v)-\mathbb E[W_j(u,v)]\bigr).
\]
Hoeffding's lemma and independence of the simulation draws imply
that, for every $\lambda\in\mathbb R$,
\[
\mathbb E\left[\exp\!\left\{
\lambda\bigl(H_m(u)-H_m(v)\bigr)
\right\}\right]
\,\leq\,
\exp\!\left\{
\frac{\lambda^2}{2}
\left(\frac{3Gr_k}{\sqrt m}\right)^2
\right\}.
\]
Thus every increment is centered and sub-Gaussian with parameter
$\sigma_k:=3Gr_k/\sqrt m$.
For every $\lambda>0$,
\[
\mathbb E \left[ \max_{(u,v)\in\mathcal A_k}
|H_m(u)-H_m(v)| \right]
\,\leq\,
\frac1\lambda
\log\mathbb E \left[
\exp\!\left\{
\lambda\max_{(u,v)\in\mathcal A_k}|H_m(u)-H_m(v)|
\right\}\right]
\,\leq\,
\frac{\log(2|\mathcal A_k|)}{\lambda}
+\frac{\lambda\sigma_k^2}{2},
\]
where the first inequality holds by Jensen's inequality and the second inequality holds by applying the exponential bound to both signs of all increments. 
Optimizing over $\lambda>0$ yields
\[
\mathbb E \left[\max_{(u,v)\in\mathcal A_k}
|H_m(u)-H_m(v)| \right]
\,\leq\,
\sigma_k\sqrt{2\log(2|\mathcal A_k|)}.
\]
Since $N_0=1$ and $N_k\leq2^{(k+2)d_z}$ for $k\geq1$, we have
\(
\log(2|\mathcal A_k|)
\leq \log(2N_kN_{k-1})
\leq [1+(2k+3)d_z]\log2
\leq 2d_z(k+2)\log2.
\)
Consequently,
\[
\mathbb E \left[\max_{(u,v)\in\mathcal A_k}
|H_m(u)-H_m(v)| \right]
\,\leq\,
\frac{6GD\sqrt{d_z\log2}}{\sqrt m}
2^{-k}\sqrt{k+2}.
\]

We next sum these increments across scales.
Because both $\hat h$ and $\bar h$ are $G$-Lipschitz on $\mathcal Z$, we have
\(
|H_m(z)-H_m(z')|
\leq2G\|z-z'\|.
\)
For any integer $K\geq1$, the telescoping identity
\[
H_m(z)
\,=\,
H_m(z_{\mathrm{ref}})
+\sum_{k=1}^K
\bigl(H_m(\pi_k(z))-H_m(\pi_{k-1}(z))\bigr)
+H_m(z)-H_m(\pi_K(z))
\]
therefore gives
\[
\sup_{z\in\mathcal Z}|H_m(z)|
\,\leq\,
|H_m(z_{\mathrm{ref}})|
+\sum_{k=1}^K
\max_{(u,v)\in\mathcal A_k}|H_m(u)-H_m(v)|
+2Gr_K.
\]
At the fixed reference point $z_{\mathrm{ref}}$, independence and the bounded-cost
assumption give
\[
\mathbb E \left[ |H_m(z_{\mathrm{ref}})| \right]
\,\leq\,
\sqrt{\mathbb E \left[ |H_m(z_{\mathrm{ref}})|^2 \right]}
\,=\,
\sqrt{
\frac{
\operatorname{Var}\left(c(z_{\mathrm{ref}};x_0,\hat Y)\right)
}{m}
}
\,\leq\,
\frac{M}{\sqrt m}.
\]
Taking expectations in the supremum bound yields
\[
\mathbb E \left[\sup_{z\in\mathcal Z}|H_m(z)| \right]
\,\leq\,
\frac{M}{\sqrt m}
+
\frac{6GD\sqrt{d_z\log2}}{\sqrt m}
\sum_{k=1}^K2^{-k}\sqrt{k+2}
+
2GD2^{-K}.
\]
Since $\sqrt{k+2}\leq k+1$ for $k\geq1$ and
\(
\sum_{k=1}^{\infty}2^{-k}\sqrt{k+2}
\leq
\sum_{k=1}^{\infty}2^{-k}(k+1)
=3.
\)
Letting $K\to\infty$, we conclude that
\[
\mathbb E \left[ \sup_{z\in\mathcal Z}|H_m(z)| \right]
\,\leq\,
\frac{M+18GD\sqrt{d_z\log2}}{\sqrt m}.
\]

Finally, applying the regret comparison established above gives
\[
\mathbb E\!\left[
\bar h(\hat z_{n,m}^*(x_0);x_0)
-\bar h(\bar z_n^*(x_0);x_0)
\right]
\,\leq\,
2\mathbb E \left[ \sup_{z\in\mathcal Z}|H_m(z)| \right]
\,\leq\,
\frac{2M+36GD\sqrt{d_z\log2}}{\sqrt m}.
\]
Thus we may take
$C_{\mathrm{SAA}}:=2M+36GD\sqrt{d_z\log2}$.
Therefore, we complete the proof. \Halmos
\endproof

\subsection{Proof of Lemma~\ref{lem:model-error0}}\label{app: proof of lemma 2}

\proof{Proof.}
By Assumption~\ref{ass:SAA-struct0}, we have $|c(z; x_0,y)|\le M$ for all $y$, so a standard characterization of total variation distance (via test functions with bounded range) yields
\begin{equation}\label{eq:LipofExp0}
\big|h(z;x_0) - \bar{h}(z;x_0)\big|
\,\le\, 2M\,\mathrm{TV}(F_{Y|X}(\cdot|x_0),\hat{F}_{Y|X}(\cdot|x_0))
\quad\text{for all }z\in\mathcal{Z}.
\end{equation}
Taking expectations with respect to $X$ and then with respect to $\mathcal{D}_n$ yields
\[
\E\Big[\E\Big[\sup_{z\in\mathcal{Z}}|h(z;X) - \bar{h}(z;X) | \,\big|\, \mathcal{D}_n \Big]\Big] \,\le\, 2M\, \E\Big[\E\big[\mathrm{TV}(\hat{F}_{Y|X}(\cdot|X),F_{Y|X}(\cdot|X)) \,\big|\, \mathcal{D}_n \big]\Big].
\]
Applying Assumption~\ref{ass: TV-rate}, we have Equation~\eqref{eq:model-error-bound-regret}. Therefore, we complete the proof. \Halmos
\endproof

\subsection{Proof of Theorem~\ref{thm:regret-GtO0}}\label{app: proof of theorem 2}

\proof{Proof.}
Taking expectations with respect to $X$ and then with respect to $\mathcal{D}_n$ on both sides of Equation~\eqref{eq:error_decomposition0}, and applying Lemmas~\ref{lem:SAA-error0} and~\ref{lem:model-error0} yields
\[
\E\Big[\E\Big[\mathcal{R}(X) \,\big|\, \mathcal{D}_n \Big]\Big] \;\le\; C_{\mathrm{SAA}}\,m^{-1/2} + 4M\,C_{\mathrm{M}}\,n^{-\alpha_0}(\log n)^{\alpha_1}.
\]
Setting $C := C_{\mathrm{SAA}} + 4M C_{\mathrm{M}}$ gives Equation~\eqref{eq:regret-bound-GtO}. Therefore, we complete the proof. \Halmos
\endproof

\subsection{Proof of Remark~\ref{remark: strongly convex}}
\label{app:stronglyconvex}

In this section, we prove Equation~\eqref{eq:decision-error-bound0} in Remark~\ref{remark: strongly convex}. We continue to use the notations in Section~\ref{sec:GtO}. Specifically, we denote $F_{Y|X}(\cdot|x)$ as the conditional distribution of $Y$ conditioned on $X=x$. On the offline stage, we use the historical data $\mathcal{D}_n =\{(x_i,y_i)\}_{i=1}^n$ to learn the conditional distribution or a simulator for $Y|\{X=x\}$, denoted as $\hat{F}_{Y|X}(\cdot|x)$. On the online stage, when a new covariate $X=x_0$ arrives, we generate $m$ i.i.d. samples from $\hat{F}_{Y|X}(\cdot|x_0)$, denoted as $\hat{Y}_1,\hat{Y}_2,\ldots,\hat{Y}_m\sim \hat{F}_{Y|X}(\cdot|x_0)$. Then, we can use SAA to construct a sample objective function $\hat{h}$ and solve the simulation-based CSO problem accordingly as
\begin{equation}
\hat{z}_{n,m}^\ast(x_0) \in \arg\min_{z \in \mathcal{Z}} \left\{ \hat{h}(z;x_0):= \frac{1}{m}\sum_{j=1}^m c\big(z;x_0,\hat{Y}_j\big) \right\}.
\end{equation}

To isolate the impact of the model error from the SAA error, we define an auxiliary decision rule as
\begin{equation}
\bar{z}_n^{\ast}(x_0) \in \arg\min_{z \in \mathcal{Z}} \Big\{ \bar{h}(z;x_0):=  \E\big[c(z;X,\hat{Y})\,|\,X = x_0 \big]\Big\},
\end{equation}
where $\hat{Y}\sim\hat{F}_{Y|X}(\cdot|x_0)$. Then we have the basic decomposition
\begin{equation}\label{eq:error_decomposition}
\underbrace{\| \hat{z}_{n,m}^\ast(x_0) - z^\ast(x_0) \|}_{\text{decision error}} \,\leq\, \underbrace{\| \hat{z}_{n,m}^\ast(x_0) - \bar{z}_n^{\ast}(x_0) \|}_{\text{SAA error}} \,+\, \underbrace{\| \bar{z}_n^{\ast}(x_0) - z^\ast(x_0) \|}_{\text{model error}}. 
\end{equation}
We discuss the SAA error and model error separately, and then summarize their combined impact on the decision error. Throughout this subsection, Assumption~\ref{ass:SAA-struct0} remains in force. We
additionally impose the following differentiability, smoothness, and
strong-convexity conditions.

\begin{assumption}\label{ass:SAA-struct}
	Suppose that
	\begin{enumerate}
		\item[(a)] The feasible set $\mathcal{Z}\subset\R^{d_z}$ is nonempty, convex, and compact, with diameter
		$$
		\mathrm{diam}(\mathcal{Z})
		:= \sup\big\{\|z_1-z_2\| : z_1,z_2\in\mathcal{Z}\big\} \le D < \infty.
		$$
		\item[(b)] For every $(x,y)\in\mathcal{X}\times\mathcal{Y}$, the cost function $c(\cdot; x,y)$ is continuously differentiable on $\mathcal{Z}$, and there exists a constant $G<\infty$ such that for all $z\in\mathcal{Z}$ and all $(x,y)\in\mathcal{X}\times\mathcal{Y}$,
		\[
		\|\nabla_z c(z; x,y)\| \le G.
		\]
		\item[(c)] The cost function $c(\cdot; x,y)$ is uniformly $\mu$-strongly convex and has uniform $L$-Lipschitz gradient on $\mathcal{Z}$, i.e., there exist constants $\mu> 0$ and $L\geq 0$ such that for all $z_1,z_2\in\mathcal{Z}$ and all $(x,y)\in\mathcal{X}\times\mathcal{Y}$,
		\begin{align*}
			c(z_2; x,y) - c(z_1; x,y)
			&\ge 
			\nabla_z c(z_1; x,y)^\top (z_2-z_1)
			+ \frac{\mu}{2} \|z_2-z_1\|^2, \\
			\|\nabla_z c(z_2; x,y) - \nabla_z c(z_1; x,y)\|
			&\le L \|z_2-z_1\|.
		\end{align*}
	\end{enumerate}
\end{assumption}

Assumption~\ref{ass:SAA-struct}(a) is standard in stochastic programming and guarantees that minimizers exist, see, e.g., \citetEC[Chapter~5]{ShapiroDentchevaRuszczynski2009}.
Assumptions~\ref{ass:SAA-struct}(b)--(c) impose uniform smoothness and strong convexity of the cost function with respect to $z$, uniformly over $(x,y)$; such conditions are routinely used in the analysis of M-estimators and SAA problems to obtain stability and convergence rates for the minimizers; see, e.g., \citetEC[Chapter~5]{ShapiroDentchevaRuszczynski2009} and \citetEC[Chapter~5]{vanderVaart1998asymptotic}; they are also used in the literature of CSO problem, see, e.g., \citeEC{hu2020sample,hu2020biased} and some references in the survey paper \citeEC{sadana25survey}.

\subsubsection{The SAA Error.}

We analyze the first term in the right-hand side of Equation~\eqref{eq:error_decomposition}, i.e., the term
$
\|\hat{z}_{n,m}^\ast(x_0) - \bar{z}_n^{\ast}(x_0)\|.
$
We fix the covariate value $x_0$, and treat the offline data $\mathcal{D}_n$ and the learned distribution or simulator $\hat{F}_{Y|X}(\cdot|x_0)$  as given. All expectations are taken with respect to the simulation draws $\hat{Y}_1,\dots,\hat{Y}_m$ conditional on $\mathcal{D}_n$ and $X=x_0$.

\begin{lemma}\label{lem:SAA-error}
	Suppose Assumption~\ref{ass:SAA-struct} holds. Fix a covariate value $x_0$ and a realization of the learned conditional distribution or simulator $\hat{F}_{Y|X}(\cdot|x_0)$.
	Then there exists a constant $C_{\mathrm{SAA}}<\infty$, depending only on $(\mu,L,G,D,d_z)$, such that for all $m\ge 2$,
	\begin{equation}\label{eq:SAA-error-bound}
		\E\left[\|\hat{z}_{n,m}^\ast- \bar{z}_n^{\ast}\| \,\big|\, \mathcal{D}_n,X=x_0 \right]
		\;\le\;
		C_{\mathrm{SAA}}\,m^{-1/2}.
	\end{equation}
\end{lemma}

Lemma~\ref{lem:SAA-error} shows that, once the simulator $\hat{F}_{Y|X}(\cdot|x_0)$ is fixed, the SAA-based decision $\hat{z}_{n,m}^\ast$ converges to the population decision $\bar{z}_n^{\ast}$ at the usual Monte Carlo rate $m^{-1/2}$ which reflects the fact that $\hat{h}$ is an empirical average of $m$ i.i.d. samples.
In the proof, we bound the gradient error and transfer it to decision error by using the strong monotonicity of the empirical objective's gradient and the constrained optimality conditions.
\smallskip

\proof{Proof.}
We suppress the conditioning on $\mathcal D_n$ and $X=x_0$ in the notation and write expectations with respect to the simulation draws $\hat Y_1,\ldots,\hat Y_m$. For brevity, write $\bar z:=\bar z_n^*(x_0)$ and $\hat z:=\hat z_{n,m}^*(x_0)$. Assumption~\ref{ass:SAA-struct}(b)--(c), together with the standing Assumption~\ref{ass:SAA-struct0}, implies that $\bar h(\cdot;x_0)$ and $\hat h(\cdot;x_0)$ are differentiable and $\mu$-strongly convex on $\mathcal Z$. Their minimizers are therefore unique.
The constrained first-order optimality conditions are
\[
\left\langle
\nabla\bar h(\bar z;x_0),z-\bar z
\right\rangle\geq0
\quad \text{and} \quad
\left\langle
\nabla\hat h(\hat z;x_0),z-\hat z
\right\rangle\geq0,
\qquad \text{for } z\in\mathcal Z.
\]
Let $d:=\hat z-\bar z$.
Taking $z=\hat z$ in the first inequality and $z=\bar z$
in the second gives
\[
\left\langle\nabla\bar h(\bar z;x_0),d\right\rangle\geq0 \quad
\text{and} \quad
\left\langle\nabla\hat h(\hat z;x_0),d\right\rangle\leq0.
\]
Combining with the strong monotonicity of the gradient of the empirical objective $\hat h(\cdot;x_0)$, the above constrained first-order optimality conditions, and Cauchy–Schwarz inequality, we have
\begin{eqnarray}
\lefteqn{\mu\|d\|^2
\,\leq\,
\left\langle
\nabla\hat h(\hat z;x_0)
-\nabla\hat h(\bar z;x_0),d
\right\rangle} \nonumber\\
&\leq&
-\left\langle\nabla\hat h(\bar z;x_0),d\right\rangle
\,\leq\,
\left\langle
\nabla\bar h(\bar z;x_0)
-\nabla\hat h(\bar z;x_0),d
\right\rangle \,\leq\,
\left\|
\nabla\bar h(\bar z;x_0)
-\nabla\hat h(\bar z;x_0)
\right\|\|d\|. \nonumber
\end{eqnarray}
If $d=0$, the desired comparison is immediate.
Otherwise, dividing by $\mu\|d\|$ yields
\[
\|\hat z-\bar z\|
\,\leq\,
\frac1\mu
\left\|
\nabla\hat h(\bar z;x_0)
-\nabla\bar h(\bar z;x_0)
\right\|.
\]

Conditional on $\mathcal D_n$ and $X=x_0$, the population
minimizer $\bar z$ is fixed.
Define
\(
V_j:=\nabla_zc(\bar z;x_0,\hat Y_j)
\text{ and } a:=\mathbb E[V_j]=\nabla\bar h(\bar z;x_0).
\)
The vectors $V_1,\ldots,V_m$ are i.i.d., and
$\|V_j\|\leq G$ almost surely.
Furthermore,
\[
\nabla\hat h(\bar z;x_0)-\nabla\bar h(\bar z;x_0)
\,=\,
\frac1m\sum_{j=1}^m(V_j-a).
\]
For $i\ne j$, independence and centering imply
\(
\mathbb E[\langle V_i-a,V_j-a\rangle]
=
\langle
\mathbb E[V_i-a],\mathbb E[V_j-a]
\rangle
=0.
\)
Therefore,
\begin{eqnarray}
\lefteqn{\mathbb E\left[\left\|
\nabla\hat h(\bar z;x_0)-\nabla\bar h(\bar z;x_0)
\right\|^2\right]} \nonumber \\
&=&
\frac1{m^2}
\sum_{i=1}^m\sum_{j=1}^m
\mathbb E\left[ \langle V_i-a,V_j-a\rangle \right]
\,=\,
\frac1m\mathbb E \left[\|V_1-a\|^2\right]
\,=\,
\frac1m\left(
\mathbb E\left[ \|V_1\|^2\right] -\|a\|^2
\right) 
\,\leq\,
\frac{G^2}{m}. \nonumber
\end{eqnarray}
Squaring the minimizer comparison and taking expectations gives
\(
\mathbb E[\|\hat z-\bar z\|^2]
\leq\frac{G^2}{\mu^2m}.
\)
Finally, the Cauchy--Schwarz inequality yields
\[
\mathbb E\left[ \|\hat z-\bar z\| \right]
\leq
\left(
\mathbb E \left[ \|\hat z-\bar z\|^2 \right]
\right)^{1/2}
\leq\frac{G}{\mu\sqrt m}.
\]
Restoring the conditioning and the original notation proves
the claimed bound. Therefore, we complete the~proof. \Halmos
\endproof

\smallskip

\subsubsection{The Model Error.}

We now turn to the second term in Equation~\eqref{eq:error_decomposition}, i.e., the term $\|\bar{z}_n^{\ast}(x_0) - z^\ast(x_0)\|$. 
We begin by quantifying how sensitive the optimal solution is to perturbations of the underlying conditional distribution, measured in total variation distance.

\begin{lemma}\label{lem:stability-TV}
	Suppose Assumption~\ref{ass:SAA-struct} holds. Fix $x_0\in\mathcal{X}$ and let $\mathbb{P}$ and $\mathbb{Q}$ be two probability measures on $\mathcal{Y}$. Define
	\[
	h_{\mathbb{P}}(z;x_0) := \E_{Y\sim\mathbb{P}}[c(z; x_0,Y)],
	\qquad
	h_{\mathbb{Q}}(z;x_0) := \E_{Y\sim\mathbb{Q}}[c(z; x_0,Y)],
	\]
	and let
	\[
	z_{\mathbb{P}}^\ast(x_0) := \arg\min_{z\in\mathcal{Z}} h_{\mathbb{P}}(z;x_0),
	\qquad
	z_{\mathbb{Q}}^\ast(x_0) := \arg\min_{z\in\mathcal{Z}} h_{\mathbb{Q}}(z;x_0).
	\]
	Then both minimizers are unique, and
	\begin{equation}\label{eq:stability-TV}
		\big\|z_{\mathbb{Q}}^\ast(x_0) - z_{\mathbb{P}}^\ast(x_0)\big\|
		\;\le\;
		\frac{2G}{\mu}\,\mathrm{TV}(\mathbb{Q},\mathbb{P}),
	\end{equation}
	where $\mathrm{TV}$ denotes the total variation distance.
	In particular, taking $\mathbb{P}=F_{Y|X}(\cdot|x_0)$ and $\mathbb{Q}=\hat{F}_{Y|X}(\cdot|x_0)$ yields
	\[
	\big\|\bar{z}_n^{\ast}(x_0) - z^\ast(x_0)\big\|
	\;\le\;
	\frac{2G}{\mu}\,\mathrm{TV}\left(\hat{F}_{Y|X}(\cdot|x_0),F_{Y|X}(\cdot|x_0)\right).
	\]
\end{lemma}

Lemma~\ref{lem:stability-TV} states that the optimal solution is Lipschitz-continuous with respect to the underlying conditional distribution in total variation distance. Intuitively, if we change the distribution of $Y|\{X=x_0\}$ only slightly (in total variation), then the expected gradient changes only slightly, and by strong convexity the minimizer cannot move far. This type of stability result is a direct consequence of the standard stability-of-argmin arguments used in M-estimation (see, e.g., \citealtEC[Chapter~5]{vanderVaart1998asymptotic}) and in stochastic programming sensitivity analysis (see \citealtEC[Chapter~5]{ShapiroDentchevaRuszczynski2009}). 
\smallskip

\proof{Proof.}
	For brevity we write $h_{\mathbb{P}}(z)$ and $h_{\mathbb{Q}}(z)$ for $h_{\mathbb{P}}(z;x_0)$ and $h_{\mathbb{Q}}(z;x_0)$, and $z_{\mathbb{P}}^\ast$ and $z_{\mathbb{Q}}^\ast$ for the corresponding minimizers.
	By Assumption~\ref{ass:SAA-struct}(c), for each $y$ the function $z\mapsto c(z; x_0,y)$ is $\mu$-strongly convex. Therefore, for any $z_1,z_2\in\mathcal{Z}$ and $\lambda\in[0,1]$,
	\[
	c(\lambda z_1 + (1-\lambda)z_2; x_0,y)
	\,\le\,
	\lambda c(z_1; x_0,y) + (1-\lambda)c(z_2; x_0,y)
	- \frac{\mu}{2}\lambda(1-\lambda)\|z_1-z_2\|^2.
	\]
	Taking expectations with respect to $Y\sim\mathbb{P}$, we obtain
	\[
	h_{\mathbb{P}}(\lambda z_1 + (1-\lambda)z_2)
	\,\le\,
	\lambda h_{\mathbb{P}}(z_1) + (1-\lambda)h_{\mathbb{P}}(z_2)
	- \frac{\mu}{2}\lambda(1-\lambda)\|z_1-z_2\|^2,
	\]
	so $h_{\mathbb{P}}$ is $\mu$-strongly convex. Similarly $h_{\mathbb{Q}}$ is $\mu$-strongly convex. Hence both have unique minimizers $z_{\mathbb{P}}^\ast$ and $z_{\mathbb{Q}}^\ast$. Because $h_\mathbb{P}$ and $h_\mathbb{Q}$ are differentiable and convex, their constrained minimizers satisfy the variational inequalities
\[
    \big\langle
        \nabla h_\mathbb{P}(z_\mathbb{P}^*),z-z_\mathbb{P}^*
    \big\rangle
    \geq 0
    \quad \text{and} \quad
    \big\langle
        \nabla h_\mathbb{Q}(z_\mathbb{Q}^*),z-z_\mathbb{Q}^*
    \big\rangle
    \geq 0,
    \qquad \text{for all } z\in\mathcal Z.
\]
Taking $z=z_\mathbb{Q}^*$ in the first inequality and $z=z_\mathbb{P}^*$ in
the second gives
\[
    \big\langle
        \nabla h_\mathbb{P}(z_\mathbb{P}^*),z_\mathbb{Q}^*-z_\mathbb{P}^*
    \big\rangle
    \geq 0 \quad \text{and} \quad
    \big\langle
        \nabla h_\mathbb{Q}(z_\mathbb{Q}^*),z_\mathbb{Q}^*-z_\mathbb{P}^*
    \big\rangle
    \leq 0.
\]

Let $d:=z_\mathbb{Q}^*-z_\mathbb{P}^*$. 
Combining with the strong monotonicity of the gradient of the expected objective $h_\mathbb{Q}$, the above constrained first-order optimality conditions, and Cauchy–Schwarz inequality, we have
\[
\begin{aligned}
    \mu\|d\|^2
    &\,\leq\,
    \big\langle
        \nabla h_\mathbb{Q}(z_\mathbb{Q}^*)-\nabla h_\mathbb{Q}(z_\mathbb{P}^*),d
    \big\rangle                                                    \,\leq\,
    -\big\langle\nabla h_\mathbb{Q}(z_\mathbb{P}^*),d\big\rangle                     \,=\,
    \big\langle
        \nabla h_\mathbb{P}(z_\mathbb{P}^*)-\nabla h_\mathbb{Q}(z_\mathbb{P}^*),d
    \big\rangle
    -
    \big\langle\nabla h_\mathbb{P}(z_\mathbb{P}^*),d\big\rangle                      \\
    &\,\leq\,
    \big\langle
        \nabla h_\mathbb{P}(z_\mathbb{P}^*)-\nabla h_\mathbb{Q}(z_\mathbb{P}^*),d
    \big\rangle                                                     \,\leq\,
    \big\|
        \nabla h_\mathbb{P}(z_\mathbb{P}^*)-\nabla h_\mathbb{Q}(z_\mathbb{P}^*)
    \big\|\|d\|.
\end{aligned}
\]
If $d=0$, the desired comparison is immediate.
Otherwise, dividing by $\mu\|d\|$ yields
\begin{equation}
    \|z_\mathbb{Q}^*-z_\mathbb{P}^*\|
    \,\leq\,
    \frac{1}{\mu}
    \left\|
        \nabla h_\mathbb{Q}(z_\mathbb{P}^*)-\nabla h_\mathbb{P}(z_\mathbb{P}^*)
    \right\|.
    \label{eq:model-decision-gradient-bound}
\end{equation}
By the definitions of $h_\mathbb{P}$ and $h_\mathbb{Q}$, we have
\[
    \nabla h_\mathbb{Q}(z_\mathbb{P}^*)-\nabla h_\mathbb{P}(z_\mathbb{P}^*)
    \,=\,
    \mathbb E_{Y\sim \mathbb{Q}}
    \left[\nabla_z c(z_\mathbb{P}^*;x_0,Y)\right] -
    \mathbb E_{Y\sim \mathbb{P}}
    \left[\nabla_z c(z_\mathbb{P}^*;x_0,Y)\right].
\]
For completeness, if a measurable vector-valued function $f$
satisfies $\sup_y\|f(y)\|\leq B$, then
\[
    \left\|
        \mathbb E_\mathbb{Q}[f(Y)]-\mathbb E_\mathbb{P}[f(Y)]
    \right\|
    \,=\,
    \sup_{\|u\|\leq 1}
    \left|
        \mathbb E_\mathbb{Q}[u^\top f(Y)]
        -
        \mathbb E_\mathbb{P}[u^\top f(Y)]
    \right| \,\leq\,
    2B\,\operatorname{TV}(\mathbb{Q},\mathbb{P}).
\]
Applying this inequality with
\(
    f(y)=\nabla_z c(z_\mathbb{P}^*;x_0,y)
\)
and using Assumption~\ref{ass:SAA-struct}(b) gives
\[
    \left\|
        \nabla h_\mathbb{Q}(z_\mathbb{P}^*)-\nabla h_\mathbb{P}(z_\mathbb{P}^*)
    \right\|
    \,\leq\,
    2G\,\operatorname{TV}(\mathbb{Q},\mathbb{P}).
\]
Combining this bound with
Equation~\eqref{eq:model-decision-gradient-bound} yields
\[
    \|z_\mathbb{Q}^*-z_\mathbb{P}^*\|
    \,\leq\,
    \frac{2G}{\mu}\operatorname{TV}(\mathbb{Q},\mathbb{P}),
\]
which proves Equation~\eqref{eq:stability-TV}. Taking
\(
    \mathbb{P}=F_{Y| X}(\cdot| x_0)\) and \(
    \mathbb{Q}=\hat F_{Y| X}(\cdot| x_0)
\)
gives the stated specialization. Therefore, we complete the proof.
\Halmos
\endproof
\smallskip

To turn Lemma~\ref{lem:stability-TV} into a quantitative bound on the average model error, we assume that the learned distribution or simulator is consistent in total variation on average over the covariate distribution.

\begin{lemma}\label{lem:model-error}
	Suppose Assumptions~\ref{ass:SAA-struct} and~\ref{ass: TV-rate} hold. Then
	\begin{equation}\label{eq:model-error-bound}
		\E\Big[\E\big[\|\bar{z}_n^{\ast}(X) - z^\ast(X)\| \,\big|\, \mathcal{D}_n \big]\Big]
		\;\le\;
		\frac{2G}{\mu}\,C_{\mathrm{M}}\,n^{-\alpha_0}(\log n)^{\alpha_1},
	\end{equation}
	where the inner expectation is with respect to $X$ and the outer expectation is with respect to $\mathcal{D}_n$.
\end{lemma}

Lemma~\ref{lem:model-error} shows that, under the stability result of Lemma~\ref{lem:stability-TV} and the average total variation convergence of Assumption~\ref{ass: TV-rate}, the decision error due purely to model misspecification (i.e., using $\hat{F}_{Y|X}(\cdot|x)$ instead of $F_{Y|X}(\cdot|x)$) decays at the same rate as the simulator itself in total variation. The constant is inflated by the inverse strong-convexity modulus $\mu^{-1}$ and the gradient envelope $G$, reflecting the sensitivity of the optimization problem to perturbations in the objective.
\smallskip

\proof{Proof.}
	By Lemma~\ref{lem:stability-TV} applied pointwise in $x$ and taking $\mathbb{P}=F_{Y|X}(\cdot|x)$ and $\mathbb{Q}=\hat{F}_{Y|X}(\cdot|x)$, we have for every $x\in\mathcal{X}$,
	\[
	\|\bar{z}_n^{\ast}(x) - z^\ast(x)\|
	\;\le\;
	\frac{2G}{\mu}\,\mathrm{TV}\big(\hat{F}_{Y|X}(\cdot|x),F_{Y|X}(\cdot|x)\big).
	\]
	Taking expectations with respect to $X$ and then with respect to $\mathcal{D}_n$ yields
	\[
	\E\Big[\E\big[\|\bar{z}_n^{\ast}(X) - z^\ast(X)\| \,\big|\, \mathcal{D}_n \big]\Big]
	\,\le\,
	\frac{2G}{\mu}\,
	\E\Big[\E\big[\mathrm{TV}(\hat{F}_{Y|X}(\cdot|X),F_{Y|X}(\cdot|X)) \,\big|\, \mathcal{D}_n \big]\Big].
	\]
	Applying Assumption~\ref{ass: TV-rate} to the right-hand side gives Equation~\eqref{eq:model-error-bound}, and we complete the proof. \Halmos
\endproof

\subsubsection{Decision Error Bound for GtO.}

We now combine the SAA error bound (Lemma~\ref{lem:SAA-error}) and the model error bound (Lemma~\ref{lem:model-error}) to obtain an overall bound on the decision error of the GtO procedure.
Recall the decomposition~\eqref{eq:error_decomposition} for a fixed $x_0$:
\[
\|\hat{z}_{n,m}^\ast(x_0) - z^\ast(x_0)\|
\,\le\,
\|\hat{z}_{n,m}^\ast(x_0) - \bar{z}_n^{\ast}(x_0)\|
+
\|\bar{z}_n^{\ast}(x_0) - z^\ast(x_0)\|.
\]
We now integrate this bound over the randomness in both the training data and the online covariate, and plug in the bounds obtained in the previous subsections.

\begin{theorem}\label{thm:decision-error}
	Suppose Assumptions~\ref{ass:SAA-struct} and~\ref{ass: TV-rate} hold.
	Then there exists a constant $C<\infty$, depending only on $(\mu,L,G,\sigma,D,d_z,C_{\mathrm{M}},\alpha_0,\alpha_1)$, such that for all integers $n,m\ge 2$,
	\begin{equation}\label{eq:decision-error-bound}
		\E\Big[\E\big[\|\hat{z}_{n,m}^\ast(X) - z^\ast(X)\| \,\big|\, \mathcal{D}_n \big]\Big]
		\;\le\;
		C\left(
		n^{-\alpha_0}(\log n)^{\alpha_1}
		+
		m^{-1/2}
		\right),
	\end{equation}
	where, conditional on $\mathcal D_n$, the inner expectation is with respect to the online covariate $X$ and the conditional simulation draws, and the outer expectation is with respect to $\mathcal{D}_n$.
\end{theorem}

Theorem~\ref{thm:decision-error} provides a clean decomposition of the average decision error of the GtO procedure into: an \emph{offline learning error} term, $n^{-\alpha_0}(\log n)^{\alpha_1}$, which reflects how accurately the conditional distribution $F_{Y|X}(\cdot|x)$ can be learned from $n$ samples, and an \emph{online SAA error} term, $m^{-1/2}$, which reflects the Monte Carlo error from using $m$ simulated samples at decision time.
In settings where $\alpha_0=1/2$, this yields the canonical rate
\[
\E\Big[\E\big[\|\hat{z}_{n,m}^\ast(X) - z^\ast(X)\| \,\big|\, \mathcal{D}_n \big]\Big]
\,=\, \mathcal{O}\Big(n^{-1/2}(\log n)^{\alpha_1} + m^{-1/2} \Big) \,=\, \widetilde{\mathcal{O}}\Big(n^{-1/2} + m^{-1/2}\Big).
\]
Theorem~\ref{thm:decision-error} makes explicit the trade-off between investing in more training data (larger $n$) and using more simulation at decision time (larger $m$). In particular, if one balances the two error terms by choosing $m$ as a suitable function of $n$, the overall decision error can be minimized.
\smallskip

\proof{Proof.}
	Taking expectations with respect to $X$ and then with respect to $\mathcal{D}_n$ to the result in Lemma~\ref{lem:SAA-error} yields
	\begin{equation}\label{eq:decision-SAA-part}
		\E\Big[\E\big[\|\hat{z}_{n,m}^\ast(X) - \bar{z}_n^{\ast}(X)\| \,\big|\, \mathcal{D}_n \big]\Big]
		\;\le\;
		C_{\mathrm{SAA}}\,m^{-1/2}.
	\end{equation}
	By Lemma~\ref{lem:model-error}, we have
	\begin{align*}
		\E\Big[\E\big[\|\hat{z}_{n,m}^\ast(X) - z^\ast(X)\| \,\big|\, \mathcal{D}_n \big]\Big]
		&\,\le\,
		\E\Big[\E\big[\|\hat{z}_{n,m}^\ast(X) - \bar{z}_n^{\ast}(X)\| \,\big|\, \mathcal{D}_n \big]\Big]
		+
		\E\Big[\E\big[\|\bar{z}_n^{\ast}(X) - z^\ast(X)\| \,\big|\, \mathcal{D}_n \big]\Big] \\
		&\,\le\,
		C_{\mathrm{SAA}}\,m^{-1/2}
		+
		\frac{2G}{\mu}\,C_{\mathrm{M}}\,n^{-\alpha_0}(\log n)^{\alpha_1}.
	\end{align*}
	By setting
	$
	C:= C_{\mathrm{SAA}} + \frac{2G}{\mu}C_{\mathrm{M}},
	$
	 we can rewrite the above bound and obtain Equation~\eqref{eq:decision-error-bound}. Therefore, we complete the proof. \Halmos
\endproof

\section{Numerical Study: Implementation and Supplementary Results}
\label{sec:EC Numerical Experiments}

This section provides implementation details and supplementary results for Section~\ref{sec:numerical experiments}. 

\subsection{Contextual Newsvendor Problem}
\label{app:num-newsvendor}

\subsubsection{Demand and Shortage-Cost Construction.}

Conditional demand follows a three-component Gaussian mixture representing low, regular, and high demand states. We divide the covariates into three groups according to their roles in the demand model. Seasonal factors are generated from a uniform distribution and transformed smoothly to determine the component means. Market and operational factors, such as traffic, discounts, and promotions, follow a lognormal distribution and determine the mixture probabilities. Risk-related factors, such as weather and recent demand history, follow a zero-mean multivariate normal distribution and determine the component variances. Given $X=x_0$, demand is drawn from the resulting mixture and censored at zero. Across covariate dimensions, the parameters are calibrated to keep average demand and the signal-to-noise ratio comparable.

The shortage penalty reflects differences in service priorities and operating conditions. Let $x_{\mathrm{risk}}$ denote the risk-related covariates and define the priority score $s(x)=\gamma^\top x_{\mathrm{risk}}$, with $\gamma$ normalized so that $s(X)$ is standard normal. We partition the score into 11 equal-probability intervals using $\tau_k=\Phi^{-1}(k/11)$ for $k=1,\ldots,10$, where $\Phi$ is the standard normal cumulative distribution function. Assigning increasing penalty levels from 30 to 120 gives
\[
b(x)=\begin{cases}
b_1, & s(x)<\tau_1,\\
b_k, & \tau_{k-1}\leq s(x)<\tau_k,\quad k=2,\ldots,10,\\
b_{11}, & s(x)\geq\tau_{10}.
\end{cases}
\]
The penalty range produces different newsvendor critical fractiles. Because $b(x)$ is a step function, the conditional objective can be discontinuous in $x$ even though the conditional demand model varies smoothly.

We consider $n\in\{800,2000,5000,10000,20000,30000\}$ and $d_x\in\{3,6,10,20,40,60,100\}$. Each setting is evaluated on 100 test contexts over 30 independent replications. Direct PtO estimates the conditional objective, including the context-dependent penalty. Structure-aware PtO instead estimates $h_+(z;x)=\mathbb{E}[(z-Y)^+| X=x]$ and $h_-(z;x)=\mathbb{E}[(Y-z)^+| X=x]$ using the same prediction weights, and combines them as $c_h\hat h_{+,n}(z;x_0)+b(x_0)\hat h_{-,n}(z;x_0)$. This formulation incorporates the known penalty directly rather than estimating its discontinuities.

The empirical objective is piecewise linear in the one-dimensional order quantity. We optimize it by evaluating demand breakpoints. This approach also accommodates signed weights when the objective is bounded below. On the unrestricted domain $z\geq0$, the terminal slope must be nonnegative, and the feasible endpoint $z=0$ must be included among the candidates. Thus, efficient optimization in this setting does not require convexity.

\subsubsection{Detailed Performance and Online Time.}
\label{app:num-newsvendor-results}

Tables~\ref{tab:newsvendor-relative-cost-direct} and~\ref{tab:newsvendor-relative-cost} report the method-level relative costs for the direct and structure-aware formulations, respectively, corresponding to Panels (a) and~(b) of Figure~\ref{fig:newsvendor-regime-map}. Higher relative cost indicates better decision performance. The GtO results are identical across the two tables; the PtO columns report the respective formulations.

\begin{table}[htbp]
\TABLE
{Relative cost for the newsvendor problem with the direct formulation.
\label{tab:newsvendor-relative-cost-direct}}
{
\small
\begin{tabular}{cccccccc}
\toprule
$n$ & $d_x$ & SAA & PtO-$k$NN & PtO-KS & PtO-KRR
& PtO-KRR$^+$ & GtO-CDM \\
\midrule

\multirow{7}{*}{800}
& 3   & 0.0000 & 0.9164 & 0.8912 & \textbf{0.9428} & 0.9331 & 0.7990 \\
& 6   & 0.0000 & 0.7488 & 0.7006 & \textbf{0.8199} & 0.7847 & 0.7405 \\
& 10  & 0.0000 & 0.5462 & 0.5121 & \textbf{0.7477} & 0.5614 & 0.6948 \\
& 20  & 0.0000 & 0.3849 & 0.2462 & 0.6605 & 0.3887 & \textbf{0.6850} \\
& 40  & 0.0000 & 0.3578 & 0.0451 & 0.5754 & 0.3185 & \textbf{0.7452} \\
& 60  & 0.0000 & 0.1736 & 0.0321 & 0.3725 & 0.1759 & \textbf{0.5947} \\
& 100 & 0.0000 & 0.1473 & 0.0161 & 0.2473 & 0.1110 & \textbf{0.6070} \\
\midrule

\multirow{7}{*}{2000}
& 3   & 0.0000 & 0.9519 & 0.9350 & \textbf{0.9647} & 0.9583 & 0.8378 \\
& 6   & 0.0000 & 0.8157 & 0.7792 & \textbf{0.8609} & 0.8327 & 0.7714 \\
& 10  & 0.0000 & 0.6142 & 0.5520 & \textbf{0.7933} & 0.5948 & 0.7672 \\
& 20  & 0.0000 & 0.4641 & 0.3679 & \textbf{0.7148} & 0.4570 & 0.7074 \\
& 40  & 0.0000 & 0.4013 & 0.0476 & 0.6224 & 0.3836 & \textbf{0.7688} \\
& 60  & 0.0000 & 0.2353 & 0.0324 & 0.6418 & 0.1969 & \textbf{0.7202} \\
& 100 & 0.0000 & 0.1881 & 0.0159 & 0.4262 & 0.1071 & \textbf{0.7690} \\
\midrule

\multirow{7}{*}{5000}
& 3   & 0.0000 & 0.9709 & 0.9550 & 0.9759 & 0.9576 & \textbf{0.9815} \\
& 6   & 0.0000 & 0.8583 & 0.8376 & \textbf{0.9014} & 0.8778 & 0.8848 \\
& 10  & 0.0000 & 0.7040 & 0.5672 & \textbf{0.8488} & 0.7153 & 0.8128 \\
& 20  & 0.0000 & 0.5140 & 0.4227 & 0.7670 & 0.5101 & \textbf{0.7723} \\
& 40  & 0.0000 & 0.4722 & 0.1498 & 0.6402 & 0.4098 & \textbf{0.7935} \\
& 60  & 0.0000 & 0.2941 & 0.0315 & 0.6677 & 0.1978 & \textbf{0.7377} \\
& 100 & 0.0000 & 0.1915 & 0.0161 & 0.5654 & 0.1164 & \textbf{0.8012} \\
\midrule

\multirow{7}{*}{10000}
& 3   & 0.0000 & 0.9800 & 0.9706 & 0.9826 & 0.9768 & \textbf{0.9916} \\
& 6   & 0.0000 & 0.8794 & 0.8502 & 0.9117 & 0.8951 & \textbf{0.9213} \\
& 10  & 0.0000 & 0.7576 & 0.6897 & \textbf{0.8725} & 0.7584 & 0.8152 \\
& 20  & 0.0000 & 0.5548 & 0.4556 & 0.7931 & 0.5522 & \textbf{0.8191} \\
& 40  & 0.0000 & 0.4969 & 0.2449 & \textbf{0.8567} & 0.4427 & 0.8250 \\
& 60  & 0.0000 & 0.3178 & 0.0318 & 0.6878 & 0.2062 & \textbf{0.7550} \\
& 100 & 0.0000 & 0.2336 & 0.0159 & 0.6727 & 0.1209 & \textbf{0.7935} \\
\midrule

\multirow{7}{*}{20000}
& 3   & 0.0000 & 0.9853 & 0.9776 & 0.9343 & 0.9717 & \textbf{0.9934} \\
& 6   & 0.0000 & 0.8979 & 0.8754 & 0.7992 & 0.8891 & \textbf{0.9870} \\
& 10  & 0.0000 & 0.7534 & 0.7415 & 0.8040 & 0.7134 & \textbf{0.9122} \\
& 20  & 0.0000 & 0.6008 & 0.4711 & 0.8021 & 0.4864 & \textbf{0.8645} \\
& 40  & 0.0000 & 0.5450 & 0.3231 & 0.8389 & 0.4394 & \textbf{0.8859} \\
& 60  & 0.0000 & 0.3608 & 0.0316 & 0.6573 & 0.2587 & \textbf{0.7953} \\
& 100 & 0.0000 & 0.2553 & 0.0158 & 0.6004 & 0.1774 & \textbf{0.8447} \\
\midrule

\multirow{7}{*}{30000}
& 3   & 0.0000 & 0.9856 & 0.9803 & 0.9899 & 0.9782 & \textbf{0.9943} \\
& 6   & 0.0000 & 0.9127 & 0.8906 & 0.9152 & 0.8994 & \textbf{0.9911} \\
& 10  & 0.0000 & 0.8104 & 0.7713 & 0.8993 & 0.7105 & \textbf{0.9805} \\
& 20  & 0.0000 & 0.6185 & 0.4795 & 0.8435 & 0.4759 & \textbf{0.9619} \\
& 40  & 0.0000 & 0.5749 & 0.3705 & 0.8428 & 0.4340 & \textbf{0.9250} \\
& 60  & 0.0000 & 0.3631 & 0.0960 & 0.6566 & 0.2608 & \textbf{0.8590} \\
& 100 & 0.0000 & 0.2715 & 0.0161 & 0.6564 & 0.1741 & \textbf{0.8665} \\
\bottomrule
\end{tabular}
}
{\textit{Note.} Boldface indicates the method with the highest relative cost for each \((n,d_x)\) setting.}
\end{table}

\begin{table}[htbp]
\TABLE
{Relative cost for the newsvendor problem with the structure-aware formulation.
\label{tab:newsvendor-relative-cost}}
{
\small
\begin{tabular}{cccccccc}
\toprule
$n$ & $d_x$ & SAA & PtO-$k$NN & PtO-KS & PtO-KRR
& PtO-KRR$^+$ & GtO-CDM  \\
\midrule

\multirow{7}{*}{800}
& 3   & 0.0000 & 0.8487 & 0.8682 & \textbf{0.9146} & 0.8923 & 0.7990 \\
& 6   & 0.0000 & 0.7387 & 0.6762 & \textbf{0.8532} & 0.7643 & 0.7405 \\
& 10  & 0.0000 & 0.4599 & 0.4302 & 0.6594 & 0.4922 & \textbf{0.6948} \\
& 20  & 0.0000 & 0.4764 & 0.3474 & \textbf{0.7485} & 0.4409 & 0.6850 \\
& 40  & 0.0000 & 0.2879 & 0.0414 & 0.7113 & 0.2890 & \textbf{0.7452} \\
& 60  & 0.0000 & 0.3113 & 0.1251 & 0.5683 & 0.3208 & \textbf{0.5947} \\
& 100 & 0.0000 & 0.2091 & 0.0416 & 0.5503 & 0.1593 & \textbf{0.6070} \\
\midrule

\multirow{7}{*}{2000}
& 3   & 0.0000 & 0.8486 & 0.9107 & \textbf{0.9429} & 0.9262 & 0.8378 \\
& 6   & 0.0000 & 0.7784 & 0.7464 & \textbf{0.8753} & 0.8124 & 0.7714 \\
& 10  & 0.0000 & 0.5984 & 0.4822 & 0.7670 & 0.5954 & \textbf{0.7672} \\
& 20  & 0.0000 & 0.5242 & 0.4694 & \textbf{0.7902} & 0.5185 & 0.7074 \\
& 40  & 0.0000 & 0.3896 & 0.0410 & \textbf{0.7698} & 0.3431 & 0.7688 \\
& 60  & 0.0000 & 0.3443 & 0.1255 & 0.6429 & 0.3392 & \textbf{0.7202} \\
& 100 & 0.0000 & 0.2263 & 0.0407 & 0.6279 & 0.1670 & \textbf{0.7690} \\
\midrule

\multirow{7}{*}{5000}
& 3   & 0.0000 & 0.9449 & 0.9335 & 0.9616 & 0.9396 & \textbf{0.9815} \\
& 6   & 0.0000 & 0.8623 & 0.7966 & \textbf{0.8913} & 0.8612 & 0.8848 \\
& 10  & 0.0000 & 0.6538 & 0.6026 & \textbf{0.8196} & 0.6712 & 0.8128 \\
& 20  & 0.0000 & 0.5703 & 0.5049 & \textbf{0.8404} & 0.5982 & 0.7723 \\
& 40  & 0.0000 & 0.4445 & 0.1409 & \textbf{0.8048} & 0.3946 & 0.7935 \\
& 60  & 0.0000 & 0.4009 & 0.1300 & 0.6740 & 0.3760 & \textbf{0.7377} \\
& 100 & 0.0000 & 0.2104 & 0.0413 & 0.6608 & 0.1776 & \textbf{0.8012} \\
\midrule

\multirow{7}{*}{10000}
& 3   & 0.0000 & 0.9600 & 0.9575 & 0.9740 & 0.9666 & \textbf{0.9916} \\
& 6   & 0.0000 & 0.8892 & 0.8248 & 0.9211 & 0.8812 & \textbf{0.9213} \\
& 10  & 0.0000 & 0.6933 & 0.6273 & \textbf{0.8227} & 0.6804 & 0.8152 \\
& 20  & 0.0000 & 0.6128 & 0.5362 & \textbf{0.8457} & 0.6146 & 0.8191 \\
& 40  & 0.0000 & 0.4889 & 0.2519 & 0.8226 & 0.4128 & \textbf{0.8250} \\
& 60  & 0.0000 & 0.4249 & 0.1285 & 0.7012 & 0.3811 & \textbf{0.7550} \\
& 100 & 0.0000 & 0.2621 & 0.0406 & 0.6846 & 0.1880 & \textbf{0.7935} \\
\midrule

\multirow{7}{*}{20000}
& 3   & 0.0000 & 0.9617 & 0.9571 & 0.9730 & 0.9653 & \textbf{0.9934} \\
& 6   & 0.0000 & 0.9075 & 0.8282 & 0.9122 & 0.8868 & \textbf{0.9870} \\
& 10  & 0.0000 & 0.7177 & 0.6887 & 0.8156 & 0.6577 & \textbf{0.9122} \\
& 20  & 0.0000 & 0.6688 & 0.5594 & \textbf{0.8731} & 0.5333 & 0.8645 \\
& 40  & 0.0000 & 0.5194 & 0.2366 & 0.8194 & 0.3926 & \textbf{0.8859} \\
& 60  & 0.0000 & 0.4593 & 0.1279 & 0.7693 & 0.3761 & \textbf{0.7953} \\
& 100 & 0.0000 & 0.2992 & 0.0408 & 0.7065 & 0.2194 & \textbf{0.8447} \\
\midrule

\multirow{7}{*}{30000}
& 3   & 0.0000 & 0.9754 & 0.9681 & 0.9804 & 0.9723 & \textbf{0.9943} \\
& 6   & 0.0000 & 0.9177 & 0.8700 & 0.9065 & 0.8299 & \textbf{0.9911} \\
& 10  & 0.0000 & 0.7911 & 0.7149 & 0.8062 & 0.5109 & \textbf{0.9805} \\
& 20  & 0.0000 & 0.6579 & 0.5797 & 0.8004 & 0.4252 & \textbf{0.9619} \\
& 40  & 0.0000 & 0.5419 & 0.3121 & 0.7442 & 0.2221 & \textbf{0.9250} \\
& 60  & 0.0000 & 0.4958 & 0.1269 & 0.7814 & 0.2320 & \textbf{0.8590} \\
& 100 & 0.0000 & 0.3338 & 0.0401 & 0.7060 & 0.2195 & \textbf{0.8665} \\
\bottomrule
\end{tabular}
}
{\textit{Note.} Boldface indicates the method with the highest relative cost for each \((n,d_x)\) setting.}
\end{table}

Tables~\ref{tab:runtime-sec-direct} and~\ref{tab:runtime-sec-structure-aware} report online times for the two formulations. PtO generally requires hundredths of a second, whereas GtO takes approximately 0.49 to 6.57 seconds across the tested settings. At $n=10{,}000$, structure-aware PtO takes approximately 0.013 to 0.042 seconds and GtO takes 0.76 to 2.80 seconds. These are the entries summarized in Table~\ref{tab:newsvendor-online-summary} in the main text.

\begin{table}[htbp]
\TABLE
{Running time for the newsvendor problem with the direct formulation.
\label{tab:runtime-sec-direct}}
{
\small
\begin{tabular}{cccccccc}
\toprule
$n$ & $d_x$ & SAA & PtO-KNN & PtO-KS & PtO-KRR & PtO-KRR$^+$ & GtO-CDM \\
\midrule

\multirow{7}{*}{800}
& 3   & 0.0000 & 0.0248 & 0.0231 & 0.0276 & 0.0257 & 0.4937 \\
& 6   & 0.0000 & 0.0285 & 0.0242 & 0.0337 & 0.0294 & 0.6161 \\
& 10  & 0.0000 & 0.0143 & 0.0109 & 0.0178 & 0.0131 & 0.7875 \\
& 20  & 0.0000 & 0.0176 & 0.0145 & 0.0298 & 0.0224 & 0.9506 \\
& 40  & 0.0000 & 0.0114 & 0.0107 & 0.0261 & 0.0068 & 1.4974 \\
& 60  & 0.0000 & 0.0098 & 0.0057 & 0.0256 & 0.0106 & 1.5247 \\
& 100 & 0.0000 & 0.0055 & 0.0018 & 0.0169 & 0.0064 & 1.5552 \\
\midrule

\multirow{7}{*}{2000}
& 3   & 0.0067 & 0.0088 & 0.0092 & 0.0096 & 0.0099 & 0.4896 \\
& 6   & 0.0112 & 0.0156 & 0.0147 & 0.0151 & 0.0159 & 0.6133 \\
& 10  & 0.0071 & 0.0093 & 0.0095 & 0.0098 & 0.0101 & 0.7868 \\
& 20  & 0.0157 & 0.0142 & 0.0146 & 0.0155 & 0.0163 & 0.9505 \\
& 40  & 0.0069 & 0.0090 & 0.0094 & 0.0097 & 0.0100 & 1.4969 \\
& 60  & 0.0106 & 0.0137 & 0.0141 & 0.0146 & 0.0150 & 1.5221 \\
& 100 & 0.0068 & 0.0089 & 0.0091 & 0.0096 & 0.0099 & 1.5556 \\
\midrule

\multirow{7}{*}{5000}
& 3   & 0.0132 & 0.0174 & 0.0171 & 0.0195 & 0.0188 & 0.4875 \\
& 6   & 0.0211 & 0.0286 & 0.0278 & 0.0294 & 0.0291 & 0.6135 \\
& 10  & 0.0129 & 0.0178 & 0.0174 & 0.0186 & 0.0182 & 0.7869 \\
& 20  & 0.0219 & 0.0291 & 0.0283 & 0.0297 & 0.0302 & 0.9508 \\
& 40  & 0.0127 & 0.0172 & 0.0168 & 0.0184 & 0.0181 & 1.4963 \\
& 60  & 0.0208 & 0.0275 & 0.0269 & 0.0287 & 0.0282 & 1.5193 \\
& 100 & 0.0131 & 0.0177 & 0.0173 & 0.0185 & 0.0180 & 1.5508 \\
\midrule

\multirow{7}{*}{10000}
& 3   & 0.0214 & 0.0287 & 0.0276 & 0.0315 & 0.0308 & 0.7556 \\
& 6   & 0.0136 & 0.0184 & 0.0177 & 0.0196 & 0.0192 & 1.0644 \\
& 10  & 0.0221 & 0.0289 & 0.0282 & 0.0305 & 0.0308 & 1.3747 \\
& 20  & 0.0139 & 0.0181 & 0.0174 & 0.0191 & 0.0190 & 1.6678 \\
& 40  & 0.0224 & 0.0292 & 0.0288 & 0.0312 & 0.0310 & 2.6496 \\
& 60  & 0.0134 & 0.0180 & 0.0171 & 0.0203 & 0.0197 & 2.7344 \\
& 100 & 0.0216 & 0.0281 & 0.0279 & 0.0301 & 0.0310 & 2.7950 \\
\midrule

\multirow{7}{*}{20000}
& 3   & 0.0146 & 0.0191 & 0.0188 & 0.0217 & 0.0210 & 1.3742 \\
& 6   & 0.0235 & 0.0301 & 0.0305 & 0.0332 & 0.0318 & 1.9059 \\
& 10  & 0.0143 & 0.0190 & 0.0184 & 0.0201 & 0.0207 & 2.5172 \\
& 20  & 0.0231 & 0.0306 & 0.0297 & 0.0314 & 0.0328 & 3.0537 \\
& 40  & 0.0145 & 0.0192 & 0.0186 & 0.0200 & 0.0197 & 4.9494 \\
& 60  & 0.0234 & 0.0298 & 0.0295 & 0.0311 & 0.0314 & 5.1321 \\
& 100 & 0.0141 & 0.0190 & 0.0185 & 0.0198 & 0.0196 & 5.2524 \\
\midrule

\multirow{7}{*}{30000}
& 3   & 0.0247 & 0.0312 & 0.0317 & 0.0351 & 0.0354 & 1.6722 \\
& 6   & 0.0152 & 0.0201 & 0.0196 & 0.0214 & 0.0208 & 2.3431 \\
& 10  & 0.0241 & 0.0320 & 0.0305 & 0.0332 & 0.0328 & 3.1093 \\
& 20  & 0.0156 & 0.0198 & 0.0193 & 0.0209 & 0.0211 & 3.7524 \\
& 40  & 0.0243 & 0.0310 & 0.0314 & 0.0325 & 0.0328 & 6.1713 \\
& 60  & 0.0151 & 0.0202 & 0.0192 & 0.0206 & 0.0207 & 6.3927 \\
& 100 & 0.0250 & 0.0313 & 0.0316 & 0.0324 & 0.0329 & 6.5674 \\
\bottomrule
\end{tabular}
}
{\textit{Note.} Running time is reported in seconds and averaged over 30 independent replications.}
\end{table}

\begin{table}[htbp]
\TABLE
{Running time for the newsvendor problem with the structure-aware formulation.
\label{tab:runtime-sec-structure-aware}
\label{tab:runtime-sec}}
{
\small
\begin{tabular}{cccccccc}
\toprule
$n$ & $d_x$ & SAA & PtO-KNN & PtO-KS & PtO-KRR & PtO-KRR$^+$ & GtO-CDM \\
\midrule

\multirow{7}{*}{800}
& 3   & 0.0000 & 0.0294 & 0.0257 & 0.0254 & 0.0309 & 0.4937 \\
& 6   & 0.0000 & 0.0429 & 0.0239 & 0.0303 & 0.0354 & 0.6161 \\
& 10  & 0.0000 & 0.0149 & 0.0140 & 0.0126 & 0.0117 & 0.7875 \\
& 20  & 0.0000 & 0.0125 & 0.0220 & 0.0409 & 0.0149 & 0.9506 \\
& 40  & 0.0000 & 0.0100 & 0.0148 & 0.0388 & 0.0054 & 1.4974 \\
& 60  & 0.0000 & 0.0123 & 0.0084 & 0.0200 & 0.0122 & 1.5247 \\
& 100 & 0.0000 & 0.0076 & 0.0019 & 0.0137 & 0.0049 & 1.5552 \\
\midrule

\multirow{7}{*}{2000}
& 3   & 0.0067 & 0.0057 & 0.0113 & 0.0096 & 0.0123 & 0.4896 \\
& 6   & 0.0112 & 0.0230 & 0.0115 & 0.0102 & 0.0113 & 0.6133 \\
& 10  & 0.0071 & 0.0063 & 0.0139 & 0.0111 & 0.0082 & 0.7868 \\
& 20  & 0.0157 & 0.0178 & 0.0106 & 0.0202 & 0.0209 & 0.9505 \\
& 40  & 0.0069 & 0.0113 & 0.0088 & 0.0132 & 0.0082 & 1.4969 \\
& 60  & 0.0106 & 0.0112 & 0.0121 & 0.0176 & 0.0120 & 1.5221 \\
& 100 & 0.0068 & 0.0081 & 0.0079 & 0.0113 & 0.0141 & 1.5556 \\
\midrule

\multirow{7}{*}{5000}
& 3   & 0.0132 & 0.0131 & 0.0154 & 0.0243 & 0.0162 & 0.4875 \\
& 6   & 0.0211 & 0.0377 & 0.0193 & 0.0222 & 0.0183 & 0.6135 \\
& 10  & 0.0129 & 0.0227 & 0.0109 & 0.0248 & 0.0199 & 0.7869 \\
& 20  & 0.0219 & 0.0431 & 0.0298 & 0.0253 & 0.0339 & 0.9508 \\
& 40  & 0.0127 & 0.0110 & 0.0140 & 0.0221 & 0.0123 & 1.4963 \\
& 60  & 0.0208 & 0.0327 & 0.0324 & 0.0193 & 0.0352 & 1.5193 \\
& 100 & 0.0131 & 0.0130 & 0.0230 & 0.0127 & 0.0273 & 1.5508 \\
\midrule

\multirow{7}{*}{10000}
& 3   & 0.0214 & 0.0392 & 0.0331 & 0.0338 & 0.0251 & 0.7556 \\
& 6   & 0.0136 & 0.0268 & 0.0208 & 0.0294 & 0.0199 & 1.0644 \\
& 10  & 0.0221 & 0.0350 & 0.0373 & 0.0405 & 0.0285 & 1.3747 \\
& 20  & 0.0139 & 0.0218 & 0.0206 & 0.0146 & 0.0150 & 1.6678 \\
& 40  & 0.0224 & 0.0250 & 0.0420 & 0.0207 & 0.0202 & 2.6496 \\
& 60  & 0.0134 & 0.0211 & 0.0235 & 0.0214 & 0.0125 & 2.7344 \\
& 100 & 0.0216 & 0.0301 & 0.0298 & 0.0210 & 0.0314 & 2.7950 \\
\midrule

\multirow{7}{*}{20000}
& 3   & 0.0146 & 0.0168 & 0.0285 & 0.0263 & 0.0142 & 1.3742 \\
& 6   & 0.0235 & 0.0440 & 0.0318 & 0.0466 & 0.0251 & 1.9059 \\
& 10  & 0.0143 & 0.0223 & 0.0147 & 0.0174 & 0.0147 & 2.5172 \\
& 20  & 0.0231 & 0.0239 & 0.0364 & 0.0378 & 0.0475 & 3.0537 \\
& 40  & 0.0145 & 0.0288 & 0.0237 & 0.0167 & 0.0162 & 4.9494 \\
& 60  & 0.0234 & 0.0371 & 0.0352 & 0.0291 & 0.0200 & 5.1321 \\
& 100 & 0.0141 & 0.0119 & 0.0134 & 0.0148 & 0.0267 & 5.2524 \\
\midrule

\multirow{7}{*}{30000}
& 3   & 0.0247 & 0.0228 & 0.0357 & 0.0292 & 0.0408 & 1.6722 \\
& 6   & 0.0152 & 0.0220 & 0.0283 & 0.0279 & 0.0310 & 2.3431 \\
& 10  & 0.0241 & 0.0327 & 0.0336 & 0.0362 & 0.0373 & 3.1093 \\
& 20  & 0.0156 & 0.0245 & 0.0186 & 0.0130 & 0.0151 & 3.7524 \\
& 40  & 0.0243 & 0.0388 & 0.0252 & 0.0339 & 0.0459 & 6.1713 \\
& 60  & 0.0151 & 0.0251 & 0.0172 & 0.0247 & 0.0248 & 6.3927 \\
& 100 & 0.0250 & 0.0301 & 0.0427 & 0.0408 & 0.0297 & 6.5674 \\
\bottomrule
\end{tabular}
}
{\textit{Note.} Running time is reported in seconds and averaged over 30 independent replications.}
\end{table}

\clearpage

\subsection{Product Placement Problem}
\label{app:num-placement}

\subsubsection{Network and Demand Construction.}

A retailer operates $d_y$ stores connected by a directed transportation network $G=(V,\mathcal A)$, where $|V|=d_y$. After observing $X=x_0$, the retailer chooses initial inventory $z\in\mathbb R_+^{d_y}$ at cost $\alpha^\top z$. Once demand $Y\in\mathbb R_+^{d_y}$ is realized, inventory can be transferred between stores. Let $f\in\mathbb R_+^{|\mathcal A|}$ denote transfers and $p\in\mathbb R_+^{d_y}$ unmet demand. The incidence matrix $A\in\mathbb R^{d_y\times|\mathcal A|}$ has entries 1 at an arc's origin and $-1$ at its destination, so $Af$ is net outflow and $z-Af$ is inventory after~transfers.

Given $z$ and demand $y$, the recourse problem minimizes transportation and shortage costs:
\[
Q(z,y)=\min_{f,p}\left\{g^\top f+r^\top p:\ z-Af+p\geq y,\ f\geq0,\ p\geq0\right\},
\]
where $g\in\mathbb R_+^{|\mathcal A|}$ and $r\in\mathbb R_+^{d_y}$ are unit transportation and shortage costs. The initial placement decision minimizes $\alpha^\top z+\mathbb E[Q(z,Y)| X=x_0]$ over $0\leq z\leq U$, where $U\in\mathbb R_+^{d_y}$ specifies store-level inventory upper~bounds.

We extend the newsvendor's three-component Gaussian mixture to multivariate demand. Each store has the same marginal model specification, with dependence introduced so that stores closer in the network have stronger demand correlations. Demand is censored at zero coordinatewise. The network contains $2d_y$ directed arcs: $d_y$ form a ring, and the remaining $d_y$ are generated randomly once per replication. We fix $d_x=10$ and $n=10^4$ and vary $d_y\in\{1,2,5,10,15\}$. Average demand and marginal variability at each store are calibrated to remain comparable across dimensions. Unit placement, transportation, and shortage costs are 2, 1, and 10, respectively.

\subsubsection{Optimization with Nonnegative and Signed Weights.}
\label{app:num-placement-solver}

The weighted problem minimizes $\alpha^\top z+\sum_{s=1}^S\omega_sQ(z,y^s)$ over $0\leq z\leq U$, where $y^s$ are observed or generated demand scenarios and $\omega_s$ their weights. With nonnegative weights, the scenario recourse problems can be incorporated into the outer minimization, yielding the linear program
\[
\begin{aligned}
\min_{z,\{f^s,p^s\}_{s=1}^S}\quad
& \alpha^\top z
+\sum_{s=1}^S\omega_s\left(g^\top f^s+r^\top p^s\right)\\
\mathrm{s.t.}\qquad
& z-Af^s+p^s\ge y^s,\quad s=1,\ldots,S,\\
& 0\le z\le U,\\
& f^s\ge0,\quad p^s\ge0,\quad s=1,\ldots,S.
\end{aligned}
\tag{PP-LP}
\]
This formulation applies to SAA, nonnegative-weight PtO methods, and GtO, which assigns equal weights to generated scenarios.

Signed KRR weights require a different treatment. The recourse problem has the dual representation $Q(z,y)=\max_{\lambda\in\Lambda}(y-z)^\top\lambda$, where $\Lambda=\{\lambda:-A^\top\lambda\leq g,\ 0\leq\lambda\leq r\}$. For $\omega_s<0$, this implies $\omega_sQ(z,y^s)=\min_{\lambda^s\in\Lambda}\omega_s(y^s-z)^\top\lambda^s$. Define $\mathcal S_+=\{s:\omega_s\geq0\}$ and $\mathcal S_-=\{s:\omega_s<0\}$. Using the primal representation for nonnegative weights and the dual representation for negative weights gives
\begin{equation}
{\renewcommand{\arraystretch}{1.5}
\begin{array}{r@{\quad}l@{\quad}l}
\displaystyle
\min_{\substack{
z,\{f^s,p^s\}_{s\in\mathcal S_+},\\
\{\lambda^s\}_{s\in\mathcal S_-}
}}
&
\multicolumn{2}{l}{
\displaystyle
\alpha^\top z
+\sum_{s\in\mathcal S_+}
\omega_s\left(g^\top f^s+r^\top p^s\right)
+\sum_{s\in\mathcal S_-}
\omega_s(y^s-z)^\top\lambda^s
}
\\
\mathrm{s.t.}
& z-Af^s+p^s\ge y^s,
& s\in\mathcal S_+,\\
& f^s\ge0,\quad p^s\ge0,
& s\in\mathcal S_+,\\
& -A^\top\lambda^s\le g,
& s\in\mathcal S_-,\\
& 0\le\lambda^s\le r,
& s\in\mathcal S_-,\\
& 0\le z\le U.
&
\end{array}}
\tag{PP-KRR}
\end{equation}
When $\mathcal S_-$ is empty, (PP-KRR) reduces to (PP-LP). Otherwise, the products $-\omega_s z^\top\lambda^s$ introduce bilinear terms, producing a generally nonconvex formulation. Gurobi's global optimization procedure can address this formulation, but may become computationally expensive as the number of scenarios and decision dimension~increase.

We instead exploit the difference-of-convex decomposition
\[
\alpha^\top z+\sum_{s=1}^S\omega_sQ(z,y^s)
\,=\, \underbrace{\alpha^\top z+\sum_{s\in\mathcal S_+}\omega_sQ(z,y^s)}_{C(z)}
-\underbrace{\sum_{s\in\mathcal S_-}|\omega_s|Q(z,y^s)}_{H(z)}.
\]
Both $C$ and $H$ are convex because $Q(\cdot,y^s)$ is convex. At each DCA iteration, we replace $-H$ by its affine approximation using a subgradient of $H$ at the current iterate, then solve the resulting convex subproblem with Gurobi. DCA does not generally certify global optimality, so the reported decision quality and runtime reflect this implementation. Table~\ref{tab:product-placement-results} in the main paper reports the results for all five demand dimensions.

\subsection{Real Case Study: KRR with Categorical Covariates}
\label{app:krr-categorical}
In the real-case experiment, categorical variables are represented using one-hot encoding, resulting in binary indicators. For KRR, we use the Gaussian kernel \[ K(x,x')=\exp(-\frac{\|x-x'\|^2}{\ell}),\] where $\ell>0$ is the kernel bandwidth parameter. Since the Gaussian kernel only depends on the squared Euclidean distance between covariate vectors, binary indicators can be treated in the same way as continuous covariates. Specifically, for each binary indicator $x_j\in\{0,1\}$, its multiplicative contribution to the kernel is
\[
\exp\Bigg(-\frac{\big(x_j-x_j'\big)^2}{\ell}\Bigg) \,=\,
\begin{cases}
1, & x_j=x_j',\\
\exp(-1/\ell), & x_j\neq x_j'.
\end{cases}
\]
Thus, matching indicator values leave the kernel similarity unchanged, whereas differing values introduce a discount factor of $\exp(-1/\ell)<1$, reducing the similarity between the two observations. Consequently, one-hot-encoded categorical variables can be naturally incorporated into the Gaussian kernel alongside continuous covariates by directly evaluating the kernel on the full covariate vectors $x$ and $x'$.

\clearpage

\bibliographystyleEC{informs2014}
\bibliographyEC{GtO}

\end{document}